\documentclass{article}

\usepackage{microtype}
\usepackage[centerlast]{caption2}
\usepackage{booktabs} 

\usepackage{hyperref}
\usepackage{multirow}

\usepackage[accepted]{icml2023}

\usepackage{amsmath}
\usepackage{amssymb}
\usepackage{mathtools}
\usepackage{amsthm}
\usepackage{setspace}

\usepackage[capitalize,noabbrev]{cleveref}

\theoremstyle{plain}
\newtheorem{theorem}{Theorem}[section]

\newtheorem{lemma}[theorem]{Lemma}
\newtheorem{corollary}[theorem]{Corollary}
\theoremstyle{definition}
\newtheorem{definition}[theorem]{Definition}

\theoremstyle{remark}

\newtheorem{fact}[theorem]{\bf Fact}

\usepackage[textsize=tiny]{todonotes}

\begin{document}

\twocolumn[
\icmltitle{Fast Locality Sensitive Hashing with Theoretical Guarantee}



\icmlsetsymbol{equal}{*}

\begin{icmlauthorlist}
\icmlauthor{Zongyuan Tan}{equal,yy}
\icmlauthor{Hongya Wang}{equal,yy}
\icmlauthor{Bo Xu}{yy}
\icmlauthor{Minjie Luo}{yyy}
\icmlauthor{Ming Du}{yy}
\end{icmlauthorlist}

\icmlaffiliation{yy}{School of Computer Science and Technology, Donghua University, Shanghai, China}
\icmlaffiliation{yyy}{College of Science, Donghua University, Shanghai, China}

\icmlcorrespondingauthor{Hongya Wang}{hywang@dhu.edu.cn}

\icmlkeywords{Machine Learning, ICML}

\vskip 0.3in
]



\printAffiliationsAndNotice{\icmlEqualContribution} 

\begin{abstract}
Locality-sensitive hashing (LSH) is an effective randomized technique widely used in many machine learning tasks. The cost of hashing is proportional to data dimensions, and thus often the performance bottleneck when dimensionality is high and the number of hash functions involved is large. Surprisingly, however, little work has been done to improve the efficiency of LSH computation. In this paper, we design a simple yet efficient LSH scheme, named FastLSH, under $l_{2}$ norm. By combining random sampling and random projection, FastLSH reduces the time complexity from $O(n)$ to $O(m)$ ($m<n$), where $n$ is the data dimensionality and $m$ is the number of sampled dimensions. Moreover, FastLSH has provable LSH property, which distinguishes it from the non-LSH fast sketches. We conduct comprehensive experiments over a collection of real and synthetic datasets for the nearest neighbor search task. Experimental results demonstrate that FastLSH is \emph{on par with} the state-of-the-arts in terms of answer quality, space occupation and query efficiency, while enjoying up to 80x speedup in hash function evaluation. We believe that FastLSH is a promising alternative to the classic LSH scheme.
\end{abstract}

\section{Introduction}
\label{sec:intro}
Nearest neighbor search (NNS) is an essential problem in machine learning, which has numerous applications such as face recognition, information retrieval and duplicate detection. The purpose of nearest neighbor search is to find the point in the dataset $\mathbf{D} =\{\textbf{\emph{v}}_{1},\textbf{\emph{v}}_{2},\ldots,\textbf{\emph{v}}_{N}\}$ that is most similar (has minimal distance) to the given query $\textbf{\emph{u}}$. For low dimensional spaces ($<$10), popular tree-based index structures such as KD-tree~\cite{KD-Tree}, SR-tree~\cite{SR-Tree}, etc. deliver exact answers and provide satisfactory performance. For high dimensional spaces, however, these index structures suffer from the well-known curse of dimensionality, that is, their performance is even worse than that of linear scans~\cite{Linear-Scan}. To address this issue, one feasible way is to offer approximate results by trading accuracy for efficiency~\cite{tradeoffForANN}.

Locality-sensitive hashing (LSH) is an effective randomized technique to solve the problem of approximate nearest neighbor (ANN) search in high dimensional spaces~\cite{lsh,pslsh,near-lsh}. The basic idea of LSH is to map high dimensional points into buckets in low dimensional spaces using random hash functions, by which similar points have higher probability to end up in the same bucket than dissimilar points. A number of LSH functions are proposed for different similarity (distance) measures such as Hamming distance~\cite{lsh}, $l_2$ norm~\cite{pslsh}, angular similarity~\cite{simHash}, Jaccard similarity~\cite{minHash} and maximum inner product~\cite{MIPS1}.

The LSH scheme for $l_2$ norm (E2LSH) is originally proposed in~\cite{pslsh,E2LSH-package} based on $p$-stable distributions. Owing to the sub-linear time complexity and theoretical guarantee on query accuracy, E2LSH is arguably one of the most popular ANN search algorithms both in theory and practice, and many variants have been proposed to achieve much better space occupation and query response time~\cite{mplsh,LSB-Tree,c2lsh,SRS,qalsh,R2LSH,PDA-LSH,PM-LSH,DB-LSH}. Throughout this article, we focus on the LSH scheme under $l_2$ norm, and the extension for angular similarity and maximum inner product will be discussed in Section~\ref{sec:extension-of-fastlsh}.

In addition to answer ANN queries, LSH finds applications in many other domains. To name a few, LSH is utilized in hierarchical clustering to approximate the distances and reduce the time complexity of clustering~\cite{hierarchinal-cluster}.
~\cite{more-cluster} uses LSH to remove outliers from the initial samples in order to perform association rule mining on large datasets in a reasonable amount of time. Recently, LSH is recognized as an effective sampler for efficient large-scale neural network training~\cite{SLIDE}.


All of these LSH applications involve computation of LSH functions over the entire dataset. Take the widely used E2LSH as an example, computing $k$ hashes of a vector $\textbf{\emph{v}}$ takes $O(nk)$ computation, where $n$ is the dimensionality of $\textbf{\emph{v}}$ and $k$ is the number of hash functions. For typical ANN search task, $k$ commonly ranges from few hundreds to several thousands, and keeps growing with the cardinality of the dataset ($N$) since the number of hashes required by E2LSH is $O(N^\rho)$~\cite{pslsh}. Hence, hashing is the main computational and resource bottleneck step in almost all LSH-based applications, especially when data come in a streaming fashion and/or LSH data structures have to be constructed repeatedly~\cite{PDA-LSH,STREAM}.



 One notable work for speeding up the computation of LSH functions is the densified one permutation hashing~\cite{densifying}, which requires only $O(n+k)$ computations instead of $O(nk)$ of Minhashing~\cite{minHash}.
 Surprisingly enough, little endeavor has been made on more efficient LSH schemes under $l_2$ norm. The only known technique, termed as ACHash~\cite{DHHash}, exploits fast Hadamard transform to estimate the distance distortion of two points in the Euclidean space. This method, like other fast JL sketches~\cite{JL-Lemma1,JL-Lemma2}, does not owns the provable LSH property. Thus, it is not a desirable alternative to E2LSH because there are substantial empirical evidence that using these (non-LSH) sketches for indexing leads to a drastic bias in the expected behavior, leading to poor accuracy~\cite{minHash}.

\textbf{Our Contribution:} We develop a simple yet efficient LSH scheme (FastLSH) under $l_{2}$ norm. FastLSH involves two basic operations - random sampling and random projection, and offers (empirically) near constant time complexity instead of $O(nk)$ of E2LSH. Also, we derive the expression of the probability of collision (equity of hash values) for FastLSH and prove the asymptotic equivalence w.r.t $m$ (the number of sampled dimensions) between FastLSH and E2LSH, which means that our proposal owns the desirable LSH property. We also rigidly analyze how $m$ affects the probability of collision when $m$ is relatively small. To further validate our claims, we conduct comprehensive experiments over various publicly available datasets for the ANN search task.  Experimental results demonstrate that FastLSH has comparable answer accuracy and query efficiency with the classic LSH scheme, while significantly reducing the cost of hashing.
Considering its simplicity and efficiency, we believe that FastLSH is a promising alternative to E2LSH in practical applications.

\section{Prelimianries}
\label{sec:prelimianries}
In this section, we introduce notations and background knowledge used in this article. Let $\mathbf{D}$ be the dataset of size $N$ in $\mathbb{R}^{n}$ and $\textbf{\emph{v}} \in \mathbf{D}$ be a data point (vector) and $\textbf{\emph{u}} \in \mathbb{R}^{n}$ be a query vector. We denote $\phi(x)=\frac{1}{\sqrt{2\pi}}\exp(-\frac{x^{2}}{2})$ and $\Phi(x)=\int_{-\infty }^{x}\frac{1}{\sqrt{2\pi}}\exp(-\frac{x^{2}}{2})dx  $ as the probability density function (PDF) and cumulative distribution function (CDF) of the standard normal distribution $\mathcal{N}(0,1)$, respectively.

\subsection{Locality Sensitive Hashing}
\label{sec:lsh}
\begin{definition}
($c$-Approximate $R$-Near Neighbor or $(c,R)$-NN problem) Given a set $\mathbf{D}$ in $n$-dimensional Euclidean space $\mathbb{R}^{n}$ and parameters $R>0$, $\delta>0$, construct a data structure which, for any given query point $\textbf{\emph{u}}$, does the following with probability $1-\delta$: if there exists an $R$-near neighbor of $\textbf{\emph{u}}$ in $\mathbf{D}$, it reports some $cR$-near neighbor of $\textbf{\emph{u}}$ in $\mathbf{D}$.
\end{definition}

Formally, an $R$-near neighbor of $\textbf{\emph{u}}$ is a point $\textbf{\emph{v}}$ such that $\left \| \textbf{\emph{v}}-\textbf{\emph{u}} \right \|_{2}\le R$. Locality Sensitive Hashing is an effectively randomized technique to solve the $(c,R)$-NN problem. It is a family of hash functions that can hash points into buckets, where similar points have higher probability to end up in the same bucket than dissimilar points. Consider a family $\mathcal{H}$ of hash functions mapping $\mathbb{R}^{n}$ to some universe $U$.

\begin{definition}
(Locality Sensitive Hashing) A hash function family $\mathcal{H}=\{h: \mathbb{R}^{n} \rightarrow U\}$ is called $(R,cR,p_{1},p_{2})$-sensitive if for any $\textbf{\emph{v}}, \textbf{\emph{u}} \in \mathbb{R}^{n}$
\begin{itemize}
\item[$\bullet$]
if $\left \| \textbf{\emph{v}}-\textbf{\emph{u}} \right \|_{2}\leq R$ then $Pr_{\mathcal{H}}[h(\textbf{\emph{v}})=h(\textbf{\emph{u}})]\geq p_{1}$;

\item[$\bullet$]
if $\left \| \textbf{\emph{v}}-\textbf{\emph{u}} \right \|_{2}\geq cR$ then $Pr_{\mathcal{H}}[h(\textbf{\emph{v}})=h(\textbf{\emph{u}})]\leq p_{2}$;
\end{itemize}
In order for the LSH family to be useful, it has to satisfy $c > 1$ and $p_{1} > p_{2}$. Please note that only hashing schemes with such a property are qualified locality sensitive hashing and can enjoy the theoretical guarantee of LSH.
\end{definition}

\subsection{LSH for $l_{2}$ Norm}
\label{sec:lsh-for-l2-norn}
\cite{pslsh} presents an LSH family that can be employed for $l_{p}$ $(p\in(0,2])$ norms based on $p$-stable distribution. When $p=2$, it yields the well-known LSH family for $l_{2}$ norm (E2LSH). The hash function is defined as follows:
\begin{equation}
\label{equ:e2lsh}
h_{\textbf{\emph{a}},b} (\textbf{\emph{v}})=\left \lfloor \frac{\textbf{\emph{a}}^{T}\textbf{\emph{v}} +b}{w} \right \rfloor
\end{equation}
where $\lfloor \rfloor$ is the floor operation, $\textbf{\emph{a}}$ is a $n$-dimensional vector with each entry chosen independently from $\mathcal{N}(0,1)$ and $b$ is a real number chosen uniformly from the range $[0,w]$. $w$ is an important parameter by which one could tune the performance of E2LSH.

For E2LSH, the probability of collision for any pair $(\textbf{\emph{v}},\textbf{\emph{u}})$ under $h_{\textbf{\emph{a}} ,b}(\cdot)$ can be computed as:
\begin{equation}
\label{eqn:prob.-of-collision-e2lsh}
p(s)=Pr[h_{\textbf{\emph{a}} ,b}(\textbf{\emph{v}})=h_{\textbf{\emph{a}} ,b}(\textbf{\emph{u}})]=\int_{0}^{w}f_{|sX|} (t)(1-\frac{t}{w})dt
\end{equation}
where $s=\|\textbf{\emph{v}}-\textbf{\emph{u}}\|_{2}$ is the Euclidean distance between $(\textbf{\emph{v}},\textbf{\emph{u}})$, and $f_{|sX|}(t)$ is the PDF of the absolute value of normal distribution $sX$ ($X$ is a random variable following the standard normal distribution). Given $w$, $p(s)$ is a monotonically decreasing function of $s$, which means $h_{\textbf{\emph{a}} ,b}(\cdot)$ satisfies the LSH property.

\begin{fact}\cite{pslsh}
Given the LSH family, a data structure can be built to solve the $(c,R)$-NN problem with $O(N^{\rho} \log N)$ query time and $O(N^{1+\rho})$ space, where $\rho = \frac{\log(1/p_{1})}{\log(1/p_{2})}$.
\end{fact}

\subsection{Truncated Normal Distribution}
\label{sec:truncated-normal}






The truncated normal distribution is suggested if one need to use the normal distribution to describe the random variation of a quantity that, for physical
reasons, must be strictly in the range of a truncated interval instead of $(-\infty, +\infty)$~\cite{truncated-normal}. The truncated normal distribution is the probability distribution derived from that of normal random variables by bounding the values from either below or above (or both). Assume that the interval $(a_{1}, a_{2})$ is the truncated interval, then the probability density function can be written as:
\begin{equation}
\label{eqn:pdf-of-truncated-normal}
\psi(x;\mu,\sigma^{2},a_{1},a_{2}) =
\left \{
\begin{array}{cc}
0 & x \leq a_{1} \\
\frac{\phi(x;\mu,\sigma^{2})}{\Phi(a_{2};\mu,\sigma^{2})-\Phi(a_{1};\mu,\sigma^{2})} & a_{1} < x < a_{2} \\
0  & a_{2} \leq x
\end{array}
\right.
\end{equation}
The cumulative distribution function is:
\begin{equation}
\label{eqn:cdf-of-truncated-normal}
\Psi(x;\mu,\sigma^{2},a_{1},a_{2}) =
\left \{
\begin{array}{cc}
0 & x \leq a_{1} \\
\frac{\Phi(x;\mu,\sigma^{2}) - \Phi(a_{1};\mu,\sigma^{2})}{\Phi(a_{2};\mu,\sigma^{2})-\Phi(a_{1};\mu,\sigma^{2})} & a_{1} < x < a_{2} \\
1  & a_{2} \leq x
\end{array}
\right.
\end{equation}




\section{Fast LSH via Random Sampling}
\label{sec:fast-lsh-via-random-sampling}
\subsection{The Proposed LSH Function Family}
\label{sec:proposed-lsh-function-family}
The cost of hashing defined in Eqn.~\eqref{equ:e2lsh} is dominated by the inner product $\textbf{\emph{a}}^T\textbf{\emph{v}}$, which takes $O(n)$ multiplication and addition operations. As mentioned in Section~\ref{sec:intro}, hashing is one of the main computational bottleneck in almost all LSH applications when the number of hashes is large and the amount of data keeps increasing. To address this issue, we propose a novel family of locality sensitive hashing termed as FastLSH. Computing hash values with FastLSH involves two steps, i.e., random sampling and random projection.

In the first step, we do random sampling from $n$ dimensions. Particularly, we draw $m$ \emph{i.i.d.} samples in the range of 1 to $n$ uniformly to form a multiset $S$. For every $\textbf{\emph{v}}=\{v_1, v_2, \cdots, v_n\}$, we concatenate all $v_i$ to form a $m$-dimensional vector $\tilde{\textbf{\emph{v}}}=\{\tilde{v}_1, \tilde{v}_2, \cdots, \tilde{v}_m\}$ if $i \in S$. As a quick example, suppose $\textbf{\emph{v}}=\{1, 3, 5, 7, 9\}$ is a 5-dimensional vector and $S=\{2, 4, 2\}$. Then we can get a 3-dimensional vector $\tilde{\textbf{\emph{v}}}=\{3, 7, 3\}$ under $S$. It is easy to see that each entry in $\textbf{\emph{v}}$ has equal probability $\frac{m}{n}$ of being chosen. Next, we slightly overuse notation $S$ and denote by $S(\cdot)$ the random sampling operator. Thus, $\tilde{\textbf{\emph{v}}}\in \mathbb{R}^{m}=S(\textbf{\emph{v}})$ for $\textbf{\emph{v}}\in \mathbb{R}^{n}$ $(m < n)$.



In the second step, the hash value is computed in the same way as Eqn.~\eqref{equ:e2lsh} using $\tilde{\textbf{\emph{v}}}$ instead of \textbf{\emph{v}}, and then the overall hash function is formally defined as follows:
\begin{equation}
\label{eqn:fastlsh}
h_{\tilde{\textbf{\emph{a}}},\tilde{b}}(\textbf{\emph{v}}) =\left \lfloor \frac{\tilde{\textbf{\emph{a}}}^{T}S(\textbf{\emph{v}})+\tilde{b}}{\tilde{w}} \right \rfloor
\end{equation}
where $\tilde{\textbf{\emph{a}}}\in \mathbb{R}^{m}$ is the random projection vector of which each entry is chosen independently from $\mathcal{N}(0,1)$, $\tilde{w}$ is a user-specified constant and $\tilde{b}$ is a real number uniformly drawn from $[0,\tilde{w}]$. The hash function $h_{\tilde{\textbf{\emph{a}}},\tilde{b}}(\textbf{\emph{v}})$ maps a $n$-dimensional vector $\textbf{\emph{v}}$ onto the set of integers.

Compared with E2LSH, FastLSH reduces the cost of hashing from $O(n)$ to $O(m)$. As will be discussed in Section~\ref{sec:experiment}, a relatively small $m<n$ suffices to provide competitive performance against E2LSH, which leads to significant performance gain in hash function evaluation.

\subsection{ANN Search Using FastLSH}
\label{sec:ann-search-using-fastlsh}
FastLSH can be easily plugged into any existing LSH applications considering its simplicity. This section gives a brief discussion about how to use FastLSH for ANN search.

The canonical LSH index structure for ANN search is built as follows. A hash code $H(\textbf{\emph{v}}) = (h_{1}(\textbf{\emph{v}}), h_{2}(\textbf{\emph{v}}),\ldots, h_{k}(\textbf{\emph{v}})$ is computed using $k$ independent LSH functions (i.e., $H(\textbf{\emph{v}})$ is the concatenation of $k$ elementary LSH codes). Then a hash table is constructed by adding the 2-tuple $ \langle H(\textbf{\emph{v}}), \textit{id of } \textbf{\emph{v}} \rangle$ into corresponding bucket. To boost accuracy, $L$ groups of hash functions $H_i(\cdot), i=1, \cdots, L$ are drawn independently and uniformly at random from the LSH family, resulting in $L$ hash tables. To use FastLSH in such a ANN search framework, the only modification is to replace hash function definitin in Eqn.~\eqref{equ:e2lsh} with that of Eqn.~\eqref{eqn:fastlsh}.



To answer a query $\textbf{\emph{u}}$, one need to first compute $H_1(\textbf{\emph{u}}), \cdots, H_L(\textbf{\emph{u}})$ and then search all these $L$ buckets to obtain the combined set of candidates. There exists two ways (approximate and exact) to process these candidates. In the approximate version, we only evaluate no more than $3L$ points in the candidate set. The LSH theory ensures that the $(c, R)$-NN is found with a constant probability. In practice, however, the exact one is widely used since it offers better accuracy at the cost of evaluating all points in the candidate set~\cite{pslsh}. The search time consists of both the hashing time and the
time taken to prune the candidate set for the exact version. In this paper, we use the exact method to process a query similar to~\cite{pslsh,E2LSH-package}.

\section{Theoretical Analysis}
\label{sec:theoretical-analysis}
While FastLSH is easy to comprehend and simple to implement, it is non-trivial to show that the proposed LSH function meets the LSH property, i.e., the probability of collision for ($\textbf{\emph{v}},\textbf{\emph{u}}$) decreases with their $l_2$ norm increasing. In this section, we will first derive the probability of collision for FastLSH, and prove that its asymptotic behavior is equivalent to E2LSH. Then, by using both rigid analysis and the numerical method, we demonstrate that FastLSH still owns desirable LSH property when $m$ is relatively small.

\subsection{Probability of Collision}
\label{sec:probability-of-collision}
For given vector pair $(\textbf{\emph{v}},\textbf{\emph{u}})$, let $s =\left \| \textbf{\emph{v}}-\textbf{\emph{u}} \right \|_{2}$. For our purpose,  the collection of $n$ entries $(v_{i}-u_{i})^{2}$ $\{i=1,2,\ldots,n\}$ is viewed as a population, which follows an unknown distribution with finite mean $\mu =( {\textstyle \sum_{i=1}^{n}} (v_{i}-u_{i})^{2}) /n$ and variance $\sigma^{2}=({\textstyle \sum_{i=1}^{n}} ((v_{i}-u_{i})^{2}-\mu)^{2}) / n $. After performing the sampling operator $S(\cdot)$ of size $m$, $\textbf{\emph{v}}$ and $\textbf{\emph{u}}$ are transformed into $\tilde{\textbf{\emph{v}}}=S(\textbf{\emph{v}})$ and $\tilde{\textbf{\emph{u}}} =S(\textbf{\emph{u}})$, and the squared distance of $(\tilde{\textbf{\emph{v}}}, \tilde{\textbf{\emph{u}}})$ is $\tilde{s}^{2}= {\textstyle \sum_{i=1}^{m}} (\tilde{v}_{i}-\tilde{u}_{i})^{2}$. By Central Limit Theorem, we have the following lemma:

\begin{lemma}
\label{lemma:clt}
If $m$ is sufficiently large, then the sum $\tilde{s}^{2}$ of $m$ \emph{i.i.d.} random samples $(\tilde{v}_{i}-\tilde{u}_{i})^{2}$ $(i\in{1,2,\ldots,m})$ converges asymptotically to the normal distribution with mean $m\mu$ and variance $m\sigma^{2}$, i.e., $\tilde{s}^{2}\sim \mathcal{N}(m\mu,m\sigma^{2})$.
\end{lemma}

Lemma~\ref{lemma:clt} states that the squared distance between $\tilde{\textbf{\emph{v}}}$ and $\tilde{\textbf{\emph{u}}}$ follows a normal distribution for large $m$. Practically, a small $m$ (say 30) often suffices to make the sampling distribution of the sample mean approaches the normal in real-life applications~\cite{CLT,CLT1}.

Random variable $\tilde{s}^{2}$, however, does not follow exactly the normal distribution since $\tilde{s}^{2} \geq 0$ whereas the range of definition of the normal distribution is $(-\infty,+\infty)$. A mathematically defensible way to preserve the main features of the normal distribution while avoiding negative values involves \textit{the truncated normal distribution}, in which the range of definition is made finite at one or both ends of the interval.

Particularly, $\tilde{s}^{2}$ can be modeled by normal distribution $\tilde{s}^{2}\sim \mathcal{N}(m\mu,m\sigma^{2})$ over the truncation interval $[0,+\infty)$, that is, the singly-truncated normal distribution $\psi(x; \tilde{\mu},\tilde{\sigma}^{2},0,+\infty)$. Considering the fact that $\tilde{s}\geq 0$, we have $Pr[\tilde{s} < t]=Pr[\tilde{s}^{2}  < t^{2}]$ for any $t > 0$. Therefore, the CDF of $\tilde{s}$, denoted by $F_{\tilde{s}}$, can be computed as follows:
\begin{equation}
\label{eqn:cdf-of-s}
\begin{aligned}
F_{\tilde{s}}(t) &= Pr[\tilde{s} < t] = Pr[\tilde{s}^{2}  < t^{2}] \\ &=\int_{0}^{t^{2}}\psi(x; \tilde{\mu},\tilde{\sigma}^{2},0,\infty)dx
\end{aligned}
\end{equation}
where $\tilde{\mu} = m\mu$ and $\tilde{\sigma}^{2} = m\sigma^{2}$. Due to the fact that the PDF is the derivative of the CDF, the PDF of $\tilde{s}$, denoted by $f_{\tilde{s}}$, is derived as follows:
\begin{equation}
\label{eqn:pdf-of-s}
f_{\tilde{s}}(t) = \frac{d}{dt}[F_{\tilde{s}}(t)] = 2t\psi (t^{2}; \tilde{\mu},\tilde{\sigma}^{2},0,\infty)
\end{equation}

Recall that $\textbf{\emph{a}}$ is a projection vector with entries being \emph{i.i.d} samples drawn from $\mathcal{N}(0,1)$. It follows from the $p$-stability that the distance between projections $(\textbf{\emph{a}}^{T}\textbf{\emph{v}}-\textbf{\emph{a}}^{T}\textbf{\emph{u}})$ for two vectors $\textbf{\emph{v}}$ and $\textbf{\emph{u}}$ is distributed as $\left \| \textbf{\emph{v}}-\textbf{\emph{u}} \right \|_{2}X$, i.e., $sX$, where $X \sim \mathcal{N}(0,1)$~\cite{stable-distribution,pslsh}. Similarly, the projection distance between $\tilde{\textbf{\emph{v}}}$ and $\tilde{\textbf{\emph{u}}}$ under $\tilde{\textbf{\emph{a}}}$ $(\tilde{\textbf{\emph{a}}}^{T}\tilde{\textbf{\emph{v}}}-\tilde{\textbf{\emph{a}}}^{T}\tilde{\textbf{\emph{u}}})$ follows the distribution $\tilde{s}X$. Note that the PDF of $sX$, i.e., $f_{sX} (x) = \frac{1}{s}\phi(\frac{x}{s})$, is an important factor in calculating the probability of collision in Eqn.~\eqref{eqn:prob.-of-collision-e2lsh}. Hence, we need to get the PDF of $\tilde{s}X$ first in order to derive the probability of collision for FastLSH.

Although we know the distributions of both $\tilde{s}$ and $X$, it is not straight-forward to figure out the distribution of their product $\tilde{s}X$. Fortunately, Lemma~\ref{lemma:characteristic-func-of-w} gives the characteristic function of random variable $W=XY$, where $X$ and $Y$ are two independent random variables, one following a standard normal distribution and the other following a distribution with mean $\mu$ and variance $\sigma^{2}$.

\begin{lemma}
\label{lemma:characteristic-func-of-w}
The characteristic function of the product of two independent random variables $W=XY$ is as follows:
\begin{equation}
\nonumber
    \varphi_{W}(x)=E_Y\{\exp(-\frac{x^{2}Y^{2}}{2})\}
\end{equation}
where $X$ is a standard normal random variable and $Y$ is an independent random variable with mean $\mu$ and variance $\sigma^{2}$.
\end{lemma}

\begin{proof}
See Appendix~\ref{lemma:characteristic-func-of-w-in-appendix}
\end{proof}


Note that the distribution of a random variable is determined uniquely by its characteristic function. With respect to $\tilde{s}X$, the characteristic function can be obtained by Lemma~\ref{lemma:characteristic-func-of-w} since $X$ follows the standard normal:

\begin{lemma}
\label{lemma:characteristic-func-of-tildesx}
The characteristic function of $\tilde{s}X$ is as follows:
\begin{equation}
\nonumber
\begin{aligned}
\varphi_{\tilde{s}X}(x)& =\frac{1}{2(1-\Phi(\frac{-\tilde{\mu}}{\tilde{\sigma}}))} \exp(\frac{1}{8}x^4\tilde{\sigma}^2 -\frac{1}{2}\tilde{\mu}x^2) \\ & \cdot\operatorname{erfc}(\frac{\frac{1}{2} x^2 \tilde{\sigma}^2-\tilde{\mu}}{\sqrt{2} \tilde{\sigma}}) \quad (-\infty <x< +\infty)
\end{aligned}
\end{equation}
where $\operatorname{erfc}(t) = \frac{2}{\sqrt{\pi}}\int_{t}^{+\infty}\exp(-x^{2})dx$ $(-\infty <t< +\infty)$ is the complementary error function.
\end{lemma}

\begin{proof}
See Appendix~\ref{lemma:characteristic-func-of-tildesx-in-appendix}
\end{proof}

Given the characteristic function, the probability density function can be obtained through the inverse Fourier transformation. Thus, the PDF of $\tilde{s}X$, denoted by $f_{\tilde{s}X}(t)$, is as follows:
\begin{equation}
\label{eqn:pdf-of-FastLSH}
f_{\tilde{s}X}(t)=\frac{1}{2\pi}\int_{-\infty }^{+\infty }\exp(-itx)\varphi_{\tilde{s}X}(x)dx
\end{equation}
where the symbol $i=\sqrt{-1}$ represents the imaginary unit.


Now we are in the position to derive the probability of collision for vector pair $(\textbf{\emph{v}},\textbf{\emph{u}})$ under the proposed LSH function in Eqn.~\eqref{eqn:fastlsh}. Let $f_{|\tilde{s}X|}(t)$ represent the PDF of the absolute value of $\tilde{s}X$. It is easy to obtain the collision probability $p(s, \sigma)$ for FastLSH by replacing $f_{|sX|}(t)$ in Eqn.~\eqref{eqn:prob.-of-collision-e2lsh} with $f_{|\tilde{s}X|}(t)$.

\begin{theorem}
\label{theorem:collision-prob.-of-FastLSH}
\begin{equation}
\label{eqn:pro.-of-FastLSH}
\begin{aligned}
p(s, \sigma)&=Pr[h_{\tilde{\textbf{a}},\tilde{b}}(\textbf{v})=h_{\tilde{\textbf{a}},\tilde{b}}(\textbf{u})]\\ &=\int_{0}^{\tilde{w}}f_{|\tilde{s}X|}(t)(1-\frac{t}{\tilde{w}})dt
\end{aligned}
\end{equation}
\end{theorem}

\begin{proof}
See Appendix~\ref{theorem:collision-prob.-of-FastLSH-in-appendix}
\end{proof}

We can see that, unlike E2LSH, the probability of collision $p(s, \sigma)$ depends on both $s$ and $\sigma$. From this point of view, FastLSH can be regarded as a generalized version of E2LSH by considering one additional impact factor, i.e., the variation in the squared distance of each dimension for vector pair $(\textbf{\emph{v}}, \textbf{\emph{u}})$, making it is more difficult to prove its LSH property. The influence of $\sigma$ on $p(s, \sigma)$ will be discussed in the next section.


\subsection{The LSH Property of FastLSH}
\label{sec:the-lsh-property-of-fastlsh}
In this section, we first prove that the asymptotic behavior of FastLSH is identical to E2LSH, and then demonstrate by both theoretical analysis and numerical computation that FastLSH owns the LSH property even for limited $m$.

\begin{fact}\cite{pslsh}
\label{fact:relation-of-w-and-s}
For E2LSH, $f_{sX}(t)$ follows the normal distribution $\mathcal{N}(0,s^{2})$ and the collision probability $p(s)$ with bucket width $w$ is equal to $p(\alpha s)$ under the bucket width $\alpha w$, i.e., $p_{w}(s) = p_{\alpha w}(\alpha s)$ where $\alpha > 0$.
\end{fact}

From Eqn.~\eqref{eqn:prob.-of-collision-e2lsh} and Eqn.~\eqref{eqn:pro.-of-FastLSH}, one can see that the expressions of the probability of collision for E2LSH and FastLSH are quite similar. Actually, if $f_{\tilde{s}X}(t)$ also follows normal distribution with zero mean, we can always scale $\tilde{w}$ to make $p_{w}(s) = p_{\tilde{w}}(s,\sigma)$ based on Fact~\ref{fact:relation-of-w-and-s}. The following theorem gives the asymptotic behavior of the characteristic function of $\tilde{s}X$.


\begin{theorem}
\label{theorem:equivalence-of-characteristic-function}
\begin{equation}
\nonumber
\lim_{m \to +\infty} \frac{\varphi_{\tilde{s}X}(x)}{\exp(-\frac{ms^{2}x^{2}}{2n})} = 1
\end{equation}
where $|x|\leq O(m^{-1/2})$ and  $\exp(-\frac{ms^{2}x^{2}}{2n})$ is the characteristic function of $\mathcal{N}(0,\frac{ms^{2}}{n})$.
\end{theorem}

\begin{proof}
See Appendix~\ref{theorem:equivalence-of-characteristic-function-in-appendix}.
\end{proof}

Note that $\exp(-\frac{ms^{2}x^{2}}{2n})$, the characteristic function of $\mathcal{N}(0,\frac{ms^{2}}{n})$, is a bell-shaped function like the PDF of normal distributions. As a result, Theorem~\ref{theorem:equivalence-of-characteristic-function} implies that $\varphi_{\tilde{s}X}(x)$ is asymptotically identical to $\exp(-\frac{ms^{2}x^{2}}{2n})$ within interval $[-\mathcal{K}\sqrt{\frac{n}{ms^{2}}}, +\mathcal{K}\sqrt{\frac{n}{ms^{2}}}]$, that is, 2$\mathcal{K}$ ``standard deviations", where $\mathcal{K}$ is an arbitrarily large constant. Because a probability distribution is uniquely determined by its characteristic function and both $\varphi_{\tilde{s}X}(x)$ and $\exp(-\frac{ms^{2}x^{2}}{2n})$ approach 0 when $x$ tends to infinity, we immediately have the following Corollary.

\begin{corollary}
\label{corollary:equivalence-of-pdf}
$f_{\tilde{s}X}(t) \sim \text{the PDF of~} \mathcal{N}(0,\frac{ms^{2}}{n})$ as $m$ approaches infinity.
\end{corollary}

By Corollary~\ref{corollary:equivalence-of-pdf}, $p(s) = p(s, \sigma)$ asymptotically if $\tilde{w} = \frac{m}{n}w$, meaning that $\sigma$ has no effect on the probability of collision.
In practical scenarios, $m$ is often limited. Next, we study the relation between $f_{\tilde{s}X}(t)$ and the PDF of $\mathcal{N}(0,\frac{ms^{2}}{n})$ when $m$ is relatively small ($m < n$). Particularly, we derive the first four moments of $\tilde{s}X$ and $\mathcal{N}(0,\frac{ms^{2}}{n})$, and analyze how $m$ and $\sigma$ affect their similarity. While in general the first four moments, or even the whole moment sequence may not determine a distribution~\cite{moment1}, practitioners find that distributions near the normal can be decided very well given the first four moments~\cite{moment2,moment3,moment4}.


\begin{lemma}
\label{lemma:four-moment-of-FastLSH}
\begin{equation}
\nonumber
\begin{cases}
\ \ E(\tilde{s}X) \ \ \ = 0
 \\
E((\tilde{s}X)^{2}) = \frac{ms^{2}}{n}(1+\epsilon)
    \\
E((\tilde{s}X)^{3})  = 0
\\
E((\tilde{s}X)^{4}) = \frac{3m^{2}s^{4}}{n^{2}}(1+\lambda)
\end{cases}
\end{equation}
where $\epsilon=\frac{\tilde{\sigma} \exp(\frac{-\tilde{\mu}^{2}}{2\tilde{\sigma}^{2}})}{\sqrt{2\pi}\tilde{\mu} (1-\Phi(\frac{-\tilde{\mu}}{\tilde{\sigma}}))}$ and $\lambda = \frac{\tilde{\sigma}^{2}}{\tilde{\mu}^{2}} + \epsilon$.
\end{lemma}

\begin{proof}
See Appendix~\ref{lemma:four-moment-of-FastLSH-in-appendix}
\end{proof}


\begin{figure*}[t]	
    \centering
    \begin{minipage}{0.22\textwidth}
		\centering
		\centerline{\includegraphics[width=\textwidth]{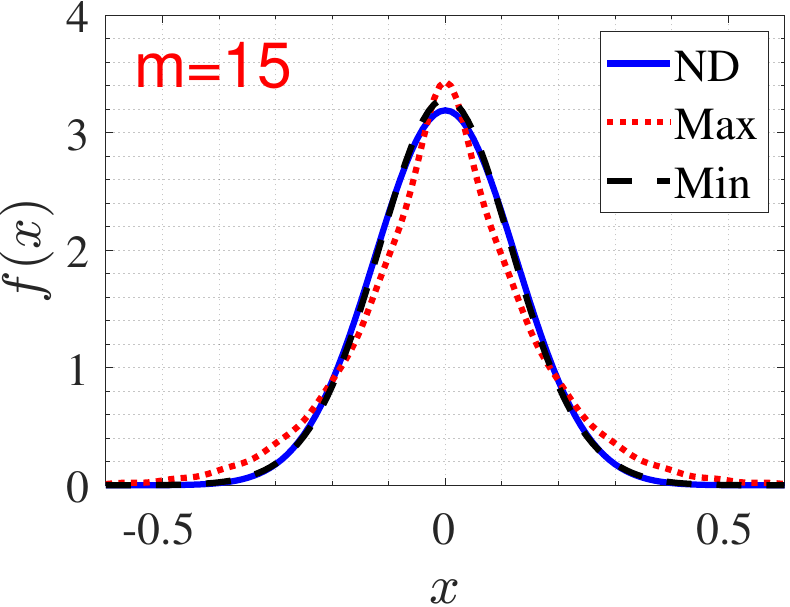}}
        \centerline{\scriptsize (a-1) Gist: $m=15$}
	\end{minipage}
    \quad
	\begin{minipage}{0.22\textwidth}
        \centering
		\centerline{\includegraphics[width=\textwidth]{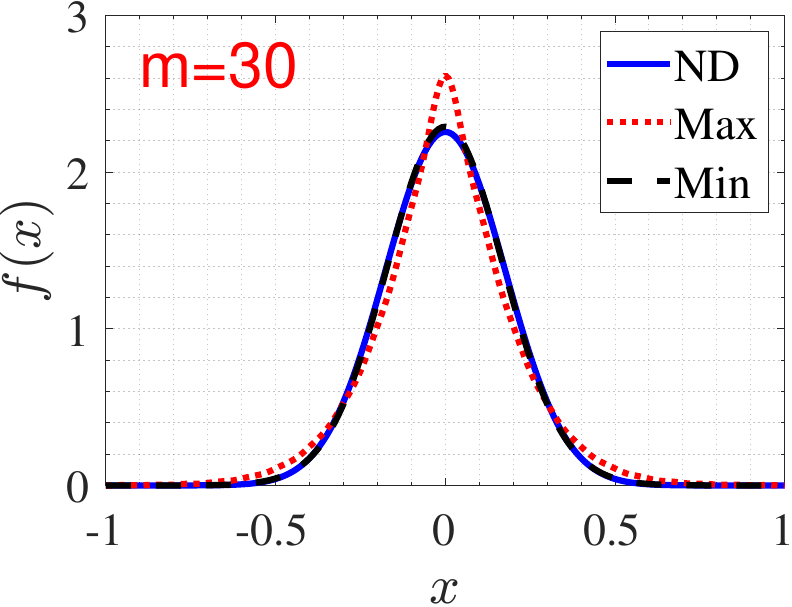}}
        \centerline{\scriptsize (a-2) Gist: $m=30$}
	\end{minipage}
    \quad
    \begin{minipage}{0.22\textwidth}
        \centering
		\centerline{\includegraphics[width=\textwidth]{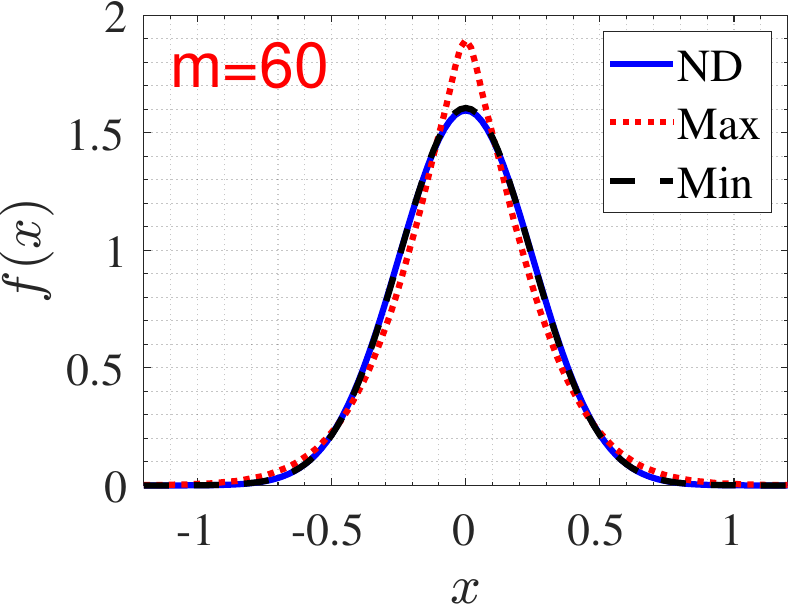}}
        \centerline{\scriptsize (a-3) Gist: $m=60$}
	\end{minipage}
     \quad
    \begin{minipage}{0.22\textwidth}
        \centering
		\centerline{\includegraphics[width=\textwidth]{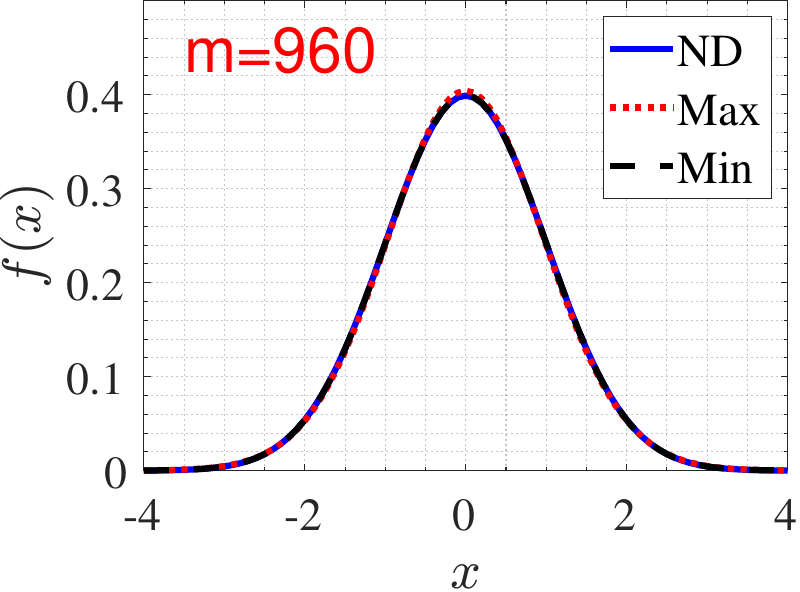}}
        \centerline{\scriptsize (a-4) Gist: $m=n$}
	\end{minipage}
    \quad
    \begin{minipage}{0.22\textwidth}
		\centering
		\centerline{\includegraphics[width=\textwidth]{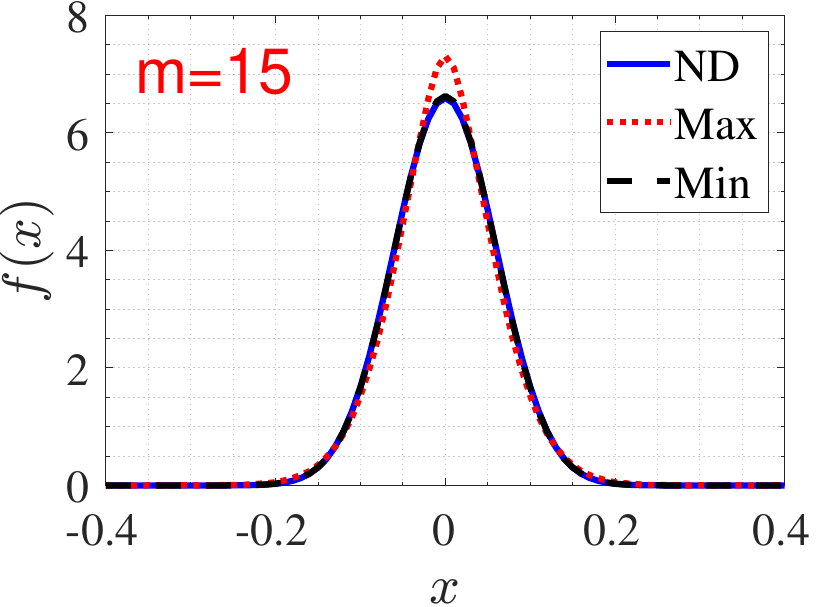}}
        \centerline{\scriptsize (b-1) Trevi: $m=15$}
	\end{minipage}
    \quad
	\begin{minipage}{0.22\textwidth}
        \centering
		\centerline{\includegraphics[width=\textwidth]{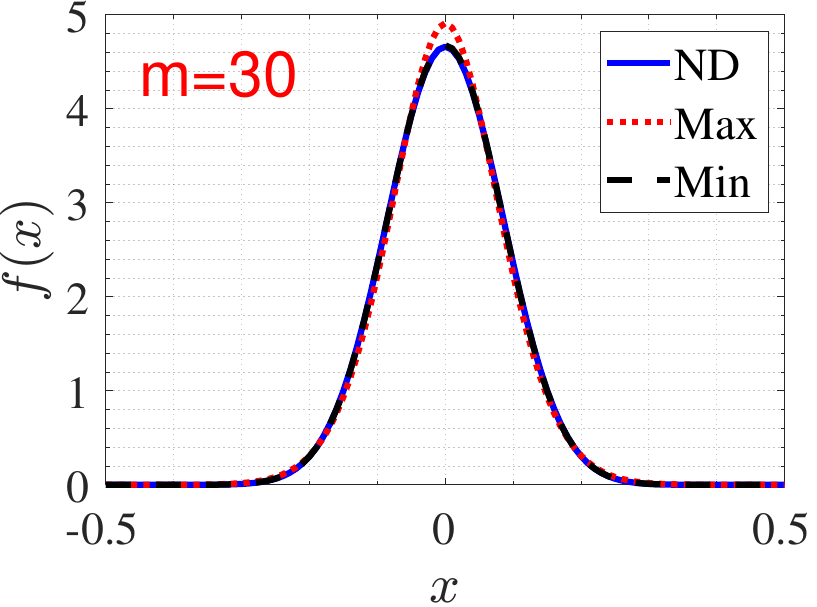}}
        \centerline{\scriptsize (b-2) Trevi: $m=30$}
	\end{minipage}
    \quad
    \begin{minipage}{0.22\textwidth}
        \centering
		\centerline{\includegraphics[width=\textwidth]{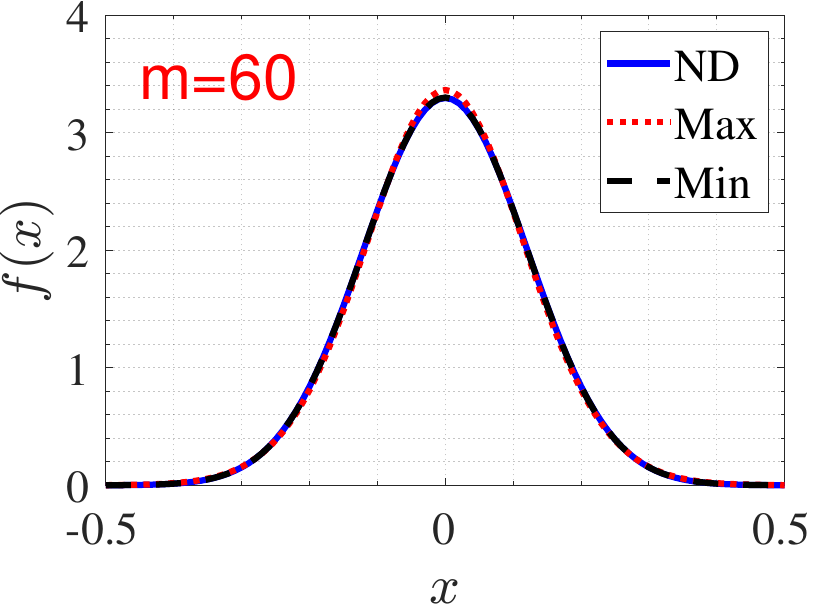}}
        \centerline{\scriptsize (b-3) Trevi: $m=60$}
	\end{minipage}
     \quad
    \begin{minipage}{0.22\textwidth}
        \centering
		\centerline{\includegraphics[width=\textwidth]{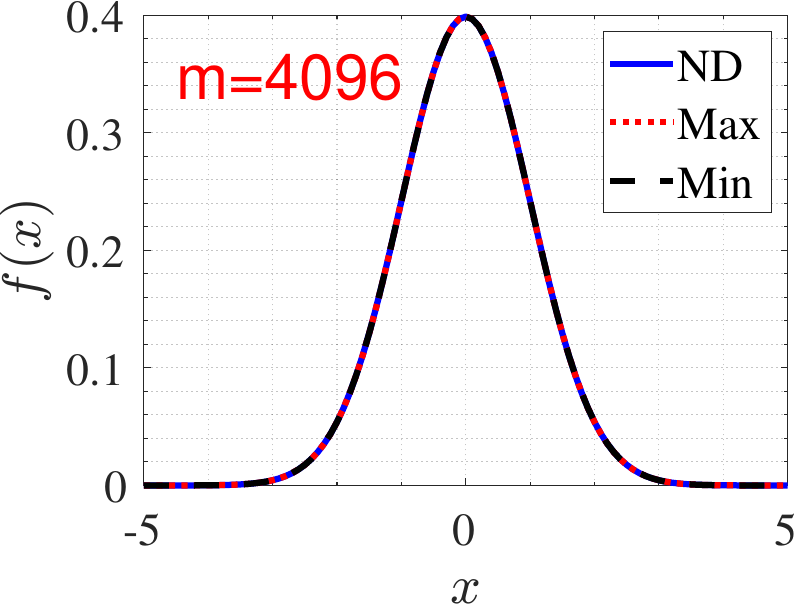}}
        \centerline{\scriptsize (b-4) Trevi: $m=n$}
	\end{minipage}
    \caption{Comparison of probability density curves of $\mathcal{N}(0,\frac{ms^{2}}{n})$ (ND) and $\tilde{s}X$ under different $m$ over datasets \emph{Gist} and \emph{Trevi}.} \label{fig:pdf-with-different-m-for-trevi}
\end{figure*}

\begin{fact}
\label{fact:moment-of-normal}
\cite{computeMoments} The first four moments of $\mathcal{N}(0,\frac{ms^{2}}{n})$ are:
\begin{equation}
\nonumber
\begin{cases}
\ \ E(sX) \ \ \ = 0
 \\
E((sX)^{2}) = \frac{ms^{2}}{n}
    \\
E((sX)^{3})  = 0
\\
E((sX)^{4}) = \frac{3m^{2}s^{4}}{n^{2}}
\end{cases}
\end{equation}
\end{fact}
It is easy to see that $\epsilon$ and $\lambda$ are monotonously decreasing function of $m$ and increasing function of $\sigma$ ($\frac{\tilde{\sigma}}{\tilde{u}}=\frac{\sigma n}{\sqrt{m} s^2}$). Hence, the first four moments of $sX$ and $\tilde{s}X$ are equal with each other, respectively, as $m$ approaches infinity. This is consistent with Corollary~\ref{corollary:equivalence-of-pdf}.


For limited $m < n$, the impact of $\sigma$ is not negligible. However, we can easily adjust $m$ to control the impact of $\sigma$ (the data-dependent factor) on $f_{\tilde{s}X}(t)$ within a reasonable range. Table~\ref{tab:variation-of-epsilon-lamda} in Appendix~\ref{sec:experiment-in-appendix} lists the values of $\epsilon$ and $\lambda$ for different $m$ over 12 datasets, where $\epsilon$ and $\lambda$ are calculated using the maximum, mean and minimum $\sigma$, respectively. As expected, $\epsilon$ and $\lambda$ decrease as $m$ increases. Take \emph{Trevi} as an example, $\epsilon$ is equal to 0 and $\lambda$ is very tiny (0.0001-0.000729) when $m=n$, manifesting the equivalence between $f_{\tilde{s}X}(t)$ and the PDF of $\mathcal{N}(0,\frac{ms^{2}}{n})$ as depicted in Figure~\ref{fig:pdf-with-different-m-for-trevi} (b-4). In the case of $m=30$, it also suffices to provide small enough $\epsilon$ and $\lambda$, indicating the high similarity between the distribution of $\tilde{s}X$ and $\mathcal{N}(0,\frac{ms^{2}}{n})$.


To visualize the similarity, we plot $f_{\tilde{s}X}(t)$ for different $m$ under the maximum and minimum $\sigma$, and the PDF of $\mathcal{N}(0,\frac{ms^{2}}{n})$ in Figure~\ref{fig:pdf-with-different-m-for-trevi} for two datasets \emph{Gist} and \emph{Trevi}. More plots for other datasets are listed in Figure~\ref{fig:pdf-curve-of-all-dataset} in Appendix~\ref{sec:experiment-in-appendix}. Three observations can be made from these figures: (1) the distribution of ${\tilde{s}X}$ matches very well with $\mathcal{N}(0,\frac{ms^{2}}{n})$ for small $\sigma$; (2) for large $\sigma$, $f_{\tilde{s}X}(t)$ differs only slightly from the PDF of $\mathcal{N}(0,\frac{ms^{2}}{n})$ for all $m$, indicating that $s$ is the dominating factor in $p(s,\sigma)$; (3) greater $m$ results in higher similarity between $f_{\tilde{s}X}(t)$ and $\mathcal{N}(0,\frac{ms^{2}}{n})$, implying that FastLSH can always achieve almost the same performance as E2LSH by choosing $m$ appropriately.

To further validate the LSH property of FastLSH, we compare the important parameter $\rho$ for FastLSH and E2LSH when $m =30$. $\rho$ is defined as the function of the approximation ratio $c$, i.e., $\rho(c)=log(1/p(s_{1}))/log(1/p(s_{2}))$, where $s_{1} = 1$ and $s_{2}=c$. Note that $\rho$ affects both the space and time efficiency of LSH algorithms. For $c$ in the range $[1,20]$ (with increments of 0.1), we calculate $\rho$ using \emph{Matlab}, where the minimal and maximal $\sigma$ are collected for different $c$ ($s$). The plots of $\rho(c)$ under different bucket widths over the datasets \emph{Gist} and \emph{Trevi} are illustrated in Figure~\ref{fig:derivative-rho-curve-of-Trevi}. Clearly, the $\rho(c)$ curve of FastLSH matches very well with that of E2LSH, verifying that FastLSH maintains almost the same LSH property with E2LSH even when $m$ is relatively small. More plots of $\rho(c)$ for other datesets are given in Figure~\ref{fig:rho-curve-for-Random-Audio-Cifar-Deep-Glove} and Figure~\ref{fig:rho-curve-for-image-notre-sift-sun-ukbench} in Appendix~\ref{sec:experiment-in-appendix}.

\begin{figure*}[t]
	\centering
	\begin{minipage}{0.27\textwidth}
        \centering
		\centerline{\includegraphics[width=\textwidth]{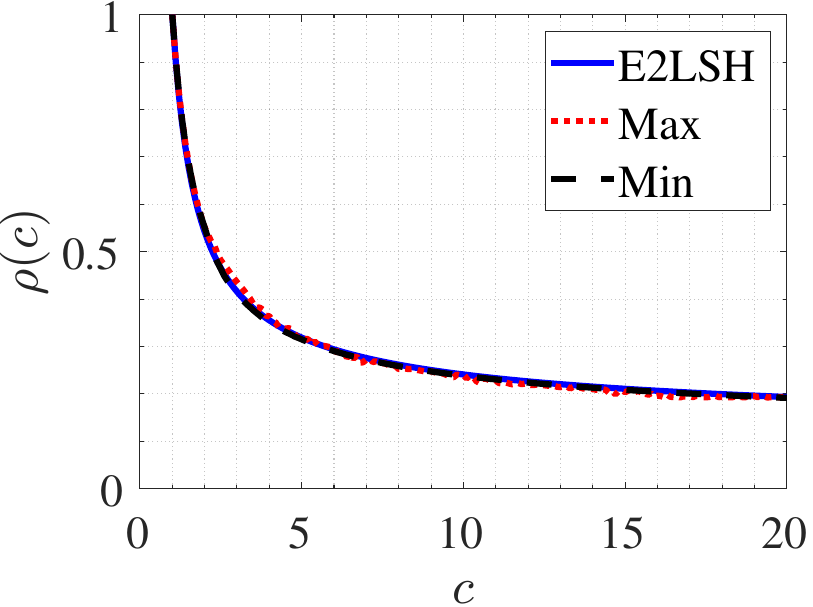}}
        \centerline{\scriptsize (a-1) Gist: $\tilde{w}$ = 0.255 and $w$ = 1.5}
	\end{minipage}
    \quad
	\begin{minipage}{0.27\textwidth}
        \centering
		\centerline{\includegraphics[width=\textwidth]{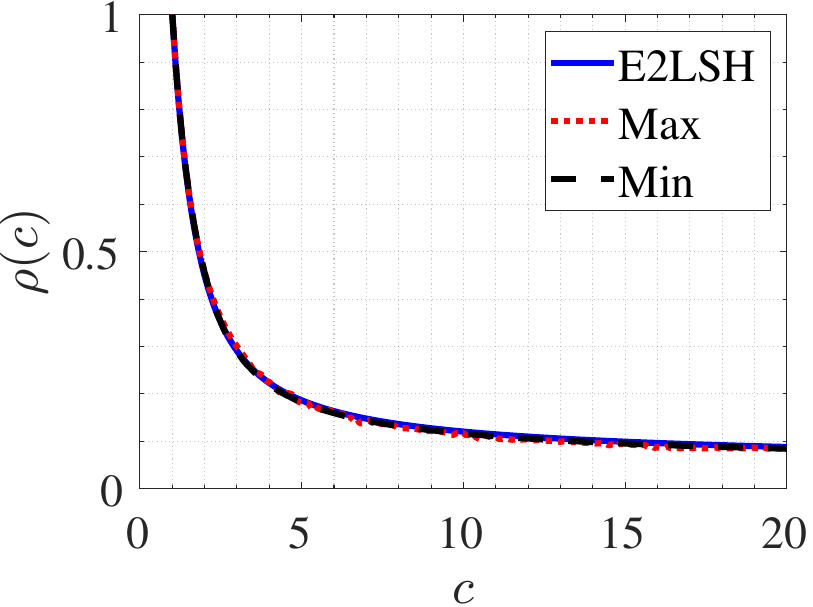}}
        \centerline{\scriptsize (a-2) Gist: $\tilde{w}$ = 0.73 and $w$ = 4}
	\end{minipage}
    \quad
	\begin{minipage}{0.27\textwidth}
        \centering
		\centerline{\includegraphics[width=\textwidth]{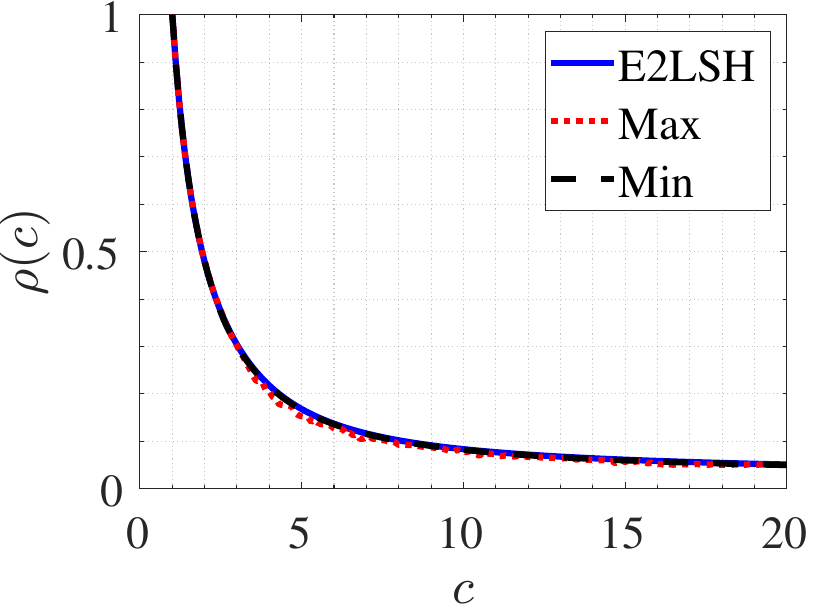}}
        \centerline{\scriptsize (a-3) Gist: $\tilde{w}$ = 1.8 and $w$ = 10}
	\end{minipage}
    \quad
     \begin{minipage}{0.27\textwidth}
        \centering
		\centerline{\includegraphics[width=\textwidth]{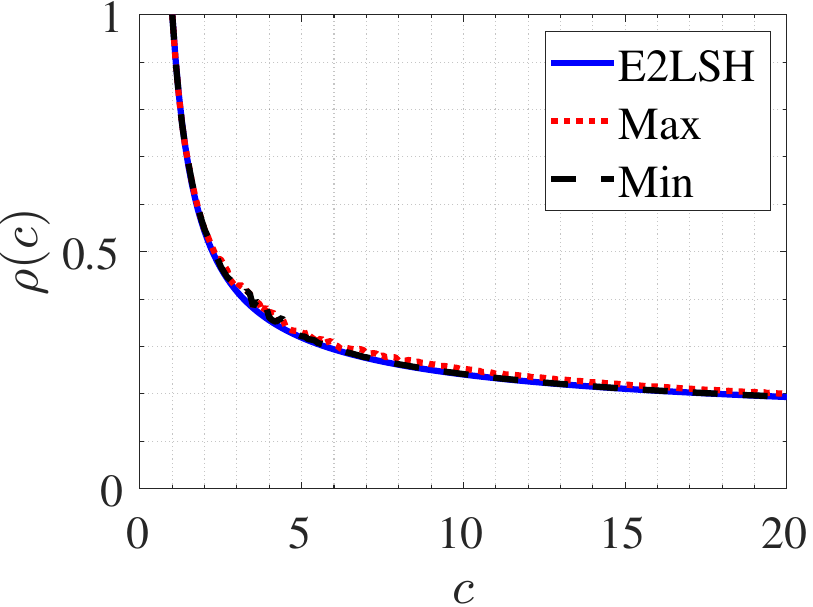}}
        \centerline{\scriptsize (b-1) Trevi: $\tilde{w}$ = 0.127 and $w$ = 1.5}
	\end{minipage}
    \quad
	\begin{minipage}{0.27\textwidth}
        \centering
		\centerline{\includegraphics[width=\textwidth]{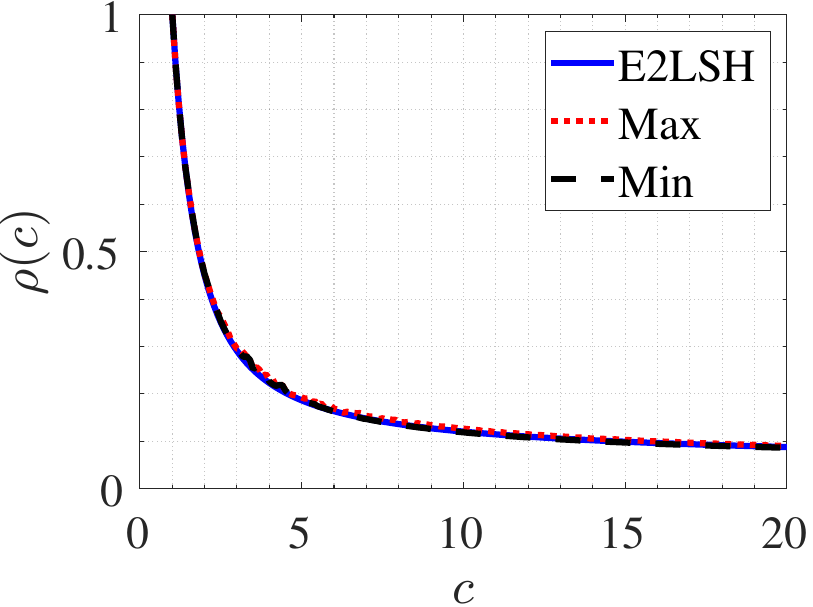}}
        \centerline{\scriptsize (b-2) Trevi: $\tilde{w}$ = 0.35 and $w$ = 4}
	\end{minipage}
    \quad
	\begin{minipage}{0.27\textwidth}
        \centering
		\centerline{\includegraphics[width=\textwidth]{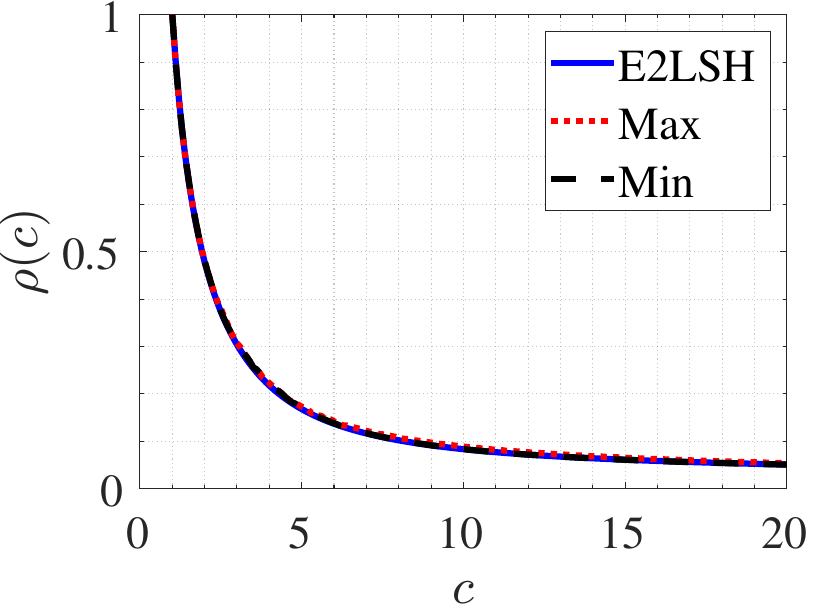}}
        \centerline{\scriptsize (b-3) Trevi: $\tilde{w}$ = 0.86 and $w$ = 10}
	\end{minipage}
    \caption{$\rho$ curves under different bucket widths over datasets \emph{Gist} and \emph{Trevi}} \label{fig:derivative-rho-curve-of-Trevi}
\end{figure*}

\section{Extension to Other Similarity Metrics}
\label{sec:extension-of-fastlsh}
In this section, we sketch how to extend FastLSH to other similarity measures. Since the angular similarity can be well approximated by the Euclidean distance if the norms of data item are identical, one can use FastLSH for the angular similarity directly after data normalization. In addition, FastLSH can solve the maximum inner product search problem by utilizing two transformation functions~\cite{MIPS}. The detailed discussion is given in Appendix~\ref{sec:extension-in-appendix}. The extension of FastLSH to support $l_{p}$ norm for $p\in(0,2)$ is left as our future work.

\section{Experiments}
\label{sec:experiment}
In this section we evaluate the performance of FastLSH against other LSH algorithms for the ANN search task. All algorithms follow the standard hash table construction and search procedure as discussed in Section~\ref{sec:ann-search-using-fastlsh}. All experiments are carried out on the server with six-cores Intel(R), i7-8750H @ 2.20GHz CPUs and 32 GB RAM, in Ubuntu 20.04.

\textbf{Datasets:} 11 publicly available high-dimensional real datasets and one synthetic dataset are experimented with~\cite{dataset}, the statistics of which are listed in Table~\ref{tab:statistics-of-dataset}. Sun~\footnote{http://groups.csail.mit.edu/vision/SUN/} is the set of containing about 70k GIST features of images. Cifar~\footnote{http://www.cs.toronto.edu/~kriz/cifar.html} is denoted as the set of 50k vectors extracted from TinyImage. Audio~\footnote{http://www.cs.princeton.edu/cass/audio.tar.gz} is the set of about 50k vectors extracted from DARPA TIMIT. Trevi~\footnote{http://phototour.cs.washington.edu/patches/default.htm} is the set of containing around 100k features of bitmap images. Notre~\footnote{http://phototour.cs.washington.edu/datasets/} is the set of features that are Flickr images and a reconstruction. Sift~\footnote{http://corpus-texmex.irisa.fr} is the set of containing 1 million SIFT vectors. Gist~\footnote{https://github.com/aaalgo/kgraph} is the set that is consist of 1 million image vectors. Deep~\footnote{https://yadi.sk/d/I\_yaFVqchJmoc} is the set of 1 million vectors that are deep neural codes of natural images obtained by convolutional neural network. Ukbench~\footnote{http://vis.uky.edu/~stewe/ukbench/} is the set of vectors containing 1 million features of images. Glove~\footnote{http://nlp.stanford.edu/projects/glove/} is the set of about 1.2 million feature vectors extracted from Tweets. ImageNet~\footnote{http://cloudcv.org/objdetect/} is the set of data points containing about 2.4 million dense SIFT features. Random is the set containing 1 million randomly selected vectors in a unit hypersphere.
\begin{table}[t]
 \caption{Statistics of Datasets}
 \centering
 \footnotesize
 \label{tab:statistics-of-dataset}
    \begin{tabular}{c|c|c|c}
    \hline\hline
    Datasets & \# of Points & \# of Queries & Dimension \\ \hline\hline
    Sun       & 69106   & 200 & 512   \\        \hline
    Cifar     & 50000   & 200 & 512   \\        \hline
    Audio     & 53387   & 200 & 192   \\        \hline
    Trevi     & 99000   & 200 & 4096  \\        \hline
    Notre     & 333000  & 200 & 128   \\        \hline
    Sift      & 1000000 & 200 & 128   \\        \hline
    Gist      & 1000000 & 200 & 960   \\        \hline
    Deep      & 1000000 & 200 & 256   \\        \hline
    Ukbench   & 1000000 & 200 & 128   \\        \hline
    Glove     & 1192514 & 200 & 100   \\        \hline
    ImageNet  & 2340000 & 200 & 150   \\        \hline
    Random    & 1000000 & 200 & 100   \\        \hline\hline
    \end{tabular}
\end{table}

\textbf{Baselines:} We compare FastLSH with two popular LSH algorithms, i.e., E2LSH~\cite{pslsh,E2LSH-package} and ACHash~\cite{DHHash}. E2LSH is the vanilla LSH scheme for approximate near neighbor search with sub-linear query time. ACHash is proposed to speedup the hash function evaluation by using Hadamard transform and sparse random projection. Note that ACHash is actually not an eligible LSH method because no expression of the probability of collision exists for ACHash, not mentioning the desirable LSH property. We choose ACHash as one of the baselines for the sake of completeness.

\begin{figure*}[t]
	\centering
	\begin{minipage}{0.27\textwidth}
		\centering
		\centerline{\includegraphics[width=\textwidth]{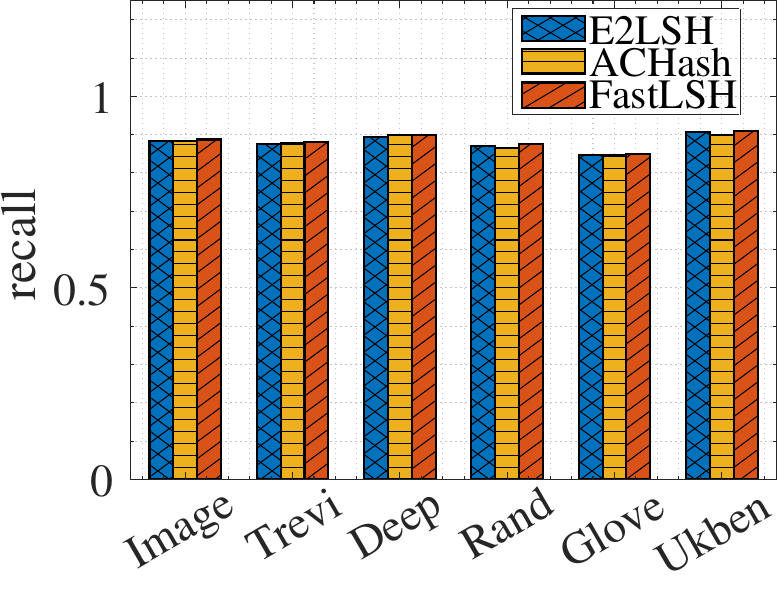}}
        \centerline{\scriptsize (a)}
	\end{minipage}
    \quad
    \begin{minipage}{0.27\textwidth}
		\centering
		\centerline{\includegraphics[width=\textwidth]{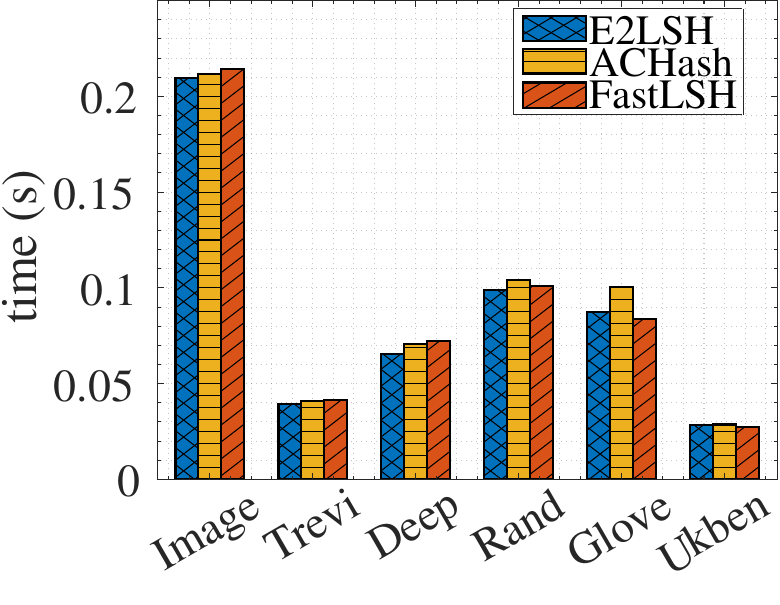}}
        \centerline{\scriptsize (b)}
	\end{minipage}
    \quad
    \begin{minipage}{0.27\textwidth}
		\centering
		\centerline{\includegraphics[width=\textwidth]{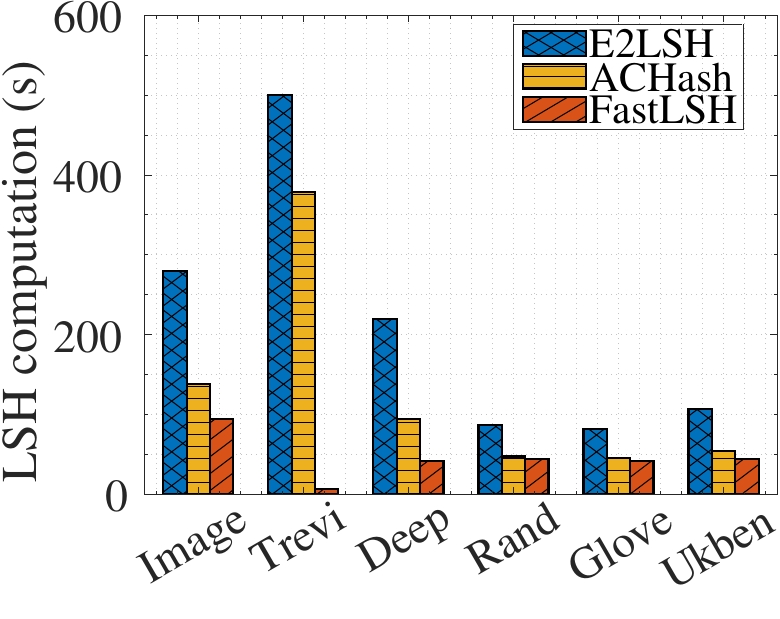}}
        \centerline{\scriptsize (c)}
	\end{minipage}
    \quad
    \begin{minipage}{0.27\textwidth}
        \centering
		\centerline{\includegraphics[width=\textwidth]{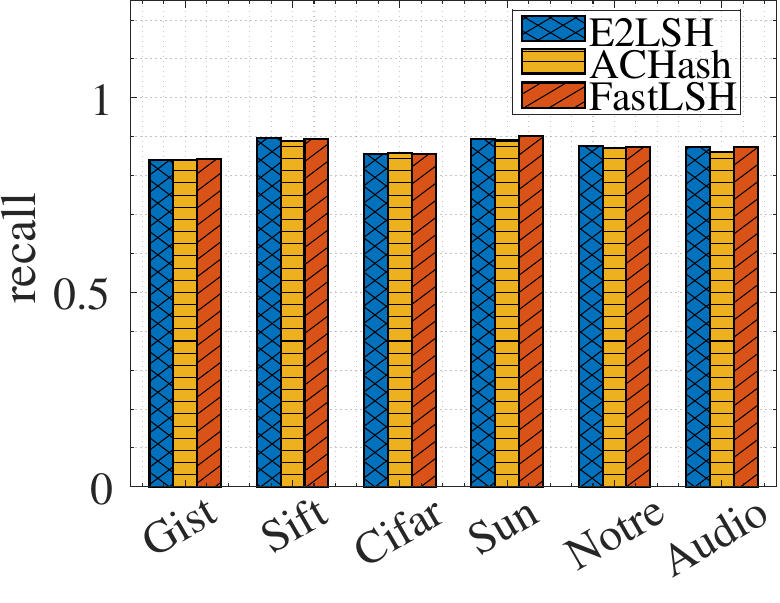}}
        \centerline{\scriptsize (d)}
	\end{minipage}
    \quad
    \begin{minipage}{0.27\textwidth}
        \centering
		\centerline{\includegraphics[width=\textwidth]{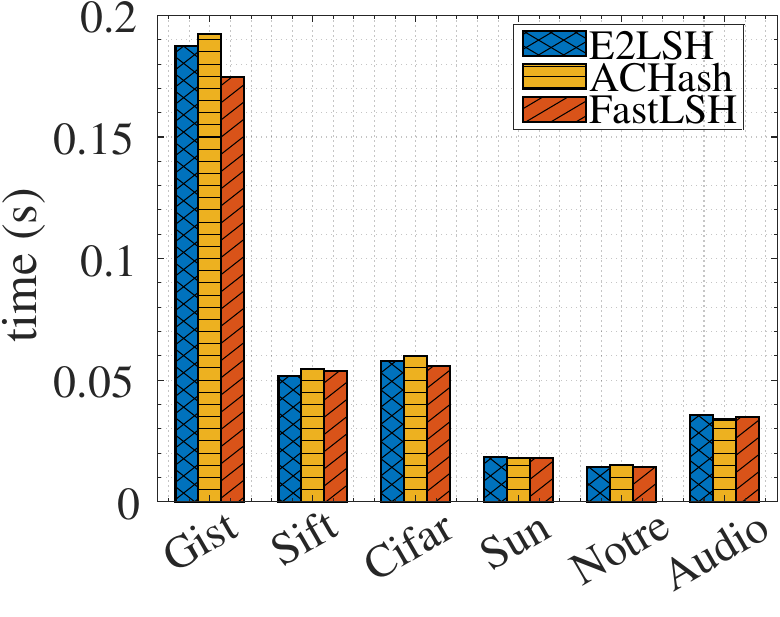}}
        \centerline{\scriptsize (e)}
	\end{minipage}
    \quad
    \begin{minipage}{0.27\textwidth}
        \centering
		\centerline{\includegraphics[width=\textwidth]{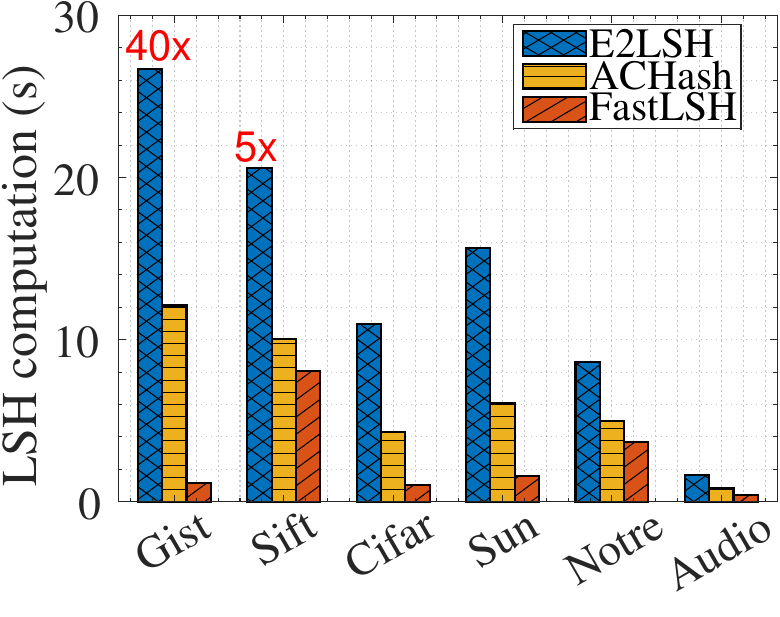}}
        \centerline{\scriptsize (f)}
	\end{minipage}
    \caption{Comparison of recall, average query time and LSH computation time.} \label{fig:time-recall-lsh-computation}
\end{figure*}

\textbf{Evaluation Metrics:} To evaluate the performance of FastLSH and baselines, we present the following metrics: 1) the average recall, i.e., the fraction of near neighbors that are actually returned; 2)  the average running time to report the near neighbors for each query; 3) the time taken to compute hash functions.

\textbf{Parameter Settings:} For the same dataset and target recall, we use identical $k$ (number of hash functions per table) and $L$ (number of hash tables) for fairness. Thus, three algorithms take the same space occupations. $m$ is set to 30 throughout all experiments for FastLSH. To achieve optimal performance, the sampling ratio for ACHash is set to the defaulst 0.25. We report R10@10 over all of the datasets for three algorithms.

\textbf{Experiment 1:} In this set of experiments, we set the target recall as 0.9. By tuning $L$, $k$ and bucket widths, the three methods achieve almost the same recall on the same dataset with the optimal query time~\footnote{The actual recall may vary around 0.9 slightly}. The recall, average query time and LSH computation time are illustrated in Figure~\ref{fig:time-recall-lsh-computation} (a), (b), (d) and (e). As analyzed in Section~\ref{sec:the-lsh-property-of-fastlsh}, FastLSH and E2LSH owns similar theoretical property, and thus achieve comparable query performance and answer accuracy as plotted in Figure~\ref{fig:time-recall-lsh-computation}. Due to lack of theoretical guarantee, ACHash performs slightly worse than FastLSH and E2LSH in most cases w.r.t query efficiency.

The performance of the three methods differs dramatically when it turns to the cost of hashing. As shown in Figure~\ref{fig:time-recall-lsh-computation} (c) and (f), the LSH computation time of FastLSH is significantly superior to E2LSH and ACHash. For example, FastLSH obtains around 80 speedup over E2LSH and runs 60 times faster than ACHash on \emph{Trevi}. This is because the time complexity of FastLSH is only $O(m)$ instead of $O(n)$ of E2LSH. For ACHash, the fixed sampling ratio and overhead in Hadamard transform make it inferior to FastLSH.

\textbf{Experiment 2:} We also plot the recall v.s. average query time curves by varying target recalls to obtain a complete picture of FastLSH, which are depicted in Figure~\ref{fig:time-recall-curve-in-appendix} in Appendix~\ref{sec:experiment-in-appendix} due to space limitation. The empirical results demonstrate that FastLSH performs almost the same in terms of answer accuracy, query efficiency and space occupation as E2LSH. This further validates that FastLSH owns the same LSH property as E2LSH with limited $m < n$. Again, ACHash is slightly inferior to the others in most cases.

In sum, FastLSH is on par with E2LSH (with provable LSH property) in terms of answer accuracy and query efficiency and marginally superior to ACHash (without provable LSH property), while significantly reducing the cost of hashing. We believe that FastLSH is a promising alternative to the existing hashing schemes for $l_2$ norm.


\section{Conclusion}
In this paper, we develop FastLSH to accelerate hash function evaluation, which maintains the same theoretical guarantee and empirical performance as the classic E2LSH. Rigid analysis shows that the probability of collision of FastLSH is asymptotically equal to that of E2LSH. In the case of limited $m$, we quantitatively analyze the impact of $\sigma$ and $m$ on the probability of collision. Extensive experiments on a number of datasets demonstrate that FastLSH is a promising alternative to the classic LSH scheme.


\nocite{langley00}

\bibliography{FasterLSH}
\bibliographystyle{icml2023}

\newpage
\appendix
\onecolumn

\begin{spacing}{1.0}
    \begin{center}
      \textbf{\LARGE Appendix}
 \end{center}
    \end{spacing}

\section{Proofs}
\label{sec:main-lemma-theorem-in-appendix}

 \begin{lemma}
 \label{lemma:characteristic-func-of-w-in-appendix}
 The characteristic function of the product of two independent random variables $W=XY$ is as follows:
 \begin{equation}
 \nonumber
    \varphi_{W}(x)=E_{Y}\{\exp(-\frac{x^{2}Y^{2}}{2})\}
 \end{equation}
 where $X$ is a standard normal random variable and $Y$ is an independent random variable with mean $\mu$ and variance $\sigma^{2}$.
\end{lemma}

\begin{proof}
 For the characteristic function of $W$, we can write:
 \begin{equation}
 \nonumber
 \begin{aligned}
 \varphi_{W}(x)&=E_{W}\{\exp(ixW)\}\\
 &=E_{XY}\{\exp(ixXY)\} \\
 &=E_{Y}\{E_{X|Y}\{\exp(ixXY)|Y\}\}\\
 &= \int_{-\infty}^{+\infty}E_{X|Y}\{\exp(-ixXY)|Y\}f(Y)dY \\
 &= \int_{-\infty}^{+\infty}(\int_{-\infty}^{+\infty}\exp(-ixXY)f(X|Y)dX) f(Y)dY \\
 &= \int_{-\infty}^{+\infty}\exp(-\frac{x^{2} Y^{2}}{2})f(Y)dY \\
 &= E_{Y}\{\exp(-\frac{x^{2} Y^{2}}{2})\}
 \end{aligned}
 \end{equation}
 where standard normal random variable $X$ is eliminated by its characteristic function. We prove this Lemma.
\end{proof}

\begin{lemma}
\label{lemma:characteristic-func-of-tildesx-in-appendix}
 The characteristic function of $\tilde{s}X$ is as follows:
 \begin{equation}
 \nonumber
 \varphi_{\tilde{s}X}(x)=\frac{1}{2(1-\Phi(\frac{-\tilde{\mu}}{\tilde{\sigma}}))} \exp(\frac{1}{8} x^4 \tilde{\sigma}^2-\frac{1}{2} \tilde{\mu} x^2) \operatorname{erfc}(\frac{\frac{1}{2} x^2 \tilde{\sigma}^2-\tilde{\mu}}{\sqrt{2} \tilde{\sigma}}) \quad (-\infty <x< +\infty)
 \end{equation}
 where $\operatorname{erfc}(t) = \frac{2}{\sqrt{\pi}}\int_{t}^{+\infty}\exp(-x^{2})dx$ $(-\infty <t< +\infty)$ is the complementary error function.
\end{lemma}

\begin{proof}
  According to Eqn.~\eqref{eqn:pdf-of-s}, we know the PDF of $\tilde{s}X$. By applying Lemma~\ref{lemma:characteristic-func-of-w-in-appendix}, we have the following result:
\begin{equation}
    \nonumber
    \begin{aligned}
    \varphi_{\tilde{s}X}(x) & = \frac{1}{\sqrt{2 \pi} \tilde{\sigma}} \int_0^{+\infty} \frac{2 y}{\Phi(a_{2};\tilde{\mu}, \tilde{\sigma}^2)-\Phi(a_{1};\tilde{\mu}, \tilde{\sigma}^2)} \exp (\frac{-x^2 y^2}{2}-\frac{(y^2-\tilde{\mu})^2}{2 \tilde{\sigma}^2})dy \\
    & =\frac{1}{\sqrt{2 \pi} \tilde{\sigma}} \int_0^{+\infty} \frac{1}{\Phi(a_{2}; \tilde{\mu}, \tilde{\sigma}^2)-\Phi(a_{1};\tilde{\mu}, \tilde{\sigma}^2)} \exp (-\frac{x^2 y^2}{2}-\frac{(y^2-\tilde{\mu})^2}{2 \tilde{\sigma}^2})dy^2 \\
    & =\frac{1}{\sqrt{2 \pi} \tilde{\sigma}} \int_0^{+\infty} \frac{1}{\Phi(a_{2};\tilde{\mu}, \tilde{\sigma}^2)-\Phi(a_{1};\tilde{\mu}, \tilde{\sigma}^2)} \exp (\frac{-(y^2-(\tilde{\mu}-\frac{1}{2} x^2 \tilde{\sigma}^2))^2-\tilde{\mu} x^2 \tilde{\sigma}^2+\frac{1}{4}x^4 \tilde{\sigma}^4)}{2 \tilde{\sigma}^2})dy^2 \\
    & =\frac{1}{\sqrt{2 \pi} \tilde{\sigma}} \int_0^{+\infty} \frac{1}{\Phi(a_{2};\tilde{\mu}, \tilde{\sigma}^2)-\Phi(a_{1};\tilde{\mu}, \tilde{\sigma}^2)} \exp (\frac{1}{8} x^4 \tilde{\sigma}^2-\frac{1}{2} \tilde{\mu} x^2) \exp(\frac{-(y^2-(\tilde{\mu}-\frac{1}{2} x^2 \tilde{\sigma}^2))^2)}{2 \tilde{\sigma}^2}) dy^2\\
    & =\frac{1}{2(\Phi(\frac{a_{2}-\tilde{\mu}}{\tilde{\sigma}})-\Phi(\frac{a_{1}-\tilde{\mu}}{\tilde{\sigma}}))} \exp(\frac{1}{8} x^4 \tilde{\sigma}^2-\frac{1}{2} \tilde{\mu} x^2)  \operatorname{erfc}(\frac{\frac{1}{2} x^2 \tilde{\sigma}^2-\tilde{\mu}}{\sqrt{2} \tilde{\sigma}})  \\
    & =\frac{1}{2(1-\Phi(\frac{-\tilde{\mu}}{\tilde{\sigma}}))} \exp(\frac{1}{8} x^4 \tilde{\sigma}^2-\frac{1}{2} \tilde{\mu} x^2) \operatorname{erfc}(\frac{\frac{1}{2} x^2 \tilde{\sigma}^2-\tilde{\mu}}{\sqrt{2} \tilde{\sigma}}) \\
    &
    \end{aligned}
\end{equation}
where $\tilde{\mu}=\frac{ms^{2}}{n}$ and $\tilde{\sigma}^{2}=m\sigma^{2}$. Hence we prove this Lemma.
\end{proof}

\begin{theorem}
\label{theorem:collision-prob.-of-FastLSH-in-appendix}
 The collision probability of FastLSH is as follows:
 \begin{equation}
 \nonumber
 p(s,\sigma)=Pr[h_{\tilde{\textbf{a}},\tilde{b}}(\textbf{v})=h_{\tilde{\textbf{a}},\tilde{b}}(\textbf{u})]=\int_{0}^{\tilde{w}}f_{|\tilde{s}X|}(t)(1-\frac{t}{\tilde{w}})dt
 \end{equation}
\end{theorem}

\begin{proof}
Let $f_{|\tilde{s}X|}(t)$ represent the PDF of the absolute value of $\tilde{s}X$. For given bucket width $\tilde{w}$, the probability $p(|\tilde{s}X| < t)$ for any pair $(\textbf{\emph{v}},\textbf{\emph{u}})$ is computed as $p(|\tilde{s}X| < t) = \int_{0}^{\tilde{w}}f_{|\tilde{s}X|}(t)dt$, where $t\in[0,\tilde{w}]$. Recall that $\tilde{b}$ follows the uniform distribution $U(0,\tilde{w})$, the probability $p(\tilde{b} < \tilde{w} - t)$ is thus $(1-\frac{t}{\tilde{w}})$. This means that after random projection, $(|\tilde{s}X| + \tilde{b})$ is also within the same bucket, and the collision probability $p(s,\sigma)$ is the product of $p(|\tilde{s}X| < t)$ and $p(\tilde{b} < \tilde{w} - t)$. Hence we prove this Theorem.
\end{proof}

\begin{theorem}
\label{theorem:equivalence-of-characteristic-function-in-appendix}
\begin{equation}
\nonumber
\lim_{m \to +\infty} \frac{\varphi_{\tilde{s}X}(x)}{\exp(-\frac{ms^{2}x^{2}}{2n})} = 1
\end{equation}
where $|x|\leq O(m^{-1/2})$ and  $\exp(-\frac{ms^{2}x^{2}}{2n})$ is the characteristic function of $\mathcal{N}(0,\frac{ms^{2}}{n})$.
\end{theorem}

\begin{proof}
Recall that $\tilde{\mu} = \frac{ms^{2}}{n}$ and $\tilde{\sigma}=m\sigma^{2}$. $\varphi_{\tilde{s}X}(x)$ can be written as follows:
\begin{equation}
\nonumber
\varphi_{\tilde{s}X}(x) =\frac{1}{2(1-\Phi(\frac{-\sqrt{m}s^{2}}{n\sigma}))} \exp(\frac{mx^4\sigma^2}{8} -\frac{ms^{2}x^2}{2n}) \operatorname{erfc}(\frac{\sqrt{m}(nx^{2}\sigma^{2}-2s^{2})}{2\sqrt{2}n\sigma}) \quad (-\infty <x< +\infty)
\end{equation}
Let $a=\frac{s^{2}}{\sqrt{2}n\sigma} > 0$ and $b=\frac{\sigma}{2\sqrt{2}} > 0 \Rightarrow b^{2} = \frac{\sigma^{2}}{8}$. According to the fact $\Phi(x) = \frac{1}{2}\operatorname{erfc}(-\frac{x}{\sqrt{2}})$, $\varphi_{\tilde{s}X}(x)$ is  simplified as:
\begin{equation}
\nonumber
  \varphi_{\tilde{s}X}(x) = \exp(b^{2}mx^{4}-\frac{1}{2}m\mu x^{2}) \cdot \frac{\operatorname{erfc}(b\sqrt{m}x^{2}-a\sqrt{m})}{2-\operatorname{erfc}(a\sqrt{m})}
\end{equation}
Let $g(x) = \frac{\varphi_{\tilde{s}X}(x)}{\varphi_{sX}(x)}$, where $\varphi_{sX}(x) = \exp(-\frac{ms^{2}x^{2}}{2n})$ is the characteristic function of $\mathcal{N}(0,\frac{ms^{2}}{n})$. Then $g(x)$ is denoted as:
\begin{equation}
    \nonumber
    g(x) = \exp(b^{2}mx^{4}) \cdot \frac{\operatorname{erfc}(b\sqrt{m}x^{2}-a\sqrt{m})}{2-\operatorname{erfc}(a\sqrt{m})}
\end{equation}
To prove $\varphi_{\tilde{s}X}(x) = \varphi_{sX}(x)$, we convert to prove whether $g(x) = 1$ as $m\to \infty$. Obviously $x^{2} \leq O(m^{-1})$. Then $\sqrt{m}x^{2} \leq O(m^{-1/2})$ and $mx^{4} \leq O(m^{-1})$. It is easy to derive:
\begin{equation}
\nonumber
\lim_{m \to +\infty} \exp(b^{2} mx^{4} )=1
\end{equation}
It holds for any fixed $b\in \mathbb{R}^{+}$. On the other hand, we have:
\begin{equation}
\nonumber
\operatorname{erfc}(b\sqrt{m}x^{2}-a\sqrt{m})  \sim  \operatorname{erfc}(-a\sqrt{m}), \quad m \to +\infty
\end{equation}
Actually using the fact $b\sqrt{m}x^{2}-a\sqrt{m} \sim -a\sqrt{m}$ as $m\to +\infty$ and $\operatorname{erfc}(-x) = 2-\frac{\exp(-x^{2})}{\sqrt{\pi}x}$ as $x \to +\infty$, we have:
\begin{equation}
\nonumber
 \lim_{m \to +\infty} \frac{\operatorname{erfc}(b\sqrt{m}x^{2} -a\sqrt{m})}{2-\operatorname{erfc}(a\sqrt{m})} = \frac{\lim_{m \to +\infty}  \operatorname{erfc}(b\sqrt{m}x^{2} -a\sqrt{m}) }{\lim_{m \to +\infty} \operatorname{erfc}(-a\sqrt{m})} =1
\end{equation}
Hence we have:
\begin{equation}
    \nonumber
    g(x) = \lim_{m \to +\infty} \exp(b^{2} mx^{4} ) \cdot \lim_{m \to +\infty} \frac{\operatorname{erfc}(b\sqrt{m}x^{2}-a\sqrt{m})}{\operatorname{erfc}(-a\sqrt{m})} = 1.
\end{equation}

\end{proof}



If the characteristic function of a random variable $Z$ exists, it provides a way to compute its various moments. Specifically, the $r$-th moment of $Z$ denoted by $E(Z^{r})$ can be expressed as the $r$-th derivative of the characteristic function evaluated at zero ~\cite{computeMoments}, i.e.,
\begin{equation}
\label{eqn:compute-momnet-with-characteristic-function}
E(Z^{r})=(i)^{-r}\frac{d^{r}}{dt^{r}} \varphi_{Z}(t)\mid _{t=0}
\end{equation}
where $\varphi_{Z}(t)$ denotes the characteristic function of $Z$. If we know the characteristic function, then all of the moments of the random variable $Z$ can be obtained.

\begin{lemma}
\label{lemma:four-moment-of-FastLSH-in-appendix}
\begin{equation}
\nonumber
\begin{cases}
\ \ E(\tilde{s}X) \ \ \ = 0
 \\
E((\tilde{s}X)^{2}) = \frac{ms^{2}}{n}(1+\epsilon)
    \\
E((\tilde{s}X)^{3})  = 0
\\
E((\tilde{s}X)^{4}) = \frac{3m^{2}s^{4}}{n^{2}}(1+\lambda)
\end{cases}
\end{equation}
where $\epsilon=\frac{\tilde{\sigma} \exp(\frac{-\tilde{\mu}^{2}}{2\tilde{\sigma}^{2}})}{\sqrt{2\pi}\tilde{\mu} (1-\Phi(\frac{-\tilde{\mu}}{\tilde{\sigma}}))}$ and $\lambda = \frac{\tilde{\sigma}^{2}}{\tilde{\mu}^{2}} + \epsilon$.
\end{lemma}

\begin{proof}
  From Lemma~\ref{lemma:characteristic-func-of-tildesx-in-appendix}, we can easily compute the first-order derivative of characteristic function with respect to $x$, which is as follows:
\begin{equation}
\label{eqn:first-order-dervative-of-characteristic-fuction}
\varphi^{\prime}_{\tilde{s}X}(x)= \frac{\exp(\frac{1}{8} x^4 \tilde{\sigma}^2-\frac{1}{2} \tilde{\mu} x^2)}{2(1-\Phi(\frac{-\tilde{\mu}}{\tilde{\sigma}}))} \left[ (\frac{1}{2}x^{3}\tilde{\sigma}^{2}-\tilde{\mu}x)\operatorname{erfc}(\frac{\frac{1}{2} x^2 \tilde{\sigma}^2-\tilde{\mu}}{\sqrt{2} \tilde{\sigma}})-\frac{2\tilde{\sigma}x}{\sqrt{2\pi}}\exp(-(\frac{\frac{1}{2} x^2 \tilde{\sigma}^2-\tilde{\mu}}{\sqrt{2} \tilde{\sigma}})^{2})\right]
\end{equation}
Then the second-order derivative is
\begin{equation}
\label{eqn:second-order-dervative-of-characteristic-fuction}
\begin{aligned}
    \varphi^{\prime \prime}_{\tilde{s}X}(x) = & \frac{\exp(\frac{1}{8} x^4 \tilde{\sigma} ^2-\frac{1}{2} \tilde{\mu} x^2)}{2(1-\Phi(\frac{-\tilde{\mu}}{\tilde{\sigma}}))}
    \left [ ((\frac{1}{2}x^{3} \tilde{\sigma}^{2} -ux)^{2}+\frac{3}{2} x^{2} \tilde{\sigma}^{2}-\tilde{\mu}) \right.
     \operatorname{erfc}(\frac{\frac{1}{2}x^2 \tilde{\sigma}^2-\tilde{\mu}}{\sqrt{2} \tilde{\sigma}}) \\ & \left. -\frac{2}{\sqrt{2\pi}}(\tilde{\sigma}+\frac{3}{2}x^{4} \tilde{\sigma}^{3} -\frac{3}{2}\tilde{\mu}\tilde{\sigma} x^{2})\exp(-(\frac{\frac{1}{2} x^2 \tilde{\sigma}^2-\tilde{\mu}}{\sqrt{2} \tilde{\sigma}})^{2})\right ]
    \end{aligned}
\end{equation}

The third-order derivative is
\begin{equation}
\label{eqn:third-order-dervative-of-characteristic-fuction}
  \begin{aligned}
  \varphi^{\prime \prime \prime } (x) = & \frac{\exp(\frac{1}{8} x^4 \tilde{\sigma} ^2-\frac{1}{2} \tilde{\mu} x^2)}{2(1-\Phi(\frac{-\tilde{\mu}}{\tilde{\sigma}}))} \left[(\frac{1}{8}\tilde{\sigma}^{6} x^{9}- \frac{3}{4} \tilde{\mu}\tilde{\sigma}^{4} x^{7} + (\frac{2}{3}\tilde{\mu}^{2} \tilde{\sigma}^{2} + \frac{9}{4} \tilde{\sigma}^{4} )x^{5} -(6\tilde{\mu}\tilde{\sigma}^{2} +\tilde{\mu}^{3} )x^{3} +3(\tilde{\mu}^{2} +\tilde{\sigma}^{2})x) \right. \\ & \left. \operatorname{erfc}(\frac{\frac{1}{2}\tilde{\sigma}^{2} x^{2} -\tilde{\mu}}{\sqrt{2} \tilde{\sigma}})
  -\frac{2}{\sqrt{2\pi}} (\frac{1}{4} \tilde{\sigma}^{5} x^{7} -\tilde{\mu}\tilde{\sigma}^{3} x^{5} +(\tilde{\mu}^{2} \tilde{\sigma}+\frac{7}{2} \tilde{\sigma}^{3} )x^{3} -3\tilde{\mu}\tilde{\sigma} x)\exp(-(\frac{\frac{1}{2}\tilde{\sigma}^{2} x^{2} -\tilde{\mu}}{\sqrt{2} \tilde{\sigma}})^{2} ) \right]
  \end{aligned}
\end{equation}

The fourth-order derivative is
\begin{equation}
\label{eqn:fourth-order-dervative-of-characteristic-fuction}
    \begin{aligned}
    \varphi^{\prime \prime \prime \prime} (x) = & \frac{\exp(\frac{1}{8} x^4 \tilde{\sigma} ^2-\frac{1}{2} \tilde{\mu} x^2)}{2(1-\Phi(\frac{-\tilde{\mu}}{\tilde{\sigma}}))} \left[(\frac{1}{16}\tilde{\sigma}^{8} x^{12} -\frac{1}{2}\tilde{\mu}\tilde{\sigma}^{6} x^{10} +(\frac{3}{2}\tilde{\mu}^{2} \tilde{\sigma}^{4} +\frac{9}{4}\tilde{\sigma}^{4} )x^{8} -(2\tilde{\mu}^{3} \tilde{\sigma}^{2} +\frac{21}{2}\tilde{\mu}\tilde{\sigma}^{4} )x^{6} \right.\\ & \left.
    + (\tilde{\mu}^{4} +\frac{51}{4}\tilde{\sigma}^{4} +9\tilde{\mu}^{2} \tilde{\sigma}^{2})x^{4}-(6\tilde{\mu}^{3} +21\tilde{\mu}\tilde{\sigma}^{2} )x^{2}
    + 3(\tilde{\mu}^{2} +\tilde{\sigma}^{2}))\operatorname{erfc}(\frac{\frac{1}{2}\tilde{\sigma}^{2} x^{2} -\tilde{\mu}}{\sqrt{2} \tilde{\sigma}}) \right. \\ & \left.
    -\frac{2}{\sqrt{2\pi}} (\frac{1}{8} \tilde{\sigma}^{7} x^{10} -\frac{3}{4} \tilde{\mu}\tilde{\sigma}^{5} x^{8} +(\frac{3}{2} \tilde{\mu}^{2} \tilde{\sigma}^{3}  +4\tilde{\sigma}^{5} )x^{6} -(6\tilde{\mu}^{2} \tilde{\sigma}+\frac{13}{2} \tilde{\sigma}^{3} )x^{2} -3\tilde{\mu}\tilde{\sigma})\exp(-(\frac{\frac{1}{2}\tilde{\sigma}^{2} x^{2} -\tilde{\mu}}{\sqrt{2} \tilde{\sigma}})^{2} )\right]
    \end{aligned}
\end{equation}

Let $E(\tilde{s}X-E(\tilde{s}X))^{i}$ for $i\in\{1,2,3,4\}$ denote the first four central moments. According to Eqn.~\eqref{eqn:compute-momnet-with-characteristic-function}, we know that $E(\tilde{s}X)=\frac{\varphi^{\prime}(0)}{i} = 0$, then it is easy to derive $E(\tilde{s}X-E(\tilde{s}X))^{i}=E((\tilde{s}X)^{i})$. To this end, by Eqn.~\eqref{eqn:first-order-dervative-of-characteristic-fuction} -~\eqref{eqn:fourth-order-dervative-of-characteristic-fuction}, we have the following results:
\begin{equation}
\label{eqn:first-four-moment}
\begin{cases}
\ \ E(\tilde{s}X) \ \ \ = 0
 \\
E((\tilde{s}X)^{2}) = \frac{\tilde{\mu}\operatorname{erfc}(\frac{-\tilde{\mu}}{\sqrt{2}\tilde{\sigma}})+
    \frac{2\tilde{\sigma}}{\sqrt{2\pi}}\exp(\frac{-\tilde{\mu}^{2}}{2\tilde{\sigma}^{2}})}{2(1-\Phi(\frac{-\tilde{\mu}}{\tilde{\sigma}}))}
    \\
E((\tilde{s}X)^{3})  = 0
\\
E((\tilde{s}X)^{4}) = \frac{3(\tilde{\mu}^{2}+\tilde{\sigma}^{2})\operatorname{erfc}(\frac{-\tilde{\mu}}{\sqrt{2}\tilde{\sigma}})+
    \frac{6\tilde{\mu}\tilde{\sigma}}{\sqrt{2\pi}}\exp(\frac{-\tilde{\mu}^{2}}{2\tilde{\sigma}^{2}})}{2(1-\Phi(\frac{-\tilde{\mu}}{\tilde{\sigma}}))}
\end{cases}
\end{equation}
where $E(\tilde{s}X)$ is the expectation; $E((\tilde{s}X)^{2})$ is the variance; $\frac{E((\tilde{s}X)^{3})}{(E((\tilde{s}X)^{2}))^{\frac{3}{2}}}$ is the skewness; $\frac{E((\tilde{s}X)^{4})}{(E((\tilde{s}X)^{2}))^{2}}$ is the kurtosis. Since
\begin{equation}
\nonumber
\frac{\operatorname{erfc}(\frac{-\tilde{\mu}}{\sqrt{2}\tilde{\sigma}})}{2(1-\Phi(\frac{-\tilde{\mu}}{\tilde{\sigma}}))} = \frac{\frac{2}{\sqrt{\pi}} \int_{\frac{-\tilde{\mu}}{\sqrt{2 \tilde{\sigma}}}}^{+\infty} \exp(-t_{1}^2)dt_{1}}{2(1-\frac{1}{\sqrt{2 \pi}} \int_{-\infty}^{\frac{-\tilde{\mu}}{\tilde{\sigma}}} \exp(\frac{-t_{2}^2}{2})dt_{2})} \\
= \frac{\int_{\frac{-\tilde{\mu}}{\sqrt{2} \tilde{\sigma}}}^{+\infty} \exp(-t_{1}^2)dt_{1}}{\sqrt{2} \int_{\frac{-\tilde{\mu}}{\tilde{\sigma}}}^{+\infty} \exp(\frac{-t_{2}^2}{2}) dt_{2}} \\
 = \frac{\int_{\frac{-\tilde{\mu}}{\sqrt{2} \tilde{\sigma}}}^{+\infty} \exp(-t_{1}^2)dt_{1}}{ \int_{\frac{-\tilde{\mu}}{\sqrt{2}\tilde{\sigma}}}^{+\infty} \exp(-t^2) dt} = 1
\end{equation}
Therefore, Eqn.~\eqref{eqn:first-four-moment} can be rewritten as below

\begin{equation}
\label{eqn:first-four-moment-final}
\begin{cases}
\nonumber
\ \ E(\tilde{s}X) \ \ \ = 0
 \\
E((\tilde{s}X)^{2}) = \tilde{\mu}(1+\epsilon)
    \\
E((\tilde{s}X)^{3})  = 0
\\
E((\tilde{s}X)^{4}) = 3\tilde{\mu}^{2}(1+\lambda)
\end{cases}
\end{equation}

where $\epsilon=\frac{\tilde{\sigma} \exp(\frac{-\tilde{\mu}^{2}}{2\tilde{\sigma}^{2}})}{\sqrt{2\pi}\tilde{\mu} (1-\Phi(\frac{-\tilde{\mu}}{\tilde{\sigma}}))}$, $\lambda = \frac{\tilde{\sigma}^{2}}{\tilde{\mu}^{2}} + \epsilon$, $\tilde{\mu}=\frac{ms^{2}}{n}$ and $\tilde{\sigma}^{2}=m\sigma^{2}$. We prove this Lemma.

\end{proof}

 \section{Extension to Maximum Inner Product Search}
\label{sec:extension-in-appendix}
\cite{MIPS, MIPS1} shows that there exists two transformation functions, by which the maximum inner product search (MIPS) problem can be converted into solve the near neighbor search problem. More specifically, the two transformation functions are $P(\textbf{\emph{v}}) = (\sqrt{\kappa^{2}-\|\textbf{\emph{v}}\|_{2}^{2}},\textbf{\emph{v}})$ for data processing and $Q(\textbf{\emph{u}}) = (0,\textbf{\emph{u}})$ for query processing respectively, where  $\kappa = max(\|\textbf{\emph{v}}_{i}\|_{2})$ $(i \in \{1,2,\ldots,N\})$. Then the relationship between maximum inner product and $l_{2}$ norm for any vector pair $(\textbf{\emph{v}}_{i},\textbf{\emph{u}})$ is denoted as $\underset{i}{argmax} (\frac{P(\textbf{\emph{v}}_{i} )Q(\textbf{\emph{u}})}{ \| P(\textbf{\emph{v}}_{i} ) \|_{2}\|Q(\textbf{\emph{u}})\|_{2}}) =\underset{i}{argmin} (\| P(\textbf{\emph{v}}_{i} )-Q(\textbf{\emph{u}}) \|_{2})$ for $\|\textbf{\emph{u}}\|_{2}=1$. To make FastLSH applicable for MIPS, we first apply the sample operator $S(\cdot)$ defined earlier to vector pairs $(\textbf{\emph{v}}_{i},\textbf{\emph{u}})$ for yielding $\tilde{\textbf{\emph{v}}}_{i} = S(\textbf{\emph{v}}_{i})$ and $\tilde{\textbf{\emph{u}}} = S(\textbf{\emph{u}})$, and then obtain $\tilde{P}(\textbf{\emph{v}}) = (\sqrt{\tilde{\kappa}^{2}-\|S(\textbf{\emph{v}})\|_{2}^{2}},S(\textbf{\emph{v}}))$ and $\tilde{Q}(\textbf{\emph{u}}) = (0,S(\textbf{\emph{u}}))$, where $\tilde{\kappa} = max(\|S(\textbf{\emph{v}}_{i})\|_{2})$ is a constant. Then $\underset{i}{argmax} (\frac{\tilde{P}(\textbf{\emph{v}}_{i} )\tilde{Q}(\textbf{\emph{u}})}{\| \tilde{P}(\textbf{\emph{v}}_{i} ) \|_{2}\|\tilde{Q}(\textbf{\emph{u}})\|_{2}}) =\underset{i}{argmin} ( \| \tilde{P}(\textbf{\emph{v}}_{i} )-\tilde{Q}(\textbf{\emph{u}}) \|_{2})$ for $\|S(\textbf{\emph{u}})\|_{2}=1$. Let $\triangle = \tilde{\kappa}^{2}-\|S(\textbf{\emph{v}})\|_{2}^{2}$. After random projection, $\textbf{\emph{a}}^{T}\tilde{P}(\textbf{\emph{v}})-\textbf{\emph{a}}^{T}\tilde{Q}(\textbf{\emph{u}})$ is distributed as $(\sqrt{\tilde{s}^{2}+\triangle})X$.  Since $\tilde{s}^{2} \sim \mathcal{N}(m\mu,m\sigma^{2})$, then $(\tilde{s}^{2} + \triangle) \sim \mathcal{N}(m\mu + \triangle,m\sigma^{2})$. Let $\sqrt{\tilde{s}^{2} + \triangle}$ be the random variable $\mathcal{I}$. Similar to Eqn.~\eqref{eqn:pdf-of-s}, the PDF of $\mathcal{I}$ represented by $f_{\mathcal{I}}$ is yielded as follow:
\begin{equation}
f_{\mathcal{I}}(t) = 2t\psi(t^{2}; m\mu+\triangle,m\sigma^{2},0,+\infty)
\end{equation}
By applying Lemma~\ref{lemma:characteristic-func-of-w-in-appendix}, the characteristic function of $\mathcal{I}X$ is as follows:
\begin{equation}
\label{eqn:characteristic-function-of-extension-in-appendix}
\varphi_{\mathcal{I}X}(x) = \frac{1}{2(1-\Phi(\frac{-ms^{2} -n\bigtriangleup}{\sqrt{m} \sigma\bigtriangleup}))} \exp(\frac{mx^{4} \sigma^{2} }{8} -\frac{(ms^{2} +n\bigtriangleup )x^{2}}{2n}) \operatorname{erfc}(\frac{mnx^{2}\sigma^{2} -2(ms^{2} +n\bigtriangleup)}{2\sqrt{2m}n\sigma}) \\
(-\infty < x < + \infty)
\end{equation}
Then the PDF of $\mathcal{I}X$ is obtained by $\varphi_{\mathcal{I}X}(x)$:
\begin{equation}
\label{eqn:pdf-of-extension-in-appendix}
f_{\mathcal{I}X}(t) = \frac{1}{2\pi}\int_{-\infty}^{\infty}\varphi_{\mathcal{I}X}(x)exp(-itx)dx
\end{equation}
It is easy to derive the collision probability of any pair $(\textbf{\emph{v}},\textbf{\emph{u}})$ by $f_{\mathcal{I}X}(t)$, which is as follows:
 \begin{equation}
 \label{eqn:probability-of-extension-in-appendix}
p(s) = \int_{0}^{\tilde{w}^{\prime}}f_{|\mathcal{I}X|}(t)(1-\frac{t}{\tilde{w}^{\prime}})dx
\end{equation}
where $f_{|\mathcal{I}X|}(t)$ denotes the PDF of the absolute value of $\mathcal{I}X$. $\tilde{w}^{\prime}$ is the bucket width.

\section{Additional Experiments}
\label{sec:experiment-in-appendix}

\begin{table*}[http]
 \caption{$\epsilon$ and $\lambda$ for different $m$}
 \centering
 \label{tab:variation-of-epsilon-lamda}
 \resizebox{\linewidth}{!}{
    \begin{tabular}{|c|c|c|c|c|c|c|c|c|c|}
    \hline
    \multirow{2}{*}{Datasets}  & & \multicolumn{2}{|c|}{$m=15$} & \multicolumn{2}{|c|}{$m=30$} & \multicolumn{2}{|c|}{$m=45/\textbf{60}$} & \multicolumn{2}{|c|}{$m=n$} \\
    \cline{3-10}
    & & $\epsilon$ & $\lambda$ & $\epsilon$ & $\lambda$ & $\epsilon$ & $\lambda$ & $\epsilon$ & $\lambda$ \\
    \hline
    \multirow{3}{*}{Cifar}
             &max  & 0.567575 & 2.527575 & 0.287600 & 1.287600 & \textbf{0.114125} & \textbf{0.618225} & 0.000064 & 0.067664  \\
             \cline{2-10}
             &mean & 0.102716 & 0.575372 & 0.027539 & 0.277239 & \textbf{0.001994} & \textbf{0.120123} & 0 & 0.014617  \\
             \cline{2-10}
             &min  & 0.019616 & 0.239671 & 0.001702 & 0.116014 & \textbf{0.000011} & \textbf{0.055001} & 0 & 0.006691  \\
    \hline
    \multirow{3}{*}{Sun}
             &max  & 0.502360 & 2.218460 & 0.236626 & 1.084867 & \textbf{0.095099} & \textbf{0.546683} & 0.000006 & 0.051535   \\
             \cline{2-10}
             &mean & 0.022783 & 0.255107 & 0.001849 & 0.118130 & \textbf{0.000018} & \textbf{0.058099} & 0 & 0.006724 \\
             \cline{2-10}
             &min  & 0 & 0.040401 & 0 & 0.020736 & \textbf{0} & \textbf{0.010000} & 0 & 0.001156 \\
    \hline
    \multirow{3}{*}{Gist}
             &max  & 0.633640 & 2.853740 & 0.274502 & 1.234902 & \textbf{0.129940} & \textbf{0.677540} & 0 & 0.032400  \\
             \cline{2-10}
             &mean & 0.042122 & 0.340238 & 0.005183 & 0.152639 & \textbf{0.000133} & \textbf{0.074662} & 0 & 0.004624 \\
             \cline{2-10}
             &min  & 0.000064 & 0.067664 & 0 & 0.038025 & \textbf{0} & \textbf{0.016926} & 0 & 0.001089  \\
    \hline
    \multirow{3}{*}{Trevi}
             &max  & 0.013420 & 0.207020 & 0.001105 & 0.106081 & \textbf{0.000003} & \textbf{0.049287} & 0 & 0.000729  \\
            \cline{2-10}
             &mean & 0.000011 & 0.055236 & 0 & 0.029241 & \textbf{0} & \textbf{0.013924} & 0 & 0.000196 \\
             \cline{2-10}
             &min  & 0 & 0.015876 & 0 & 0.008100 & \textbf{0} & \textbf{0.003969} & 0 & 0.000100 \\
    \hline
    \multirow{3}{*}{Audio}
             &max  & 0.268000 & 1.208900 & 0.099004 & 0.561404 & 0.043520 & 0.346020 & 0.000033 & 0.062533  \\
             \cline{2-10}
             &mean & 0.033344 & 0.303017 & 0.003637 & 0.138767 & 0.000480 & 0.091021 & 0 & 0.020967  \\
             \cline{2-10}
             &min  & 0.000022 & 0.059705 & 0 & 0.030765 & 0 & 0.021874 & 0 & 0.005329 \\
    \hline
    \multirow{3}{*}{Notre}
             &max  & 0.341109 & 1.507509 & 0.143598 & 0.728823 & 0.074226 & 0.467355 & 0.004017 & 0.142401  \\
             \cline{2-10}
             &mean & 0.022783 & 0.255107 & 0.001902 & 0.118866 & 0.000172 & 0.077456 & 0 & 0.027225\\
             \cline{2-10}
             &min  & 0.000101 & 0.071925 & 0 & 0.036481 & 0 & 0.023716 & 0 & 0.008281 \\
    \hline
    \multirow{3}{*}{Glove}
             &max  & 0.261530 & 1.183130 & 0.094131 & 0.543031 & 0.040065 & 0.331665 & 0.002362 & 0.124862  \\
             \cline{2-10}
             &mean & 0.003190 & 0.134234 & 0.000047 & 0.065072 & 0.000001 & 0.043265 & 0 & 0.019853 \\
             \cline{2-10}
             &min  & 0.000025 & 0.060541 & 0 & 0.029929 & 0 & 0.020335 & 0 & 0.009409  \\
    \hline
    \multirow{3}{*}{Sift}
             &max  & 0.294194 & 1.314294 & 0.098513 & 0.559554 & 0.057014 & 0.400410 & 0.001296 & 0.109537  \\
             \cline{2-10}
             &mean & 0.054265 & 0.389506 & 0.007185 & 0.167986 & 0.001509 & 0.113065 & 0 & 0.036864 \\
             \cline{2-10}
             &min  & 0.003755 & 0.139916 & 0.000107 & 0.072468 & 0.000001 & 0.045370 & 0 & 0.016129  \\
    \hline
    \multirow{3}{*}{Deep}
             &max  & 0.036744 & 0.317644 & 0.003841 & 0.140741 & 0.000463 & 0.090463 & 0 & 0.014400   \\
             \cline{2-10}
             &mean & 0.003047 & 0.132719 & 0.000044 & 0.064560 & 0.000001 & 0.043306 & 0 & 0.007569 \\
             \cline{2-10}
             &min  & 0.000199 & 0.079160 & 0& 0.039204 & 0 & 0.026374 & 0 & 0.004638  \\
    \hline
    \multirow{3}{*}{Random}
             &max  & 0.077344 & 0.479300 & 0.015195 & 0.216796 & 0.003505 & 0.137461 & 0.000041 & 0.064050  \\
             \cline{2-10}
             &mean & 0.002824 & 0.130273 & 0.000038 & 0.063542 & 0.000001 & 0.042437 & 0 & 0.019044 \\
             \cline{2-10}
             &min  & 0.000024 & 0.060049 & 0 & 0.029929 & 0. & 0.019881 & 0 & 0.008649 \\
    \hline
    \multirow{3}{*}{Ukbench}
             &max  & 0.692955 & 3.157855 & 0.361581 & 1.593681 & 0.217238 & 1.009338 & 0.039727 & 0.330248  \\
             \cline{2-10}
             &mean & 0.139183 & 0.712232 & 0.034502 & 0.308031 & 0.012672 & 0.202768 & 0.000056 & 0.066620  \\
             \cline{2-10}
             &min  & 0.002895 & 0.131059 & 0.000041 & 0.064050 & 0.000001 & 0.042437 & 0 & 0.015376  \\
    \hline
    \multirow{3}{*}{ImageNet}
             &max  & 0.466568 & 2.054168 & 0.217238 & 1.009338 & 0.119323 & 0.637723 & 0.004773 & 0.149173  \\
             \cline{2-10}
             &mean & 0.019125 & 0.237214 & 0.001336 & 0.110236 & 0.000107 & 0.072468 & 0 & 0.022201 \\
             \cline{2-10}
             &min  & 0 & 0.033489 & 0 & 0.016641 & 0 & 0.011025 & 0 & 0.003481 \\
    \hline
    \end{tabular}\label{tab}
   }
\end{table*}

\begin{figure}[t]
	\centering
	\begin{minipage}{0.23\textwidth}
		\centering
		\centerline{\includegraphics[width=\textwidth]{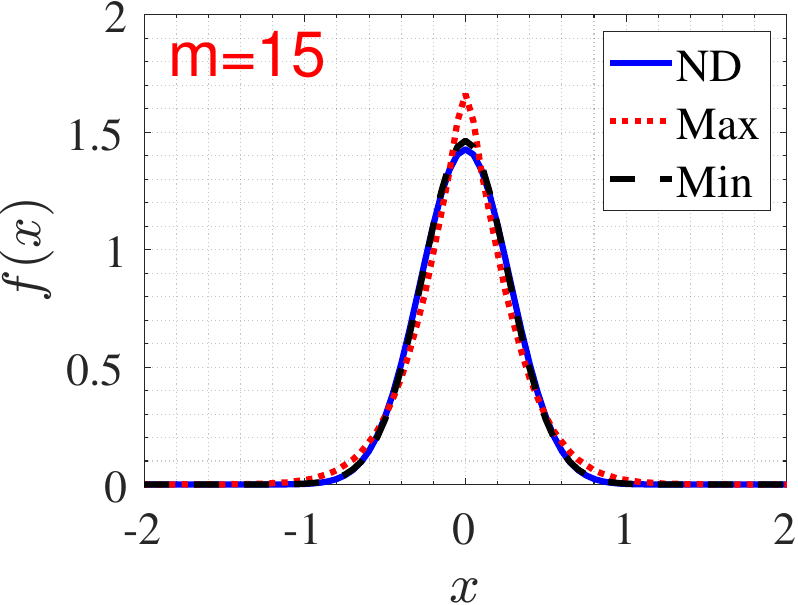}}
        \centerline{\scriptsize (a-1) Audio: $m=15$}
	\end{minipage}
    \hfill
	\begin{minipage}{0.23\textwidth}
        \centering
		\centerline{\includegraphics[width=\textwidth]{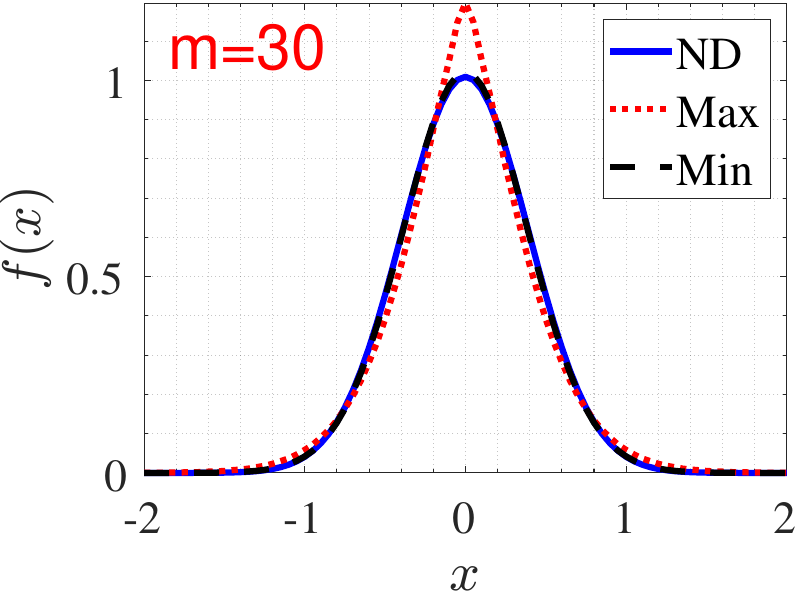}}
        \centerline{\scriptsize (a-2) Audio: $m=30$}
	\end{minipage}
    \hfill
    \begin{minipage}{0.23\textwidth}
        \centering
		\centerline{\includegraphics[width=\textwidth]{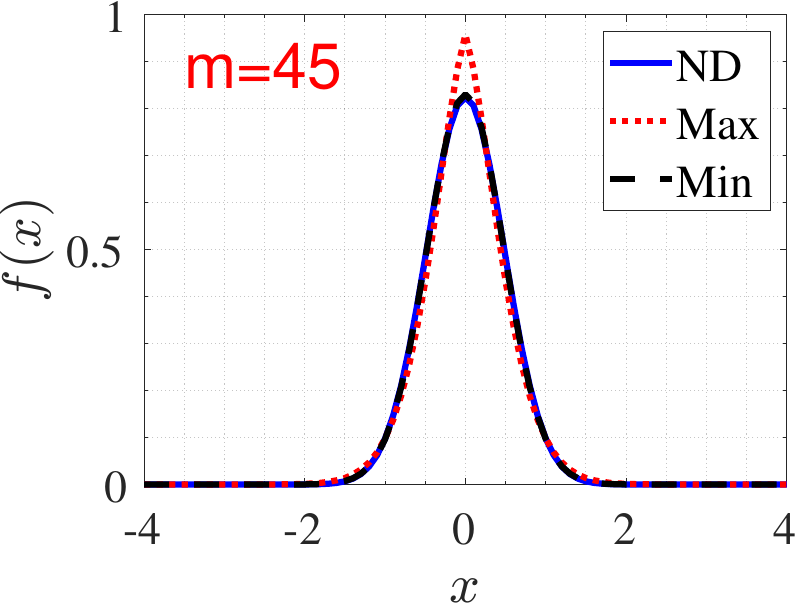}}
        \centerline{\scriptsize (a-3) Audio: $m=45$}
	\end{minipage}
    \hfill
	\begin{minipage}{0.23\textwidth}
        \centering
		\centerline{\includegraphics[width=\textwidth]{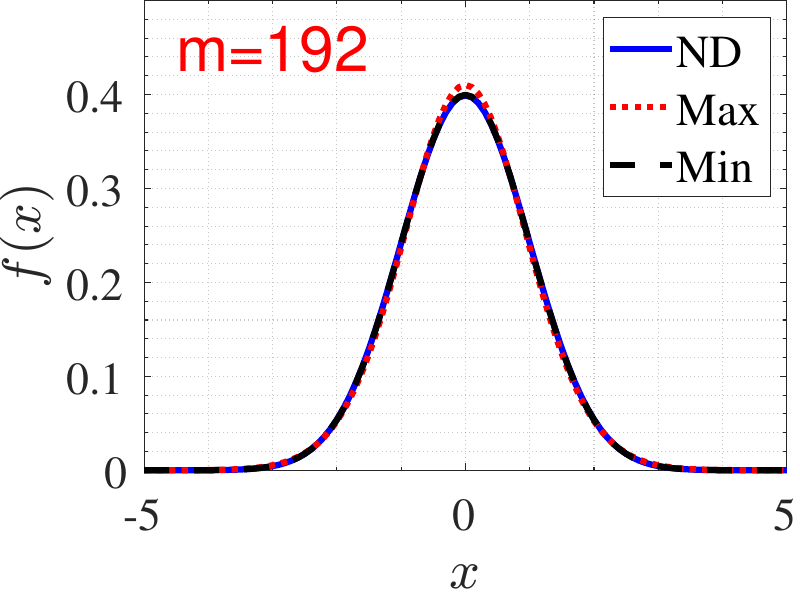}}
        \centerline{\scriptsize (a-4) Audio: $m=n$}
	\end{minipage}
    \hfill
	\begin{minipage}{0.23\textwidth}
        \centering
		\centerline{\includegraphics[width=\textwidth]{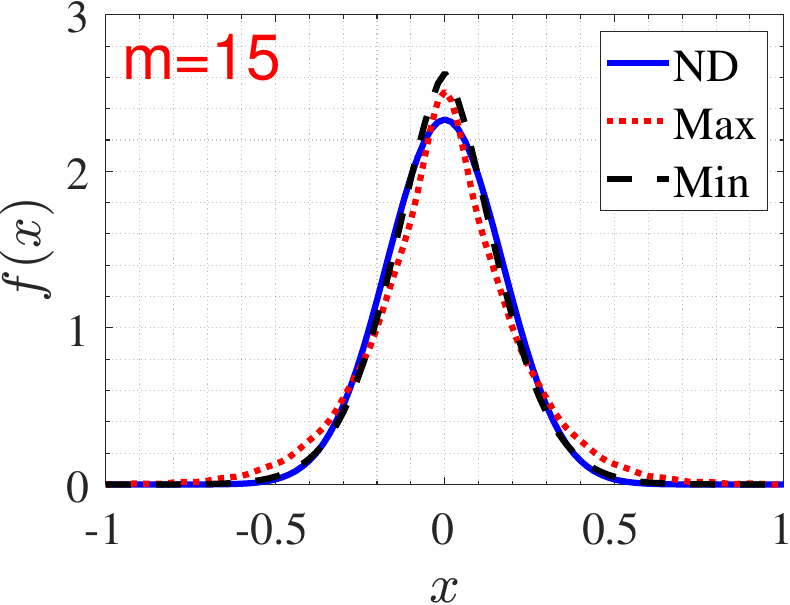}}
        \centerline{\scriptsize (b-1) Cifar: $m=15$}
	\end{minipage}
    \hfill
	\begin{minipage}{0.23\textwidth}
        \centering
		\centerline{\includegraphics[width=\textwidth]{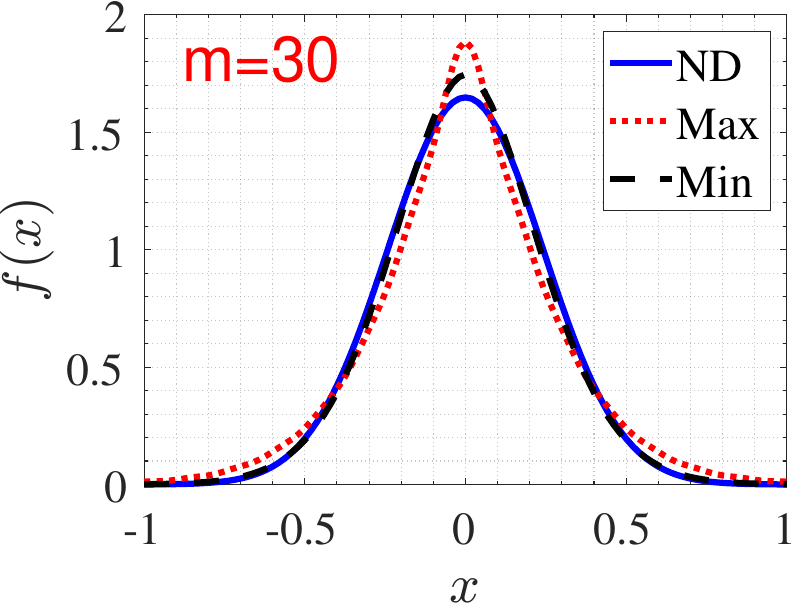}}
        \centerline{\scriptsize (b-2) Cifar: $m=30$}
	\end{minipage}
    \hfill
	\begin{minipage}{0.23\textwidth}
        \centering
		\centerline{\includegraphics[width=\textwidth]{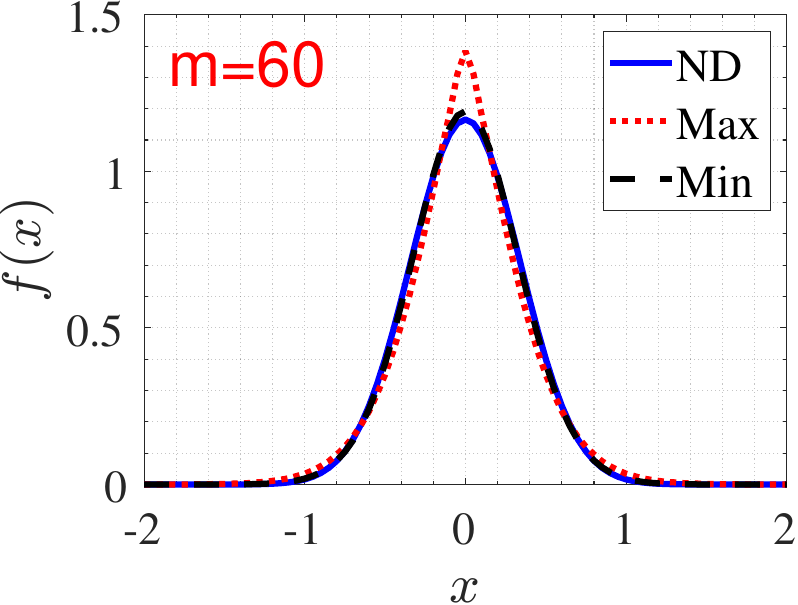}}
        \centerline{\scriptsize (b-3) Cifar: $m=60$}
	\end{minipage}
    \hfill
	\begin{minipage}{0.23\textwidth}
        \centering
		\centerline{\includegraphics[width=\textwidth]{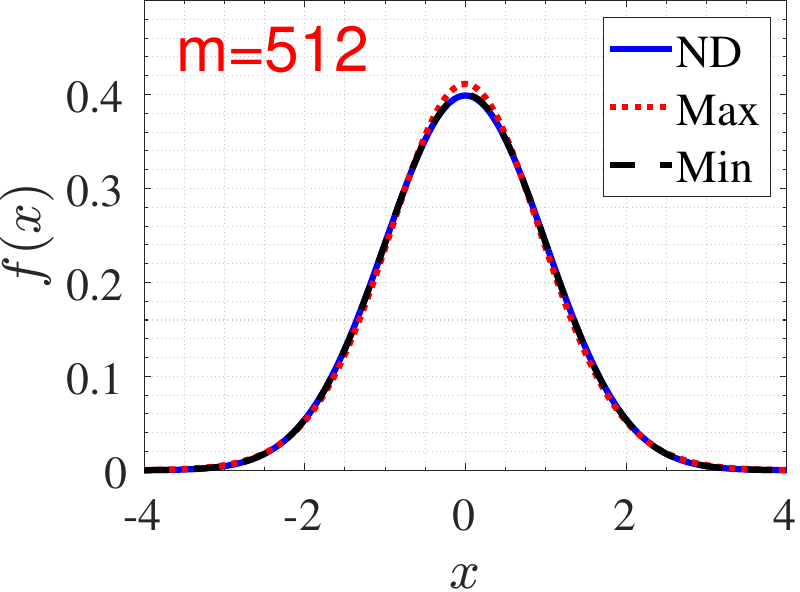}}
        \centerline{\scriptsize (b-4) Cifar: $m=n$}
	\end{minipage}
    \hfill
	\begin{minipage}{0.23\textwidth}
        \centering
		\centerline{\includegraphics[width=\textwidth]{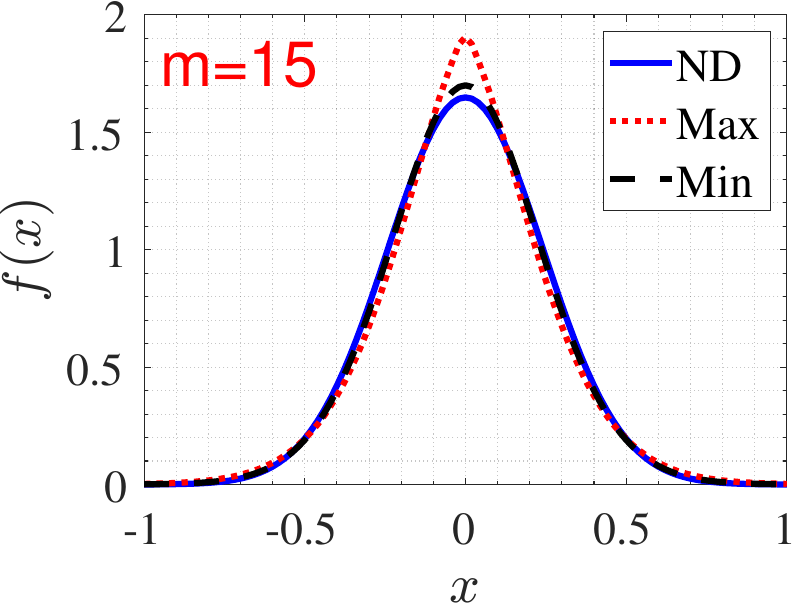}}
        \centerline{\scriptsize (c-1) Deep: $m=15$}
	\end{minipage}
    \hfill
	\begin{minipage}{0.23\textwidth}
        \centering
		\centerline{\includegraphics[width=\textwidth]{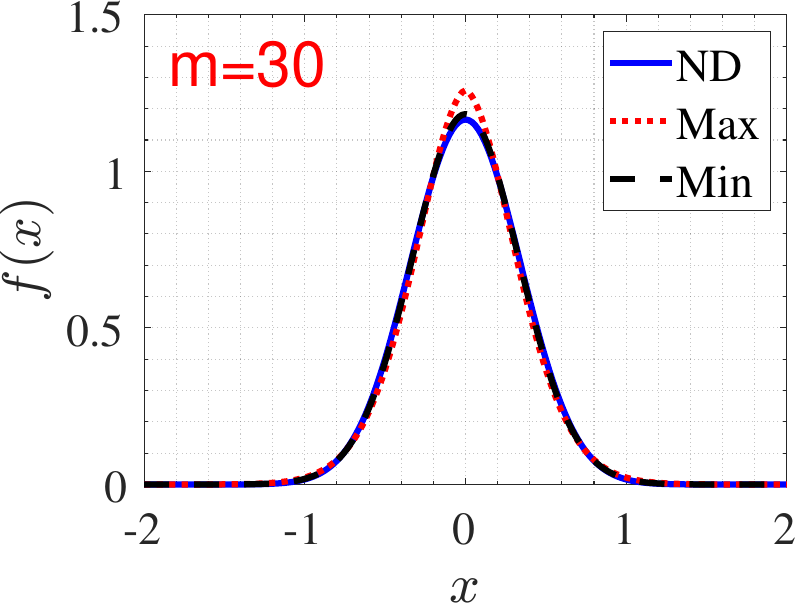}}
        \centerline{\scriptsize (c-2) Deep: $m=30$}
	\end{minipage}
    \hfill
	\begin{minipage}{0.23\textwidth}
        \centering
		\centerline{\includegraphics[width=\textwidth]{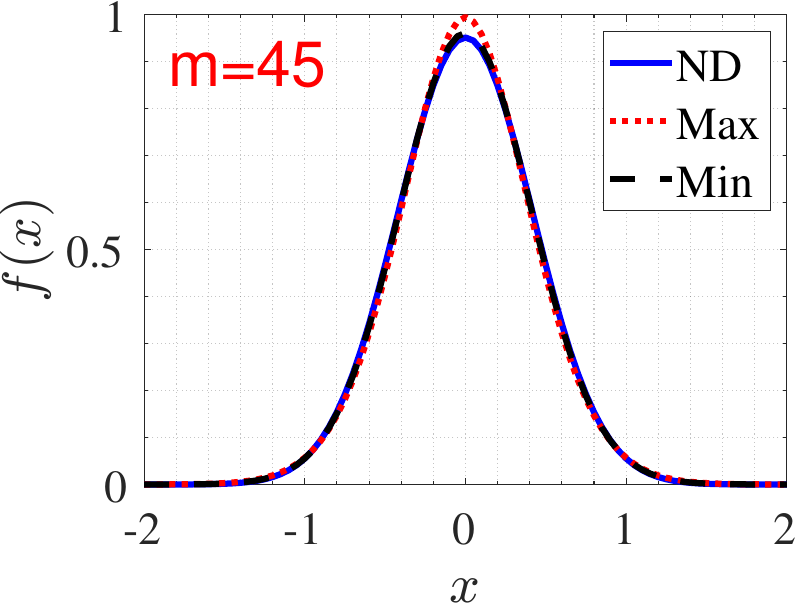}}
        \centerline{\scriptsize (c-3) Deep: $m=45$}
	\end{minipage}
    \hfill
	\begin{minipage}{0.23\textwidth}
        \centering
		\centerline{\includegraphics[width=\textwidth]{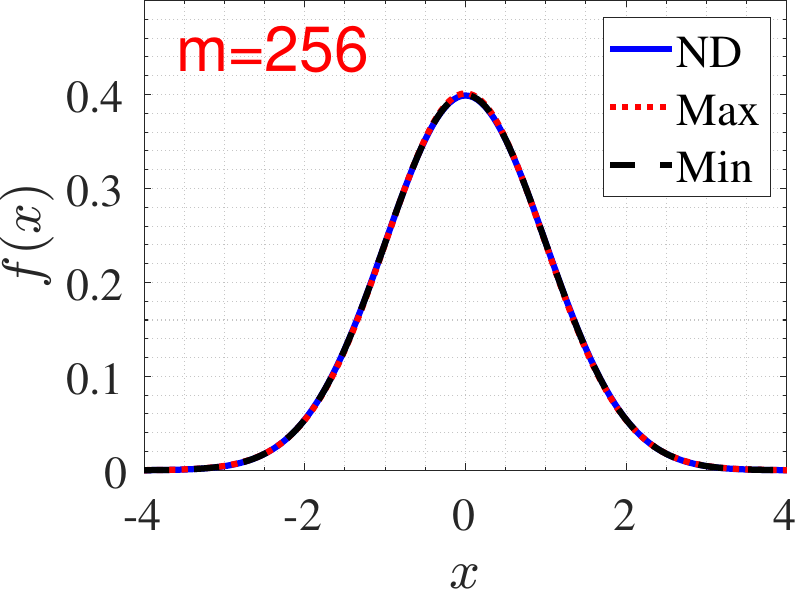}}
        \centerline{\scriptsize (c-4) Deep: $m=n$}
	\end{minipage}
    \hfill
	\begin{minipage}{0.23\textwidth}
        \centering
		\centerline{\includegraphics[width=\textwidth]{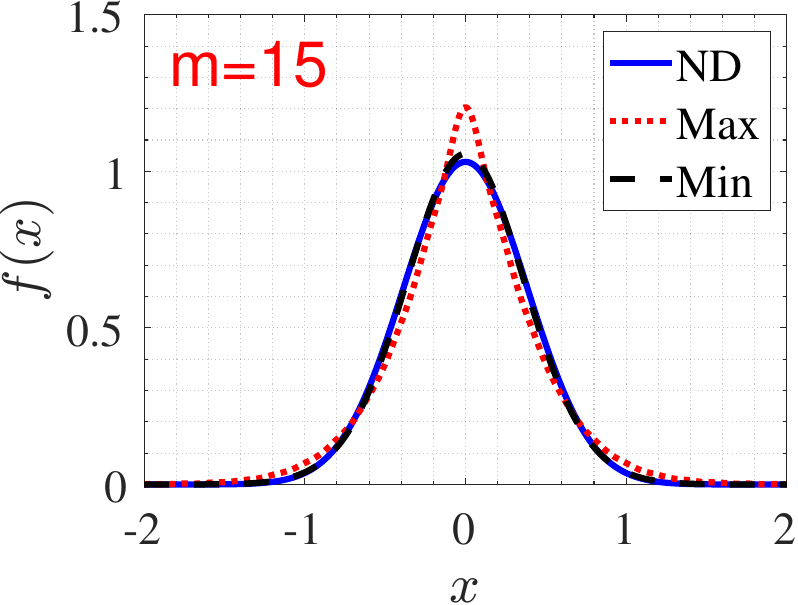}}
        \centerline{\scriptsize (d-1) Glove: $m=15$}
	\end{minipage}
    \hfill
	\begin{minipage}{0.23\textwidth}
        \centering
		\centerline{\includegraphics[width=\textwidth]{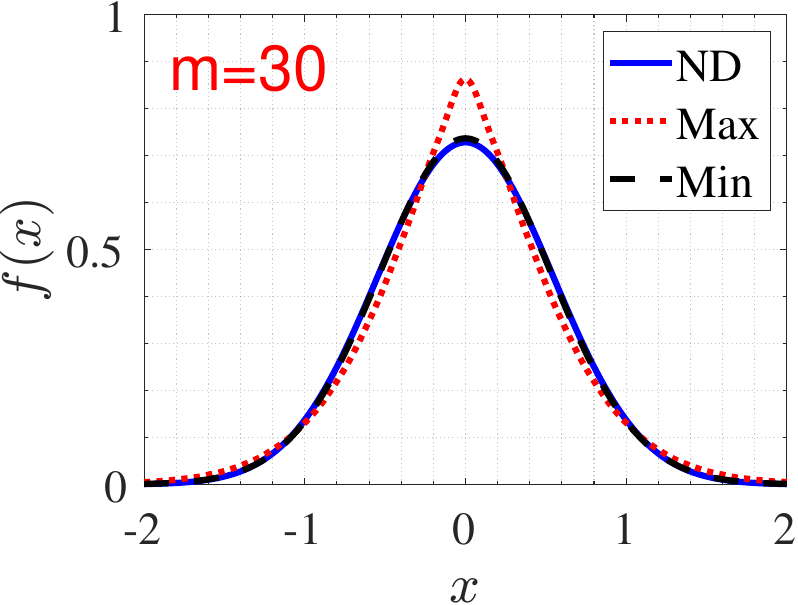}}
        \centerline{\scriptsize (d-2) Glove: $m=30$}
	\end{minipage}
    \hfill
	\begin{minipage}{0.23\textwidth}
        \centering
		\centerline{\includegraphics[width=\textwidth]{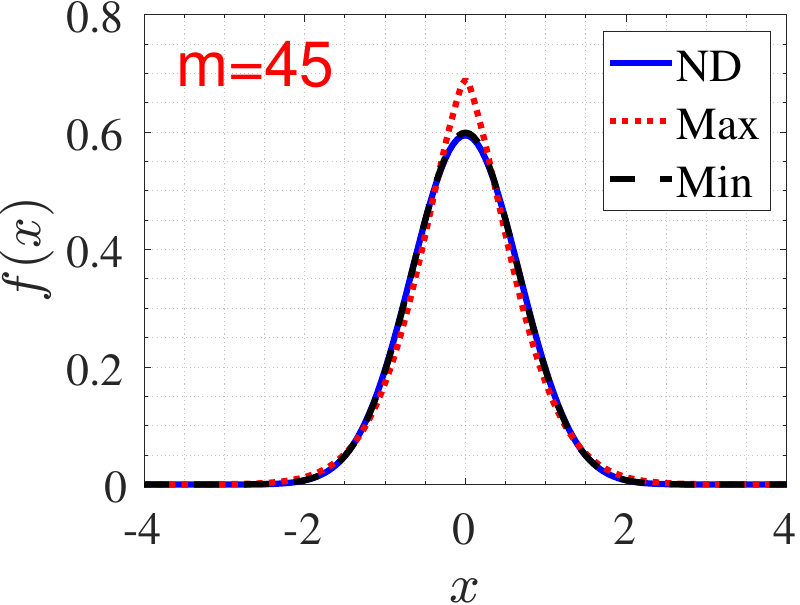}}
        \centerline{\scriptsize (d-3) Glove: $m=45$}
	\end{minipage}
    \hfill
	\begin{minipage}{0.23\textwidth}
        \centering
		\centerline{\includegraphics[width=\textwidth]{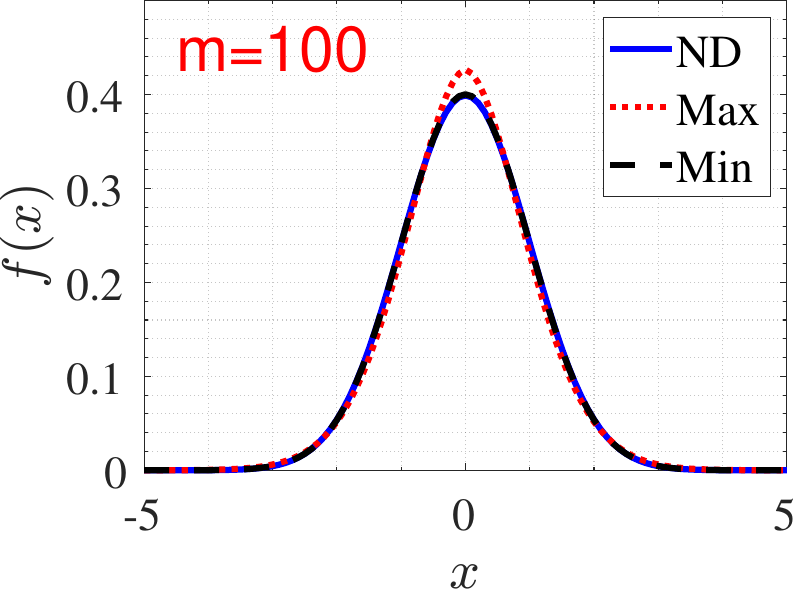}}
        \centerline{\scriptsize (d-4) Glove: $m=n$}
	\end{minipage}
    \hfill
	\begin{minipage}{0.23\textwidth}
        \centering
		\centerline{\includegraphics[width=\textwidth]{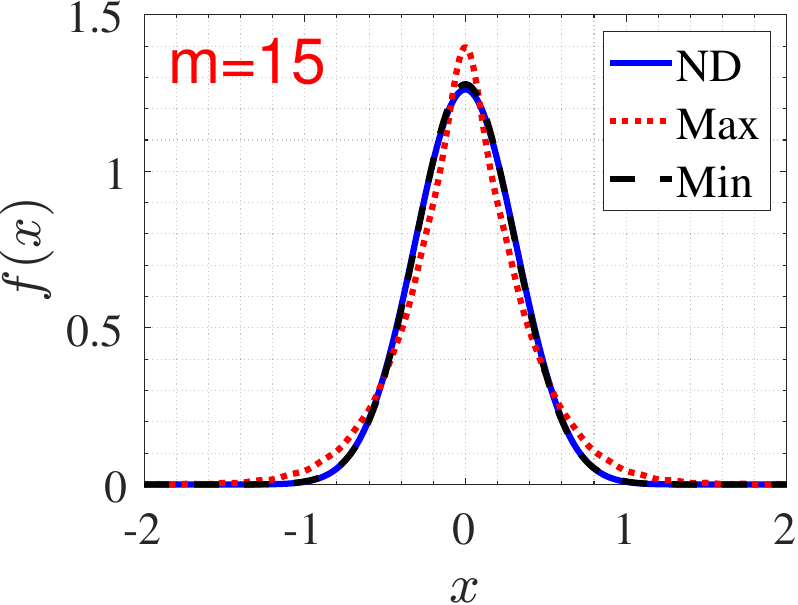}}
        \centerline{\scriptsize (e-1) ImageNet: $m=15$}
	\end{minipage}
    \hfill
	\begin{minipage}{0.23\textwidth}
        \centering
		\centerline{\includegraphics[width=\textwidth]{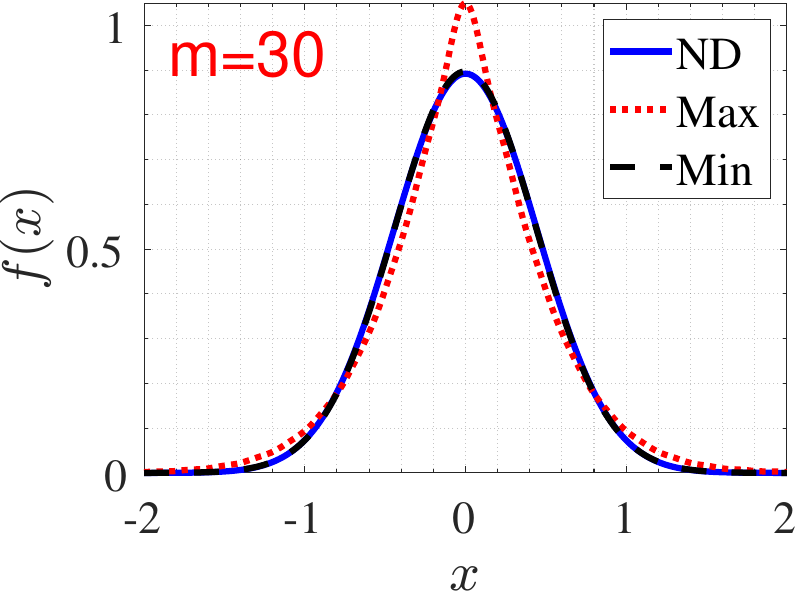}}
        \centerline{\scriptsize (e-2) ImageNet: $m=30$}
	\end{minipage}
    \hfill
	\begin{minipage}{0.23\textwidth}
        \centering
		\centerline{\includegraphics[width=\textwidth]{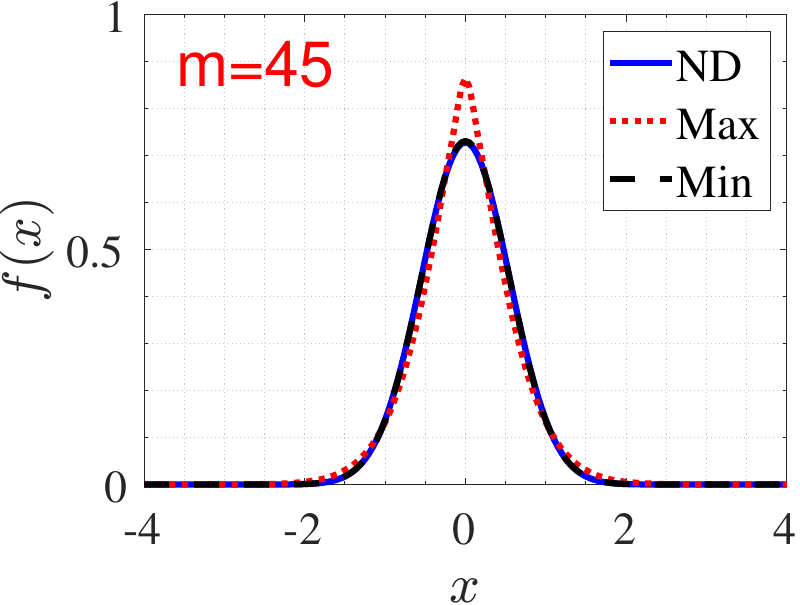}}
        \centerline{\scriptsize (e-3) ImageNet: $m=45$}
	\end{minipage}
    \hfill
	\begin{minipage}{0.23\textwidth}
        \centering
		\centerline{\includegraphics[width=\textwidth]{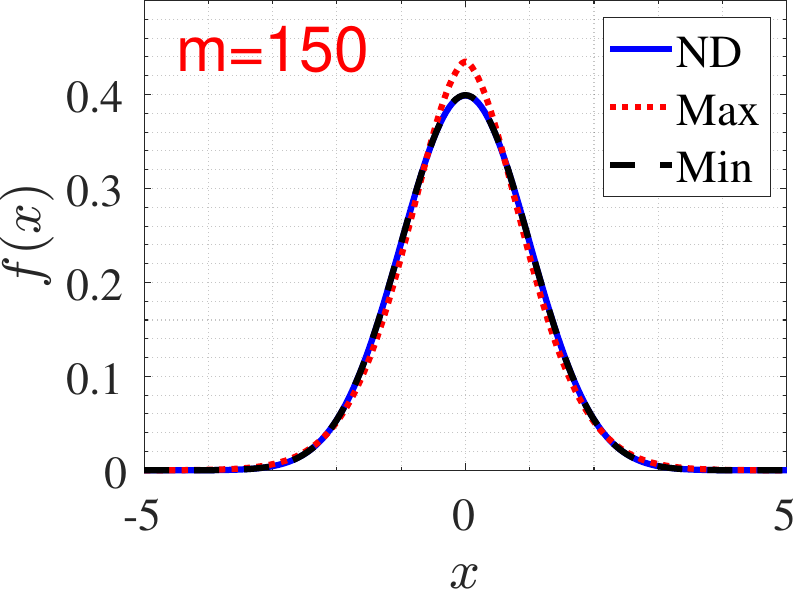}}
        \centerline{\scriptsize (e-4) ImageNet: $m=n$}
	\end{minipage}
    \hfill
	\begin{minipage}{0.23\textwidth}
        \centering
		\centerline{\includegraphics[width=\textwidth]{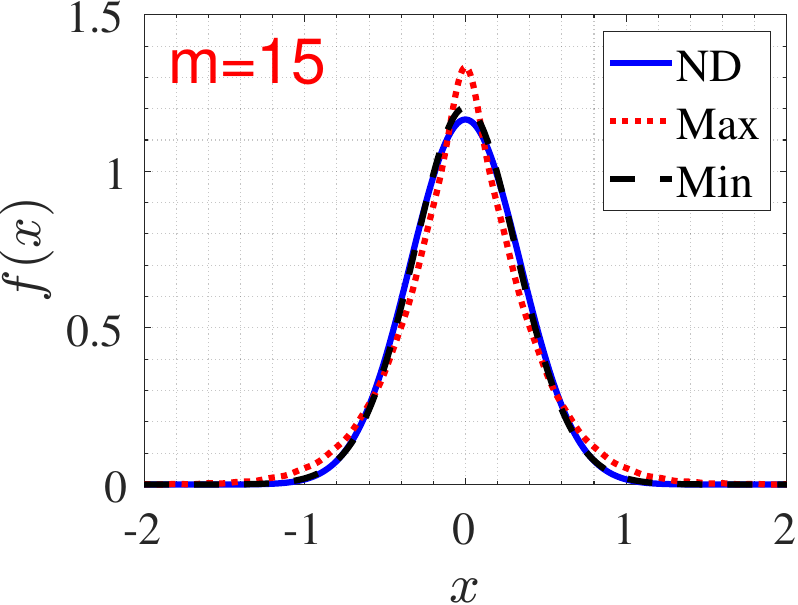}}
        \centerline{\scriptsize (f-1) Notre: $m=15$}
	\end{minipage}
    \hfill
	\begin{minipage}{0.23\textwidth}
        \centering
		\centerline{\includegraphics[width=\textwidth]{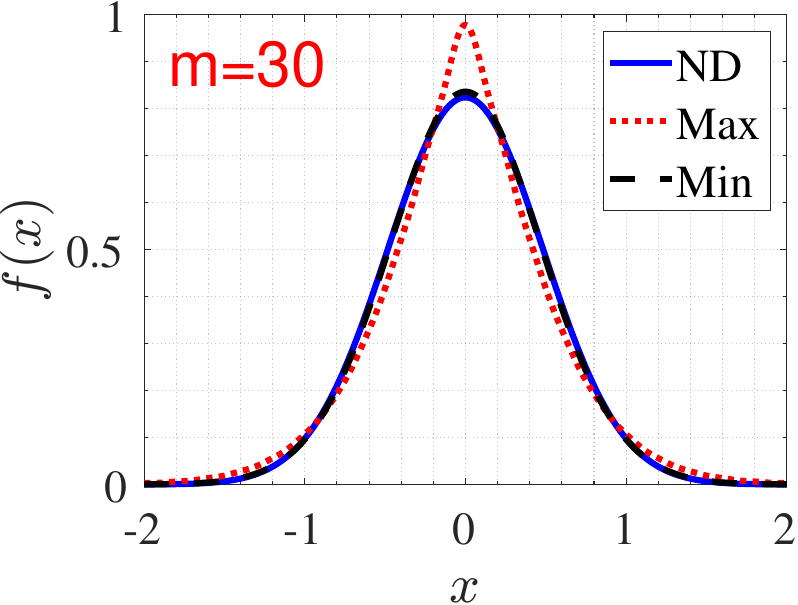}}
        \centerline{\scriptsize (f-2) Notre: $m=30$}
	\end{minipage}
    \hfill
	\begin{minipage}{0.23\textwidth}
        \centering
		\centerline{\includegraphics[width=\textwidth]{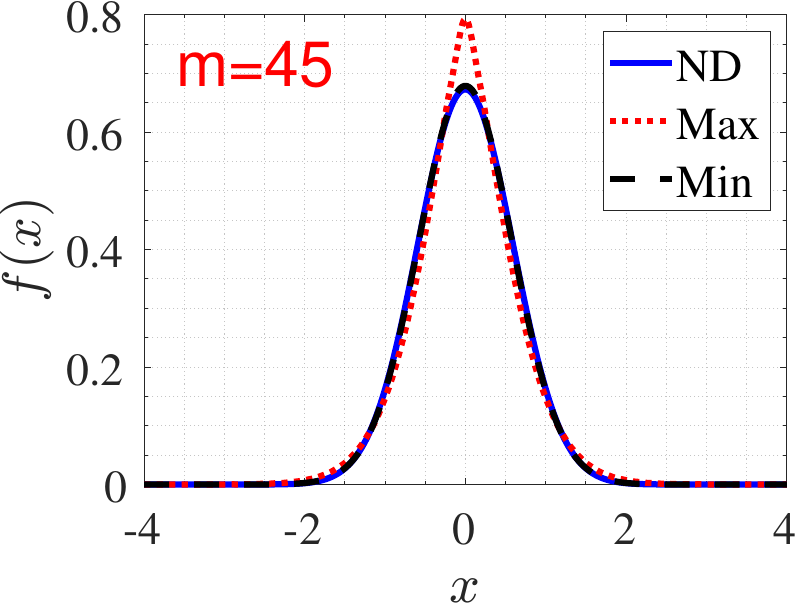}}
        \centerline{\scriptsize (f-3) Notre: $m=45$}
	\end{minipage}
    \hfill
	\begin{minipage}{0.23\textwidth}
        \centering
		\centerline{\includegraphics[width=\textwidth]{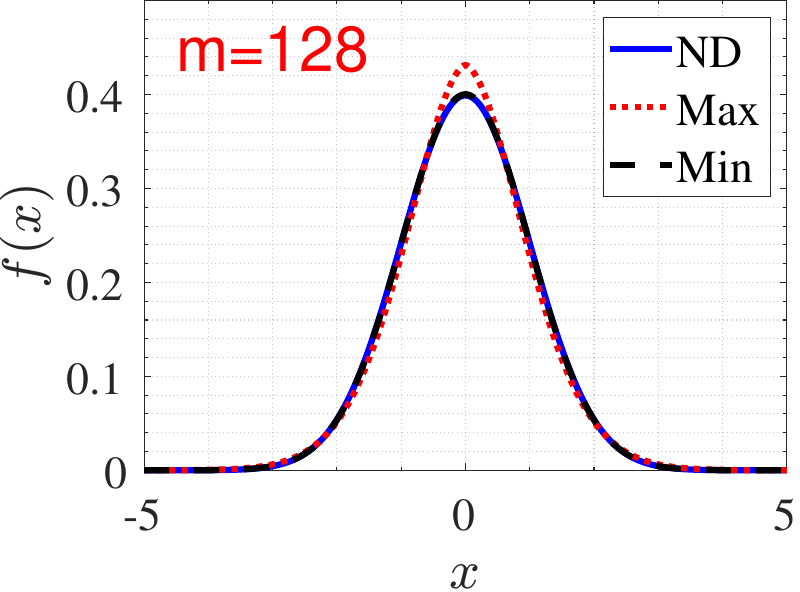}}
        \centerline{\scriptsize (f-4) Notre: $m=n$}
	\end{minipage}
    \addtocounter{figure}{-1}
\end{figure}
\begin{figure}[t]
    \addtocounter{figure}{1}
	\centering
	\begin{minipage}{0.23\textwidth}
        \centering
		\centerline{\includegraphics[width=\textwidth]{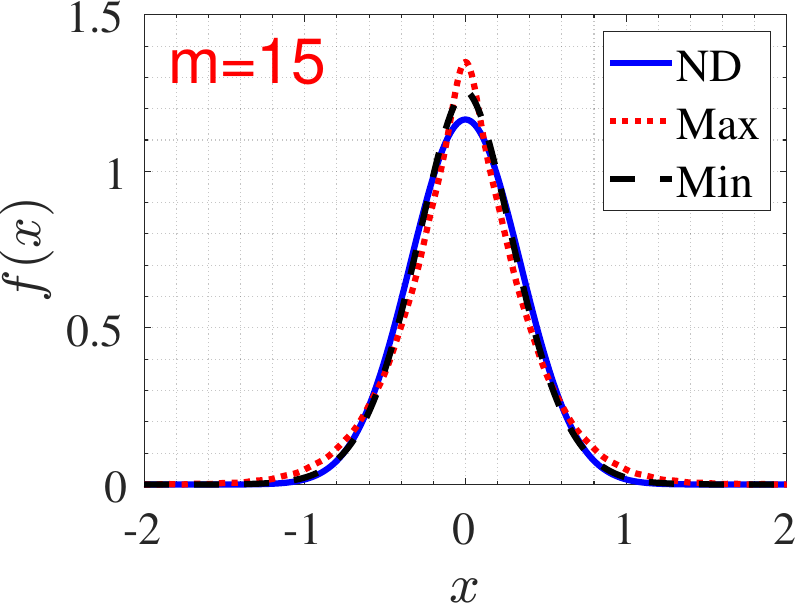}}
        \centerline{\scriptsize (g-1) Sift: $m=15$}
	\end{minipage}
    \hfill
	\begin{minipage}{0.23\textwidth}
        \centering
		\centerline{\includegraphics[width=\textwidth]{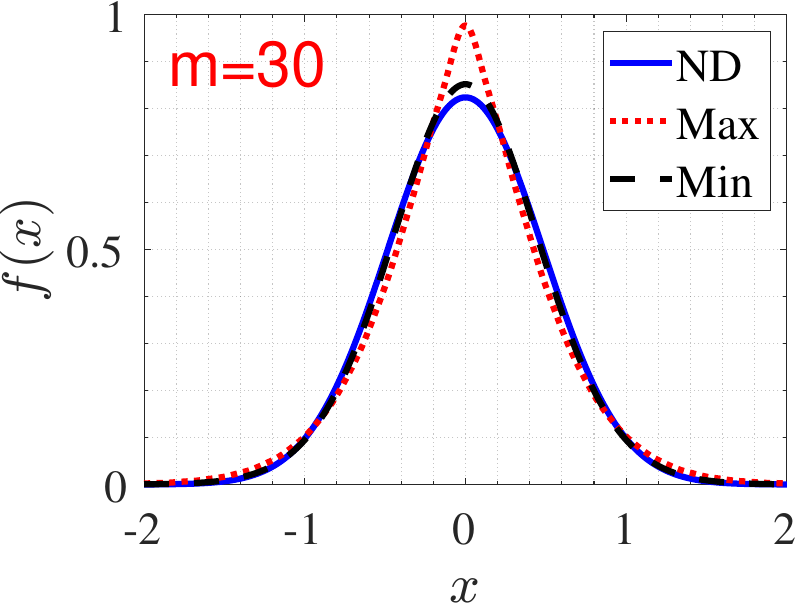}}
        \centerline{\scriptsize (g-2) Sift: $m=30$}
	\end{minipage}
    \hfill
	\begin{minipage}{0.23\textwidth}
        \centering
		\centerline{\includegraphics[width=\textwidth]{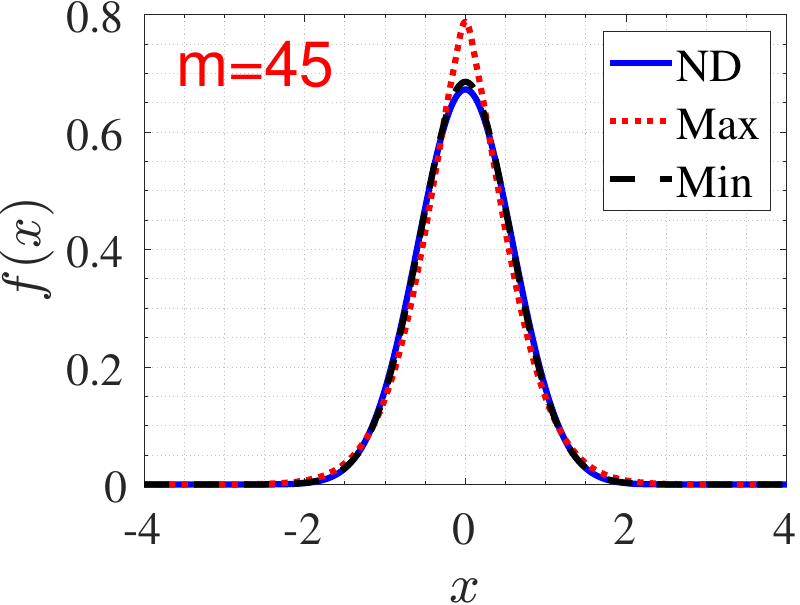}}
        \centerline{\scriptsize (g-3) Sift: $m=45$}
	\end{minipage}
    \hfill
	\begin{minipage}{0.23\textwidth}
        \centering
		\centerline{\includegraphics[width=\textwidth]{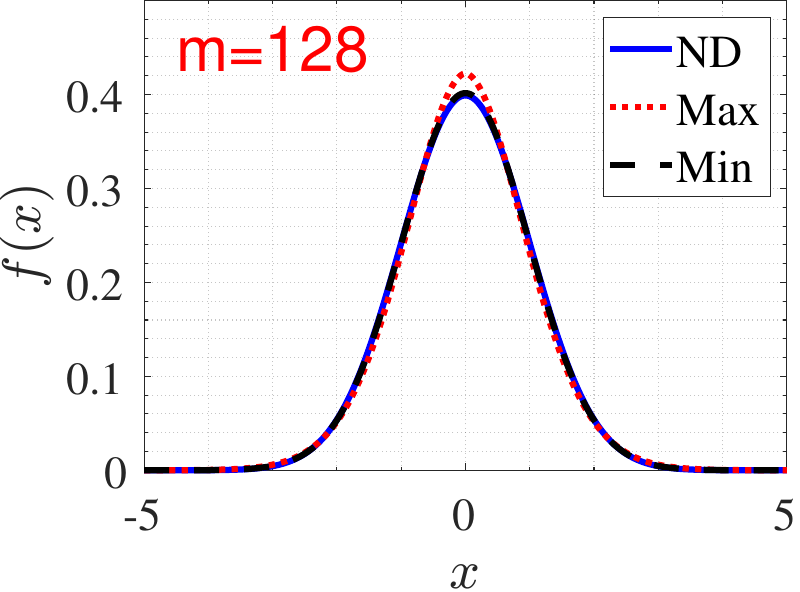}}
        \centerline{\scriptsize (g-4) Sift: $m=n$}
	\end{minipage}
    \hfill
	\begin{minipage}{0.23\textwidth}
        \centering
		\centerline{\includegraphics[width=\textwidth]{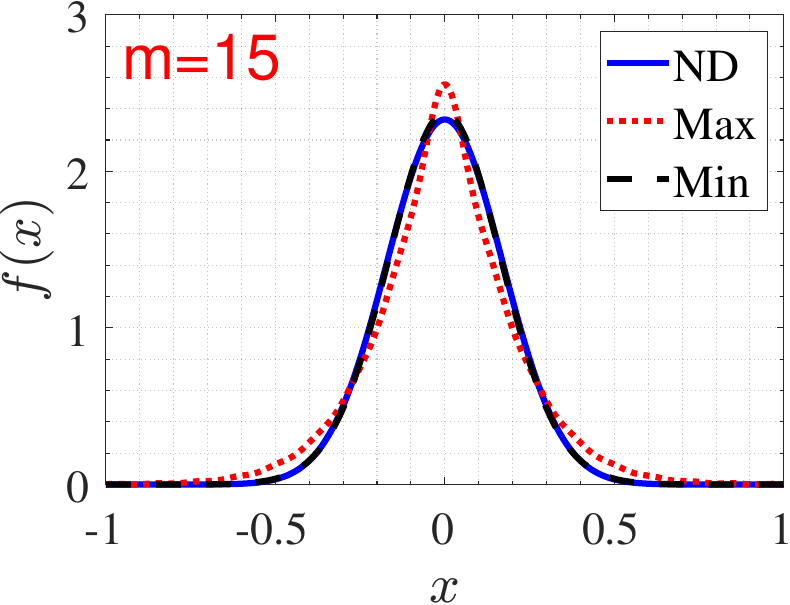}}
        \centerline{\scriptsize (h-1) Sun: $m=15$}
	\end{minipage}
    \hfill
	\begin{minipage}{0.23\textwidth}
        \centering
		\centerline{\includegraphics[width=\textwidth]{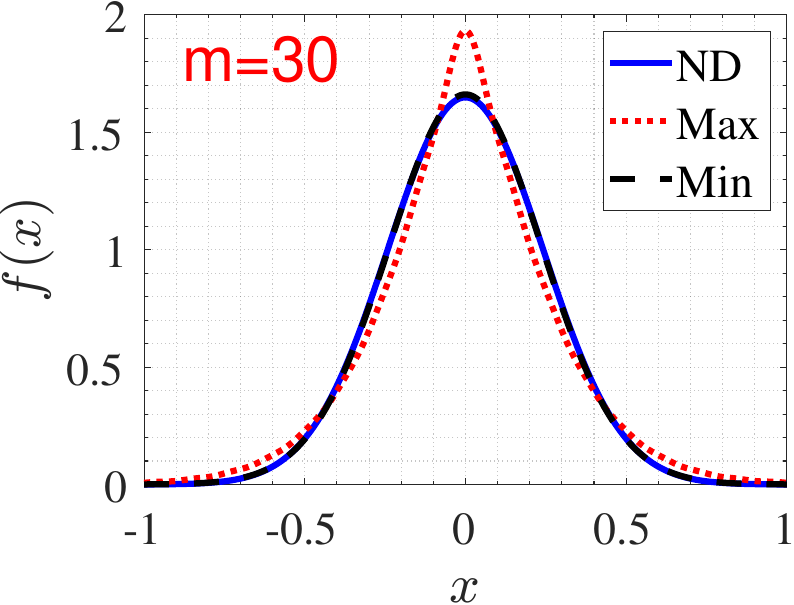}}
        \centerline{\scriptsize (h-2) Sun: $m=30$}
	\end{minipage}
    \hfill
	\begin{minipage}{0.23\textwidth}
        \centering
		\centerline{\includegraphics[width=\textwidth]{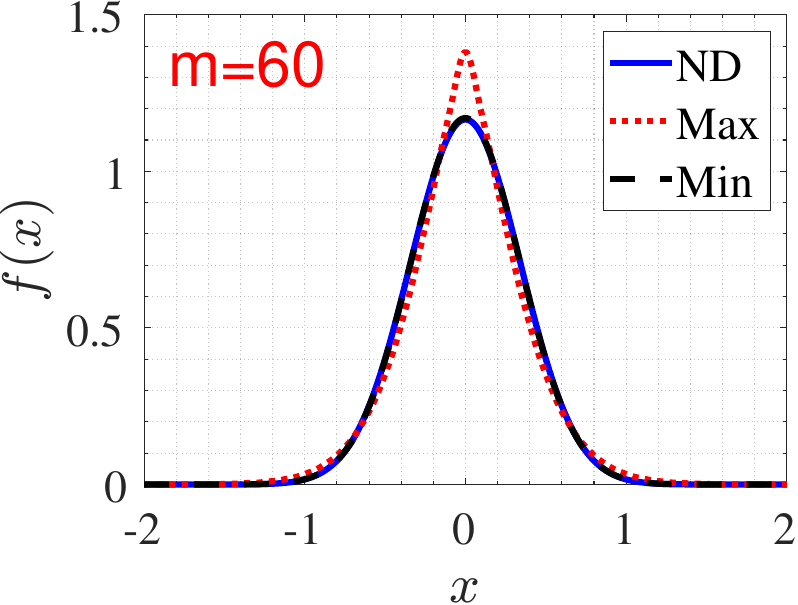}}
        \centerline{\scriptsize (h-3) Sun: $m=60$}
	\end{minipage}
    \hfill
	\begin{minipage}{0.23\textwidth}
        \centering
		\centerline{\includegraphics[width=\textwidth]{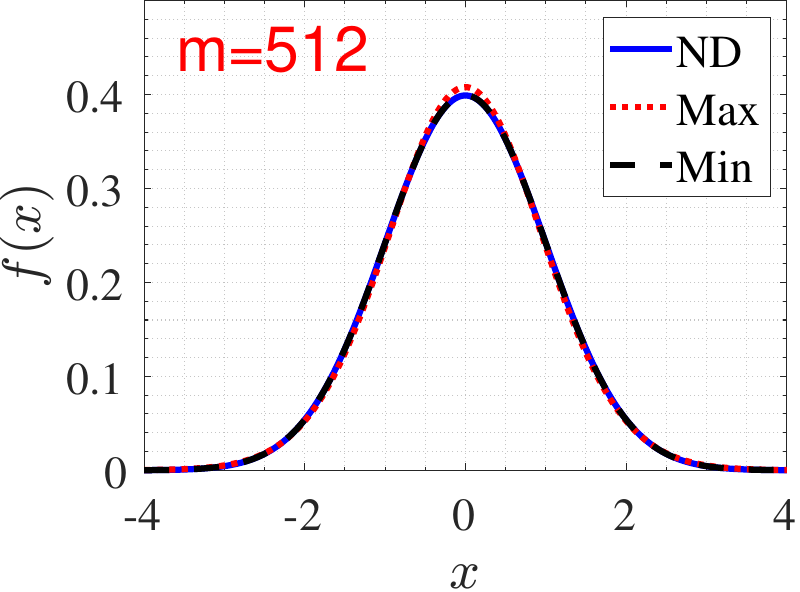}}
        \centerline{\scriptsize (h-4) Sun: $m=n$}
	\end{minipage}
    \hfill
	\begin{minipage}{0.23\textwidth}
        \centering
		\centerline{\includegraphics[width=\textwidth]{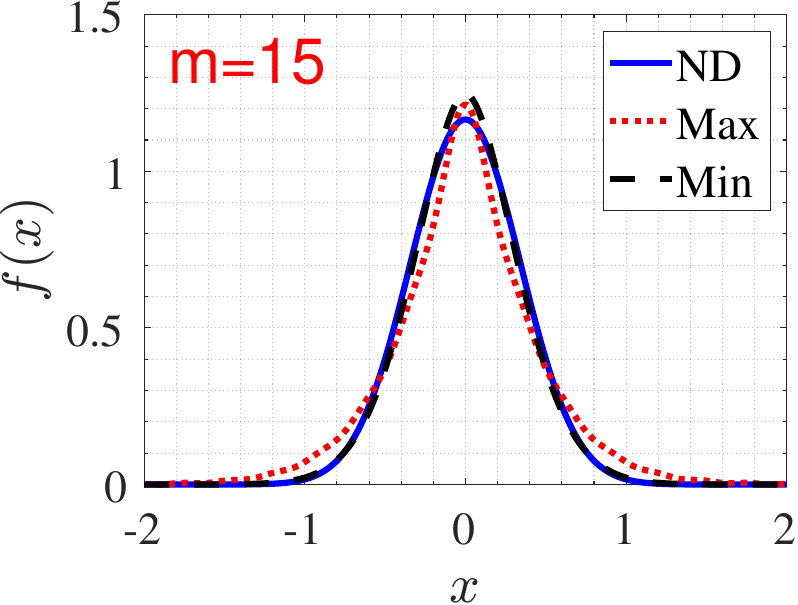}}
        \centerline{\scriptsize (i-1) Ukbench: $m=15$}
	\end{minipage}
    \hfill
	\begin{minipage}{0.23\textwidth}
        \centering
		\centerline{\includegraphics[width=\textwidth]{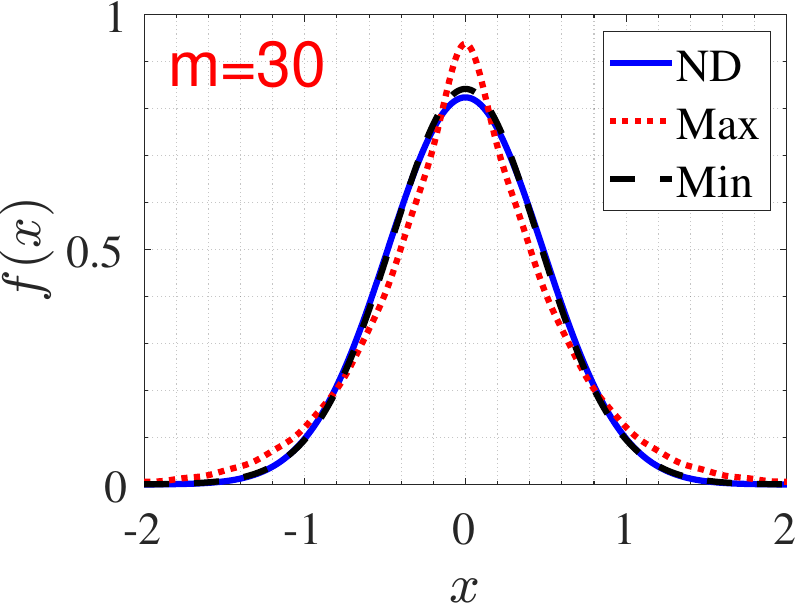}}
        \centerline{\scriptsize (i-2) Ukbench: $m=30$}
	\end{minipage}
    \hfill
	\begin{minipage}{0.23\textwidth}
        \centering
		\centerline{\includegraphics[width=\textwidth]{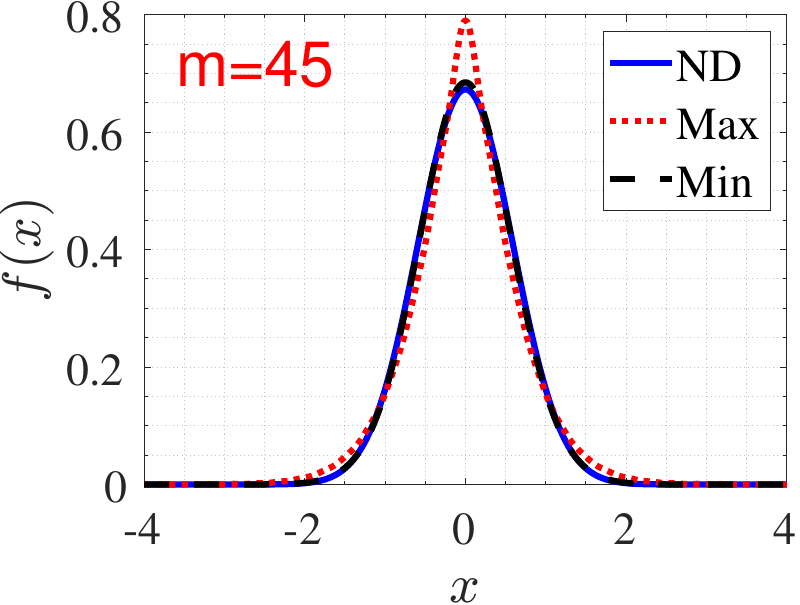}}
        \centerline{\scriptsize (i-3) Ukbench: $m=45$}
	\end{minipage}
    \hfill
	\begin{minipage}{0.23\textwidth}
        \centering
		\centerline{\includegraphics[width=\textwidth]{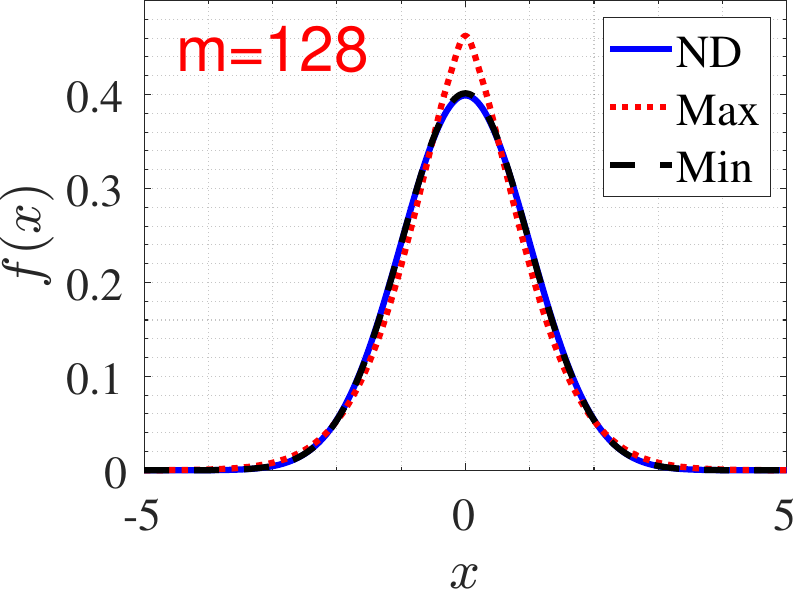}}
        \centerline{\scriptsize (i-4) Ukbench: $m=n$}
	\end{minipage}
    \hfill
    \begin{minipage}{0.23\textwidth}
		\centering
		\centerline{\includegraphics[width=\textwidth]{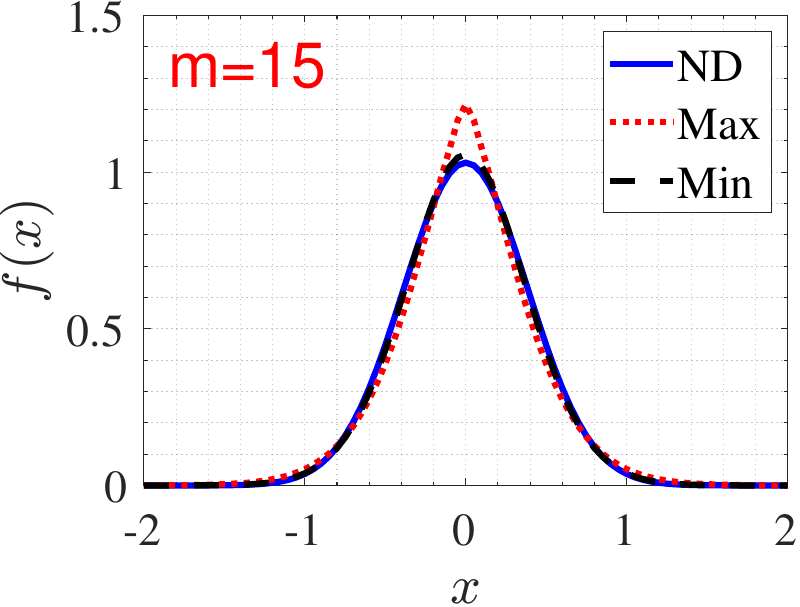}}
        \centerline{\scriptsize (j-1) Random: $m=15$}
	\end{minipage}
    \hfill
	\begin{minipage}{0.23\textwidth}
        \centering
		\centerline{\includegraphics[width=\textwidth]{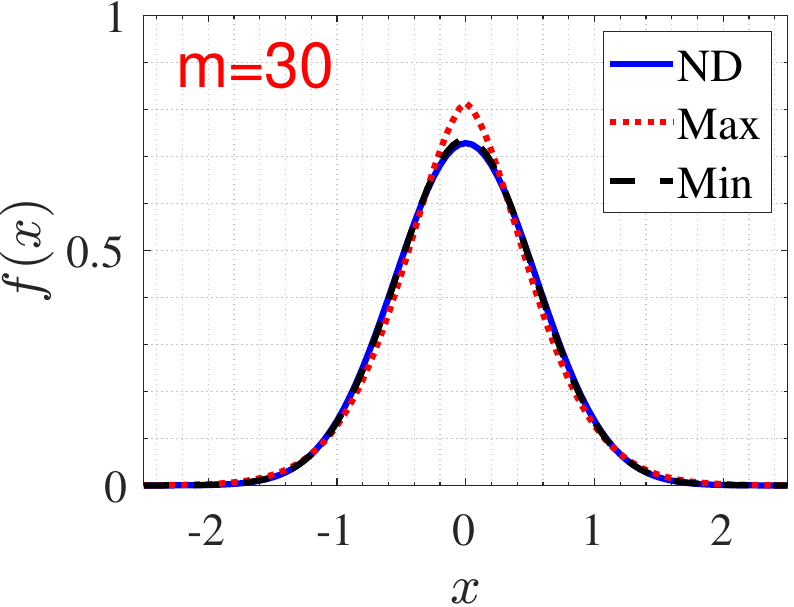}}
        \centerline{\scriptsize (j-2) Random: $m=30$}
	\end{minipage}
    \hfill
    \begin{minipage}{0.23\textwidth}
        \centering
		\centerline{\includegraphics[width=\textwidth]{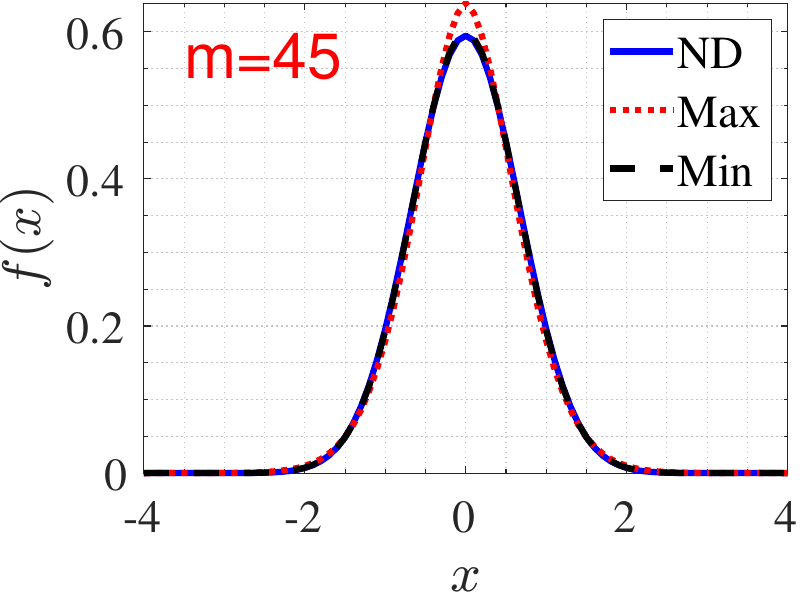}}
        \centerline{\scriptsize (j-3) Random: $m=45$}
	\end{minipage}
     \hfill
    \begin{minipage}{0.23\textwidth}
        \centering
		\centerline{\includegraphics[width=\textwidth]{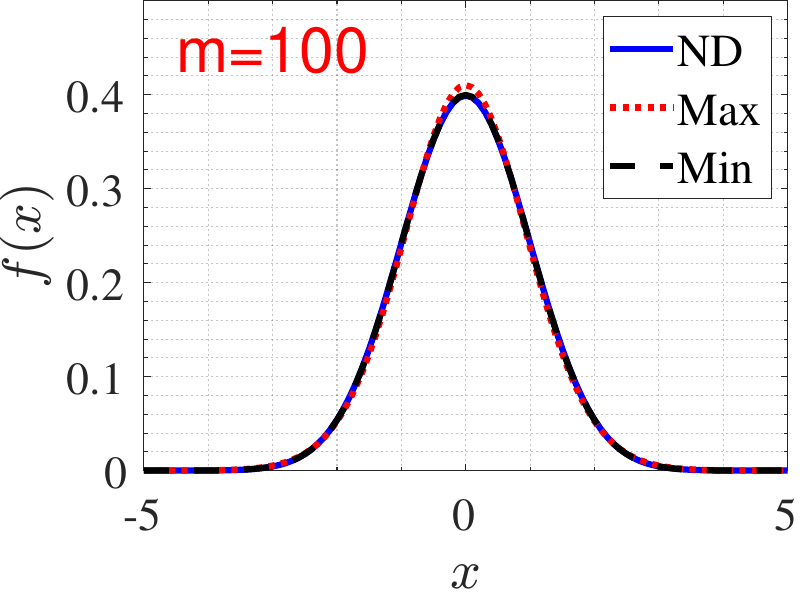}}
        \centerline{\scriptsize (j-4) Random: $m=n$}
	\end{minipage}
    \caption{Comparison of probability density curves of $\mathcal{N}(0,\frac{ms^{2}}{n})$ (ND) and $\tilde{s}X$ under different $m$.}
    \label{fig:pdf-curve-of-all-dataset}
\end{figure}

\begin{figure}[t]
	\centering
	\begin{minipage}{0.3\textwidth}
        \centering
		\centerline{\includegraphics[width=\textwidth]{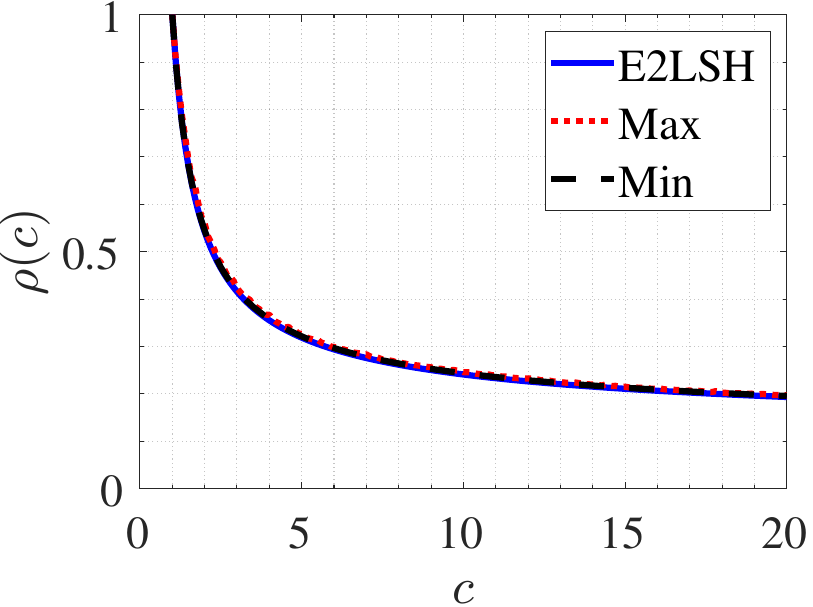}}
        \centerline{\scriptsize (a-1) Random: $\tilde{w}$ = 0.7 and $w$ = 1.5}
	\end{minipage}
    \hfill
	\begin{minipage}{0.3\textwidth}
        \centering
		\centerline{\includegraphics[width=\textwidth]{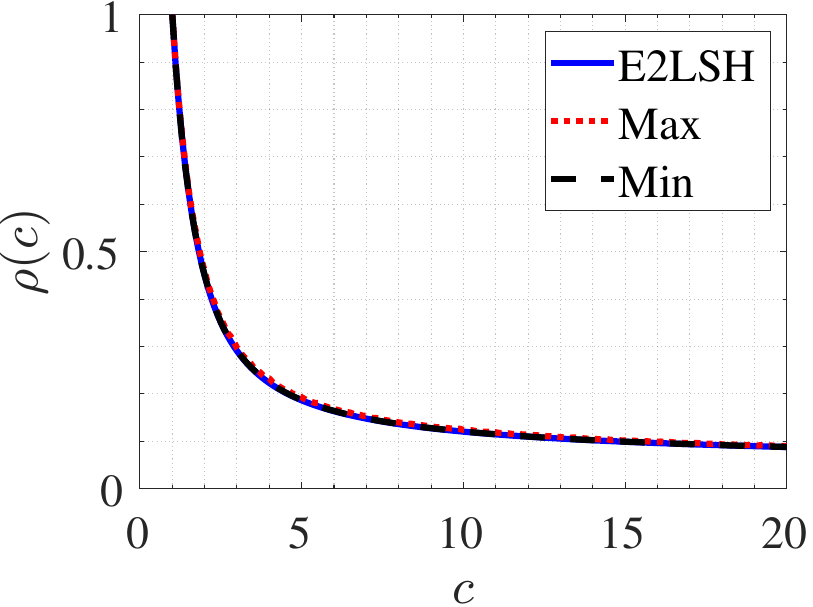}}
        \centerline{\scriptsize (a-2) Random: $\tilde{w}$ = 2 and $w$ = 4}
	\end{minipage}
    \hfill
	\begin{minipage}{0.3\textwidth}
        \centering
		\centerline{\includegraphics[width=\textwidth]{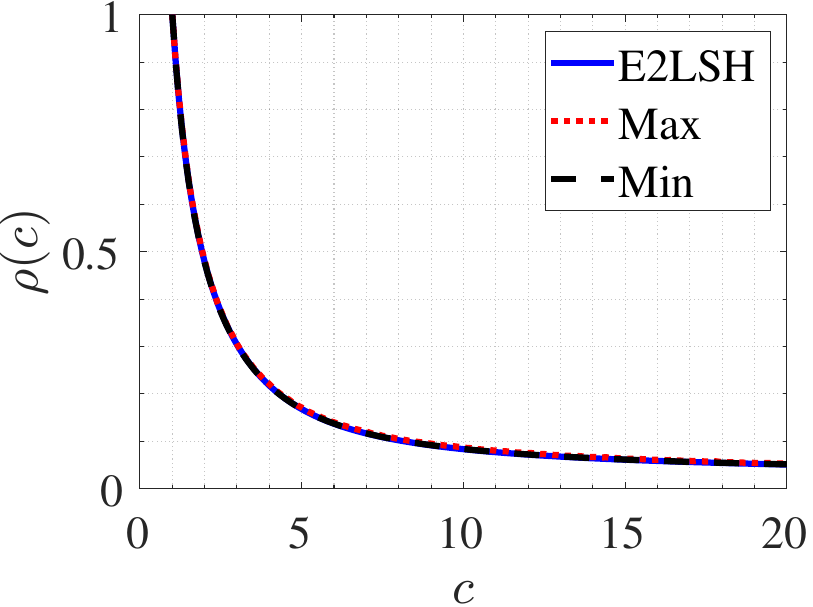}}
        \centerline{\scriptsize (a-3) Random: $\tilde{w}$ = 4.95 and $w$ = 10}
	\end{minipage}
    \hfill
	\begin{minipage}{0.3\textwidth}
        \centering
		\centerline{\includegraphics[width=\textwidth]{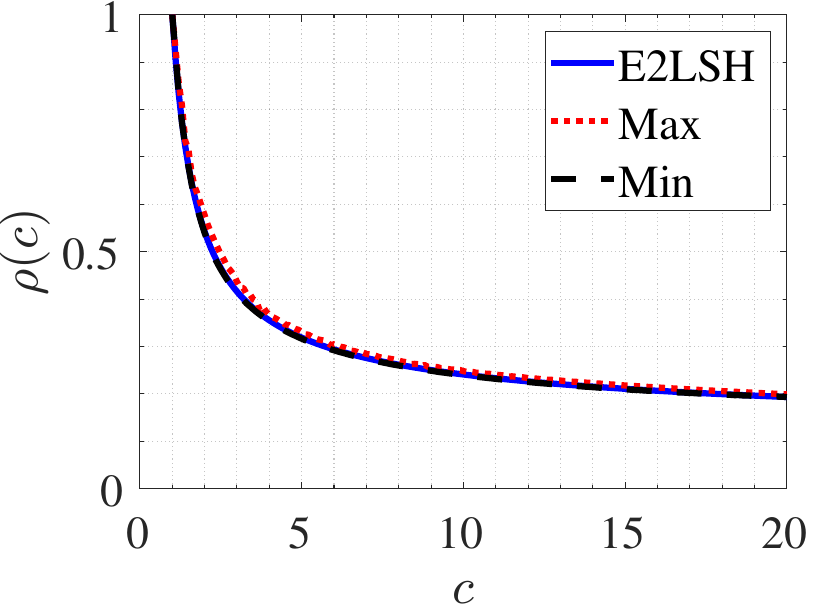}}
        \centerline{\scriptsize (b-1) Audio: $\tilde{w}$ = 0.59 and $w$ = 1.5}
	\end{minipage}
    \hfill
	\begin{minipage}{0.3\textwidth}
        \centering
		\centerline{\includegraphics[width=\textwidth]{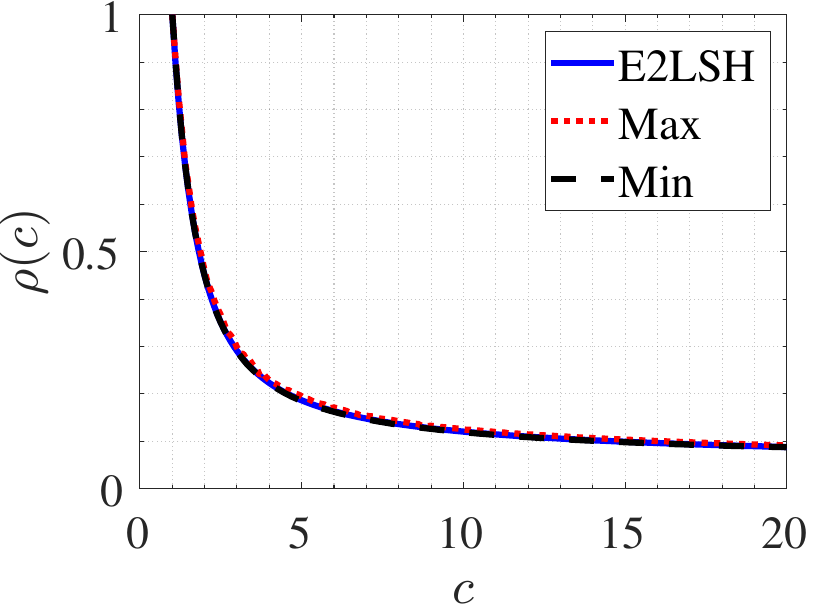}}
        \centerline{\scriptsize (b-2) Audio: $\tilde{w}$ = 1.6 and $w$ = 4}
	\end{minipage}
    \hfill
	\begin{minipage}{0.3\textwidth}
        \centering
		\centerline{\includegraphics[width=\textwidth]{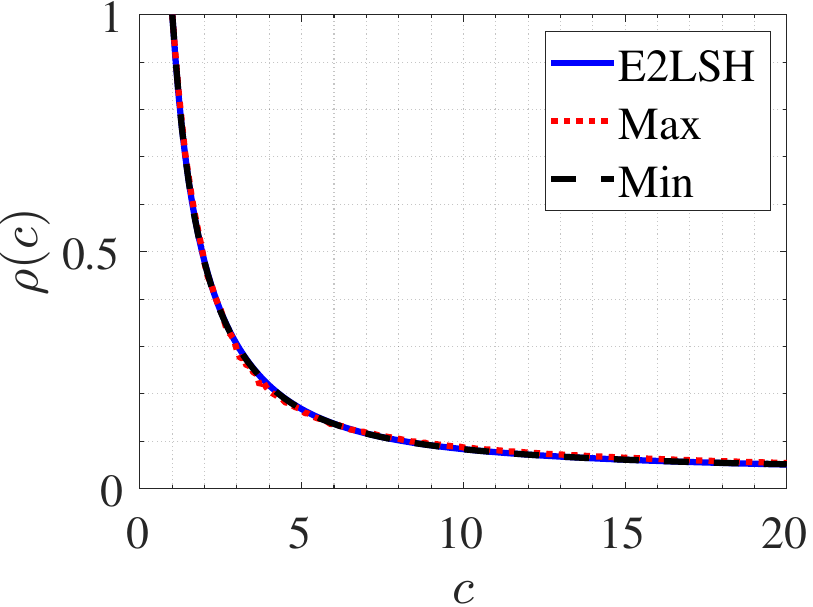}}
        \centerline{\scriptsize (b-3) Audio: $\tilde{w}$ = 3.95 and $w$ = 10}
	\end{minipage}
    \hfill
	\begin{minipage}{0.3\textwidth}
        \centering
		\centerline{\includegraphics[width=\textwidth]{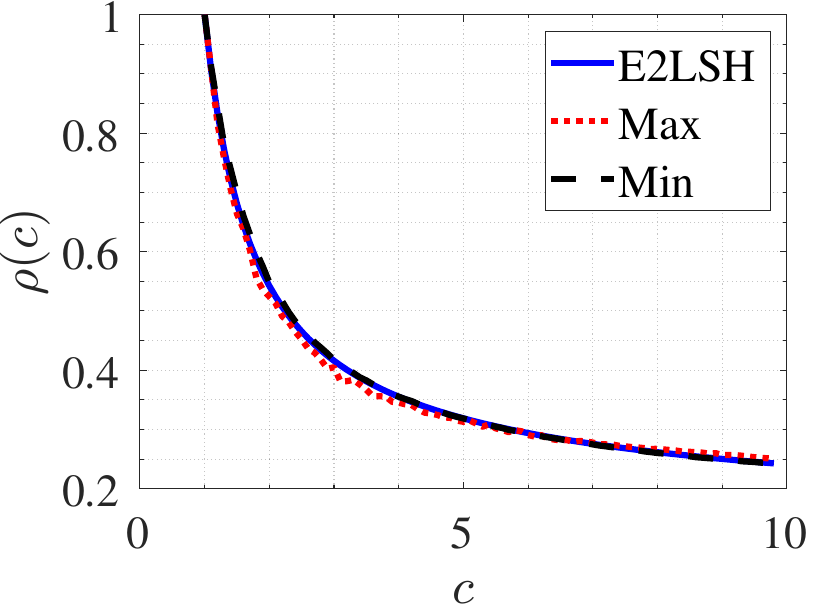}}
        \centerline{\scriptsize (c-1) Cifar: $\tilde{w}$ = 0.345 and $w$ = 1.5}
	\end{minipage}
    \hfill
	\begin{minipage}{0.3\textwidth}
        \centering
		\centerline{\includegraphics[width=\textwidth]{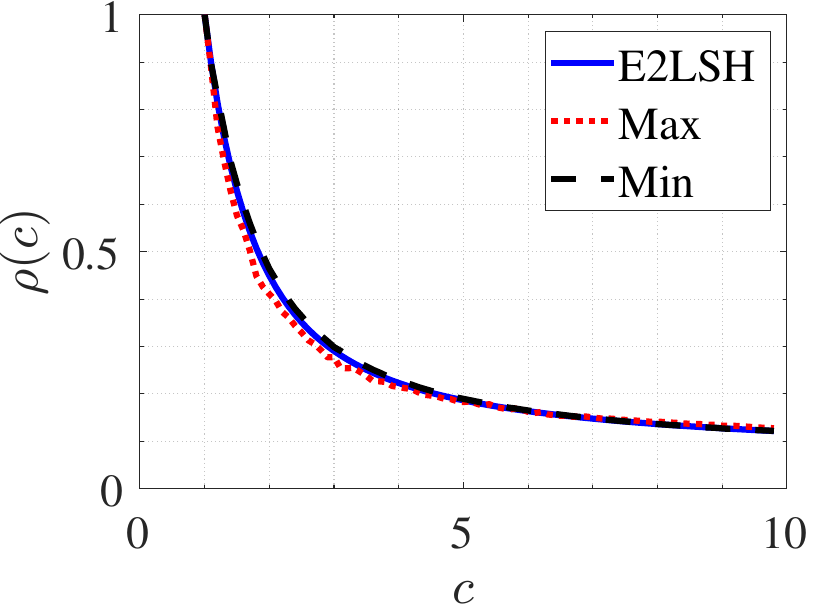}}
        \centerline{\scriptsize (c-2) Cifar: $\tilde{w}$ = 1 and $w$ = 4}
	\end{minipage}
    \hfill
	\begin{minipage}{0.3\textwidth}
        \centering
		\centerline{\includegraphics[width=\textwidth]{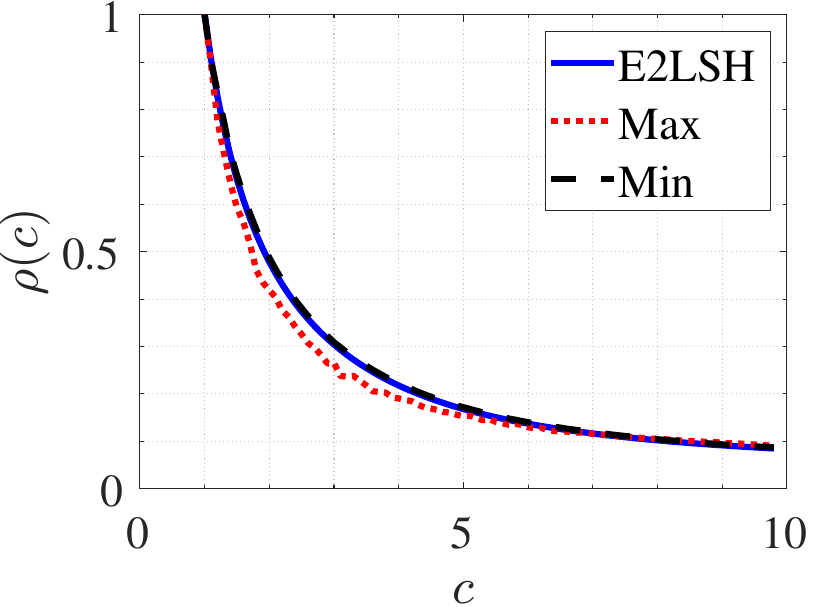}}
        \centerline{\scriptsize (c-3) Cifar: $\tilde{w}$ = 2.43 and $w$ = 10}
	\end{minipage}
    \hfill
	\begin{minipage}{0.3\textwidth}
        \centering
		\centerline{\includegraphics[width=\textwidth]{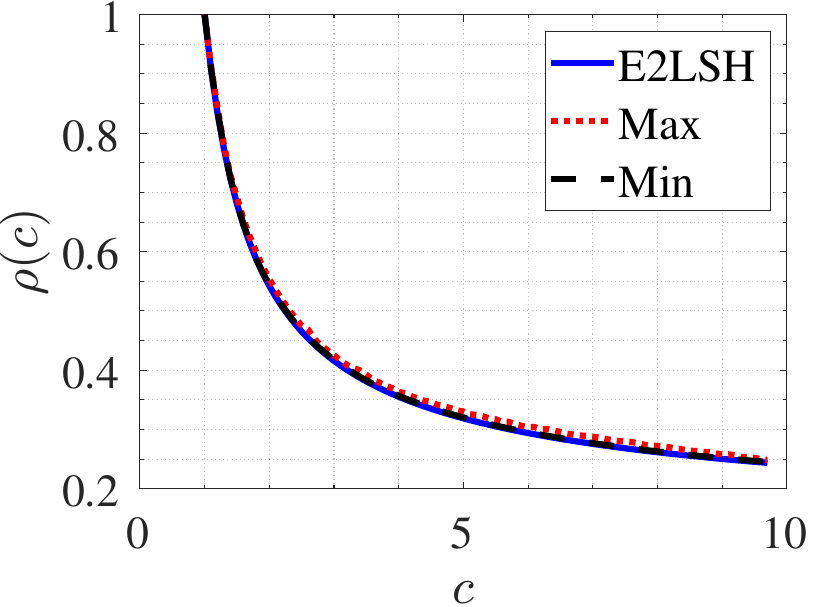}}
        \centerline{\scriptsize (d-1) Deep: $\tilde{w}$ = 0.5 and $w$ = 1.5}
	\end{minipage}
    \hfill
	\begin{minipage}{0.3\textwidth}
        \centering
		\centerline{\includegraphics[width=\textwidth]{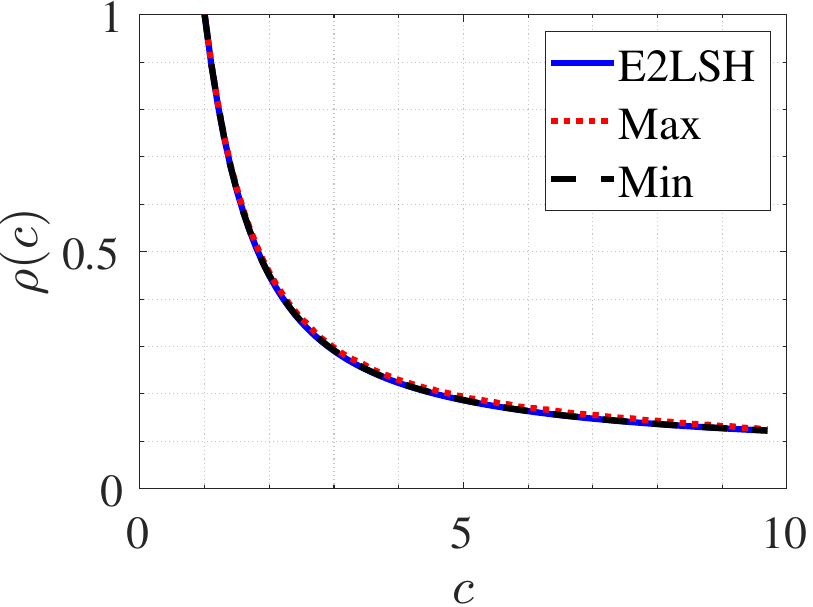}}
        \centerline{\scriptsize (d-2) Deep: $\tilde{w}$ = 1.36 and $w$ = 4}
	\end{minipage}
    \hfill
	\begin{minipage}{0.3\textwidth}
        \centering
		\centerline{\includegraphics[width=\textwidth]{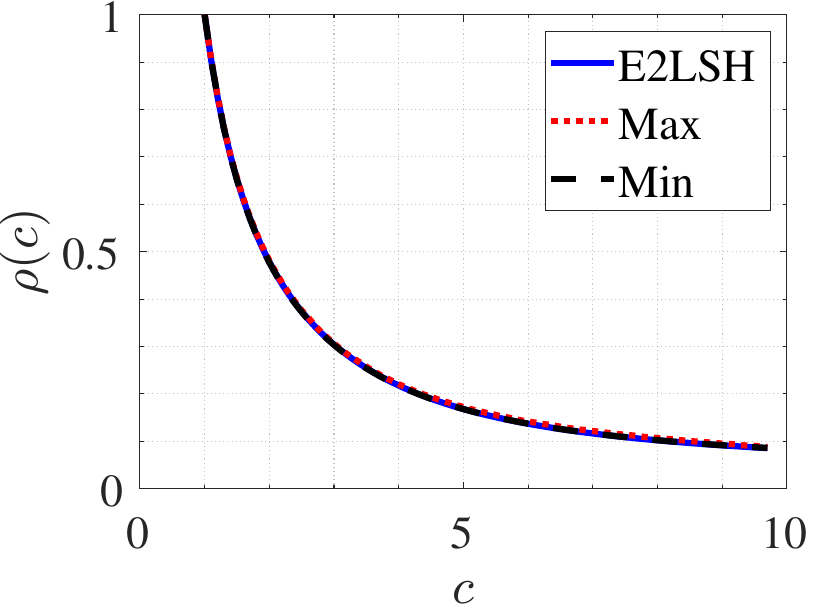}}
        \centerline{\scriptsize (d-3) Deep: $\tilde{w}$ = 3.4 and $w$ = 10}
	\end{minipage}
    \hfill
	\begin{minipage}{0.3\textwidth}
        \centering
		\centerline{\includegraphics[width=\textwidth]{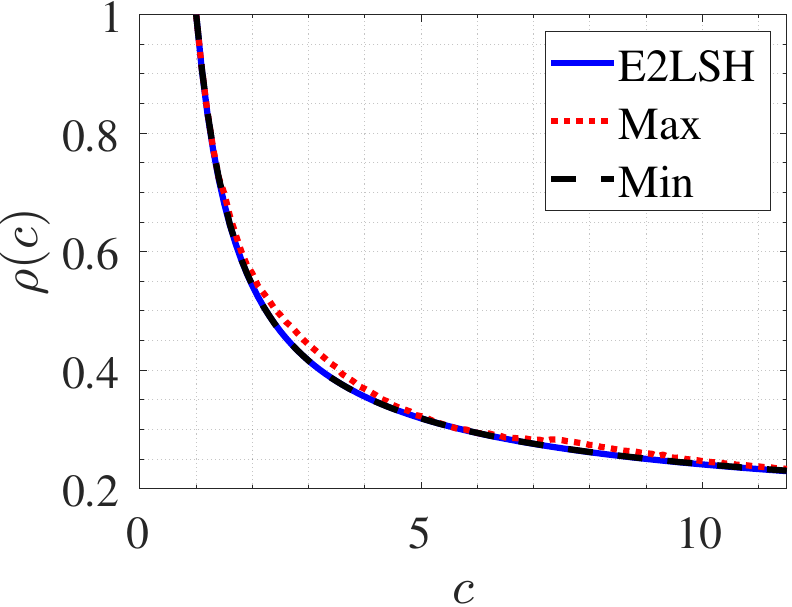}}
        \centerline{\scriptsize (e-1) Glove: $\tilde{w}$ = 0.81 and $w$ = 1.5}
	\end{minipage}
    \hfill
	\begin{minipage}{0.3\textwidth}
        \centering
		\centerline{\includegraphics[width=\textwidth]{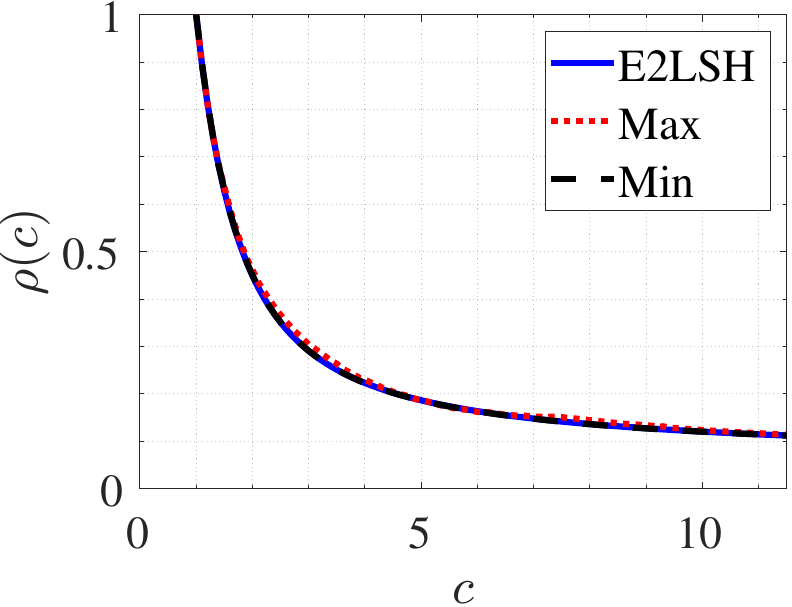}}
        \centerline{\scriptsize (e-2) Glove: $\tilde{w}$ = 2.2 and $w$ = 4}
	\end{minipage}
    \hfill
	\begin{minipage}{0.3\textwidth}
        \centering
		\centerline{\includegraphics[width=\textwidth]{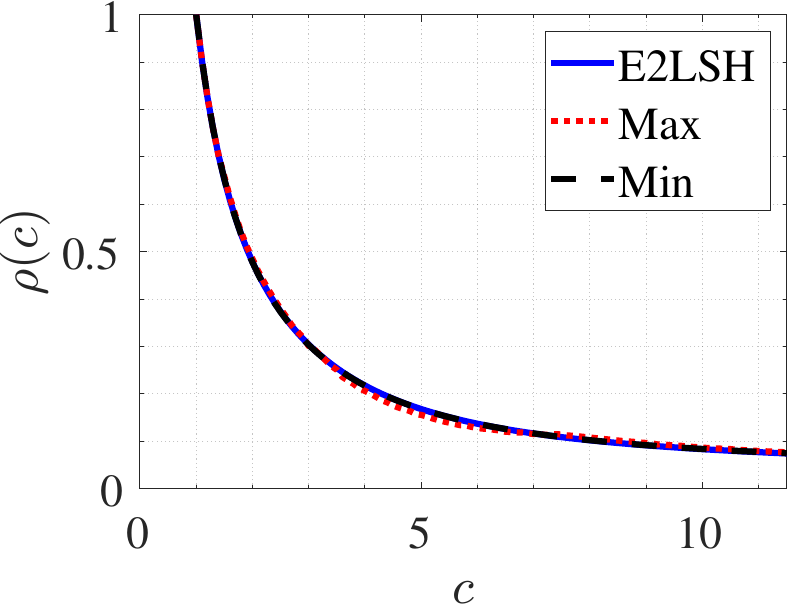}}
        \centerline{\scriptsize (e-3) Glove: $\tilde{w}$ = 5.45 and $w$ = 10}
	\end{minipage} 	
    \caption{$\rho$ curves under different bucket widths over datasets \emph{Random}, \emph{Audio} and \emph{Cifar}, \emph{Deep} and \emph{Glove}.}
    \label{fig:rho-curve-for-Random-Audio-Cifar-Deep-Glove}
\end{figure}

\begin{figure}[t]	
	\centering
        \begin{minipage}{0.3\textwidth}
        \centering
		\centerline{\includegraphics[width=\textwidth]{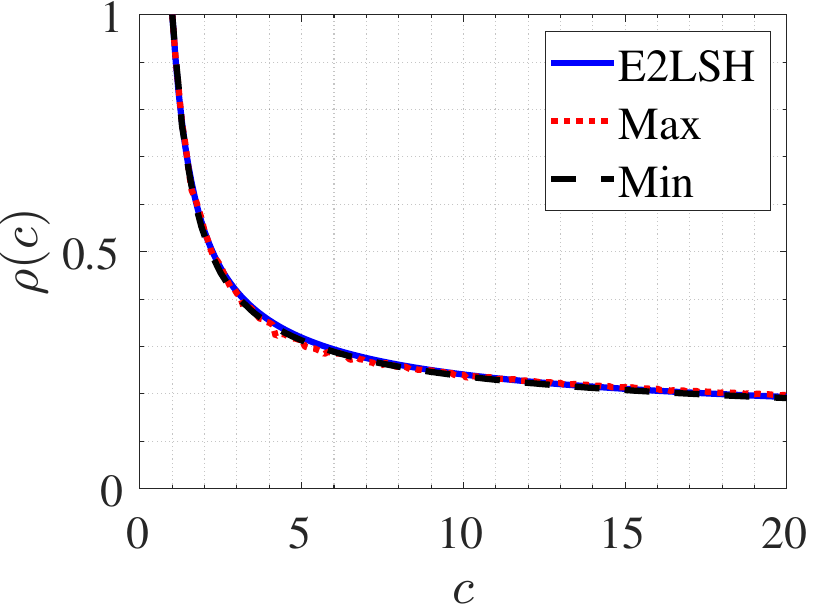}}
        \centerline{\scriptsize (a-1) ImageNet: $\tilde{w}$ = 0.635 and $w$ = 1.5}
	\end{minipage}
    \hfill
	\begin{minipage}{0.3\textwidth}
        \centering
		\centerline{\includegraphics[width=\textwidth]{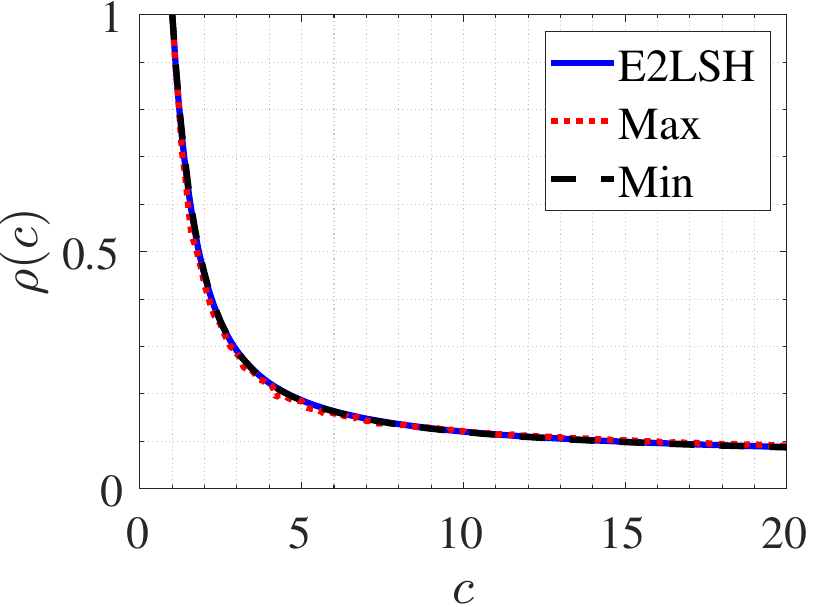}}
        \centerline{\scriptsize (a-2) ImageNet: $\tilde{w}$ = 1.8 and $w$ = 4}
	\end{minipage}
    \hfill
	\begin{minipage}{0.3\textwidth}
        \centering
		\centerline{\includegraphics[width=\textwidth]{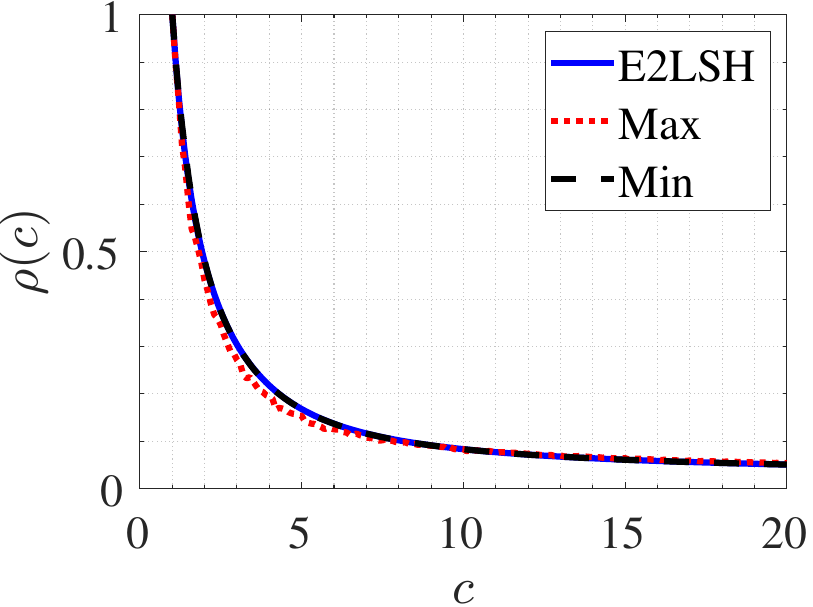}}
        \centerline{\scriptsize (a-3) ImageNet: $\tilde{w}$ = 4.45 and $w$ = 10}
	\end{minipage}
    \hfill
	\begin{minipage}{0.3\textwidth}
        \centering
		\centerline{\includegraphics[width=\textwidth]{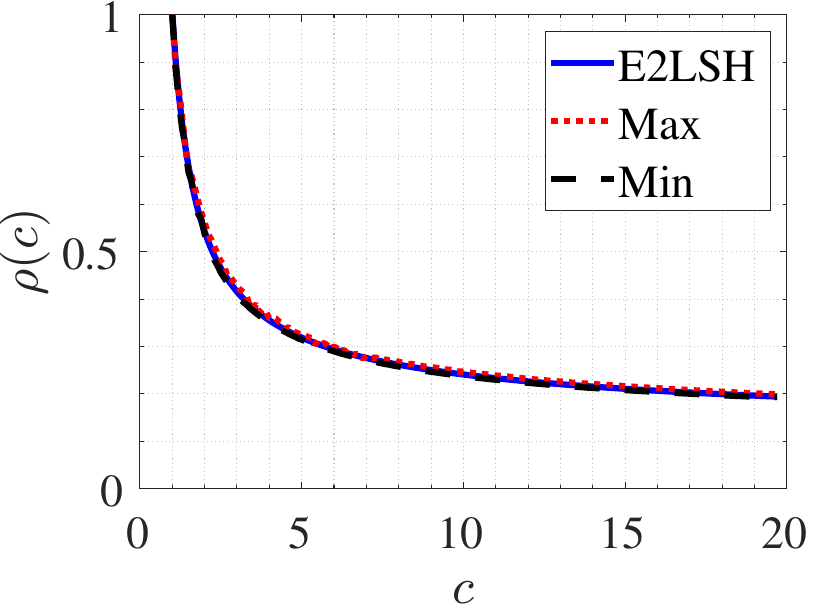}}
        \centerline{\scriptsize (b-1) Notre: $\tilde{w}$ = 0.69 and $w$ = 1.5}
	\end{minipage}
    \hfill
	\begin{minipage}{0.3\textwidth}
        \centering
		\centerline{\includegraphics[width=\textwidth]{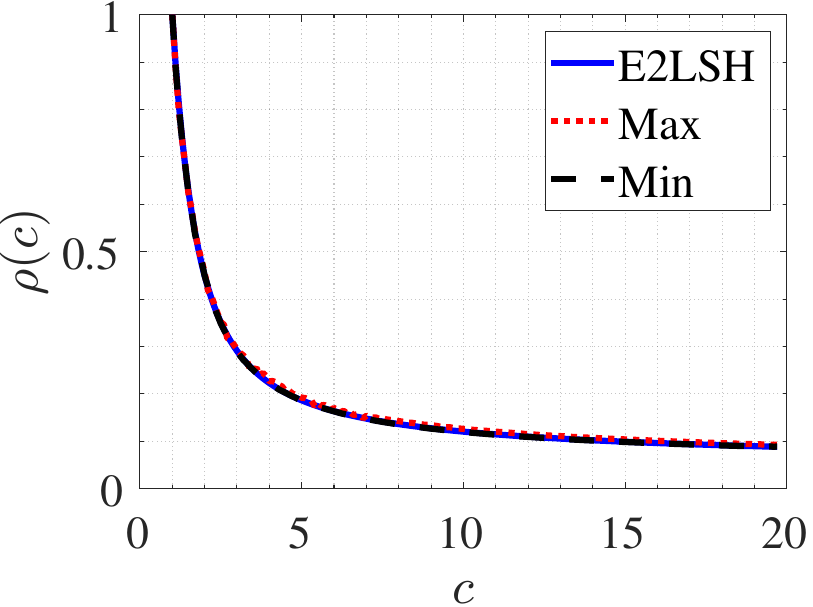}}
        \centerline{\scriptsize (b-2) Notre: $\tilde{w}$ = 1.9 and $w$ = 4}
	\end{minipage}
    \hfill
	\begin{minipage}{0.3\textwidth}
        \centering
		\centerline{\includegraphics[width=\textwidth]{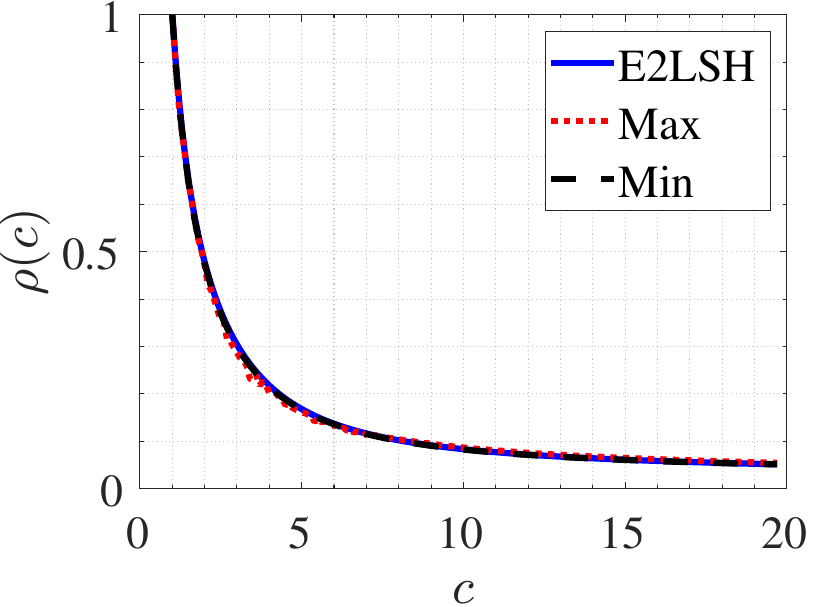}}
        \centerline{\scriptsize (b-3) Notre: $\tilde{w}$ = 4.75 and $w$ = 10}
	\end{minipage}
    \hfill
	\centering
	\begin{minipage}{0.3\textwidth}
        \centering
		\centerline{\includegraphics[width=\textwidth]{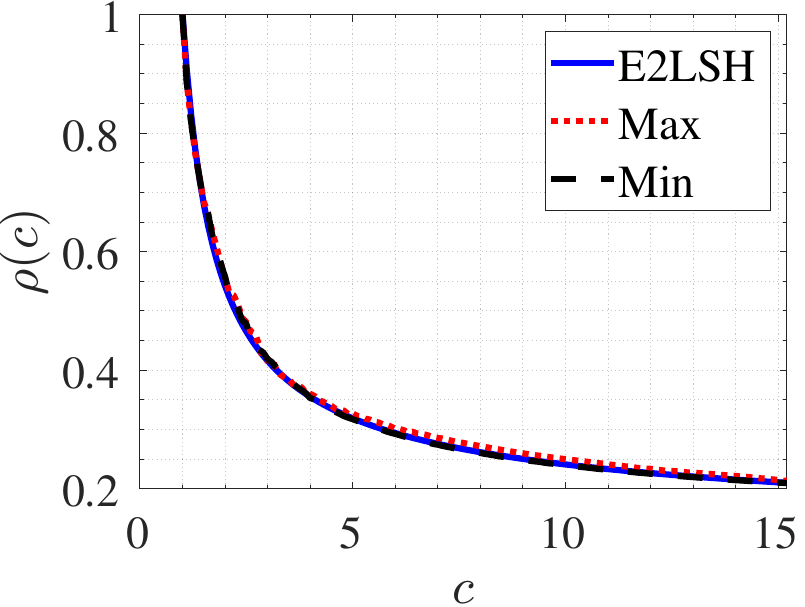}}
        \centerline{\scriptsize (c-1) Sift: $\tilde{w}$ = 0.68 and $w$ = 1.5}
	\end{minipage}
    \hfill
	\begin{minipage}{0.3\textwidth}
        \centering
		\centerline{\includegraphics[width=\textwidth]{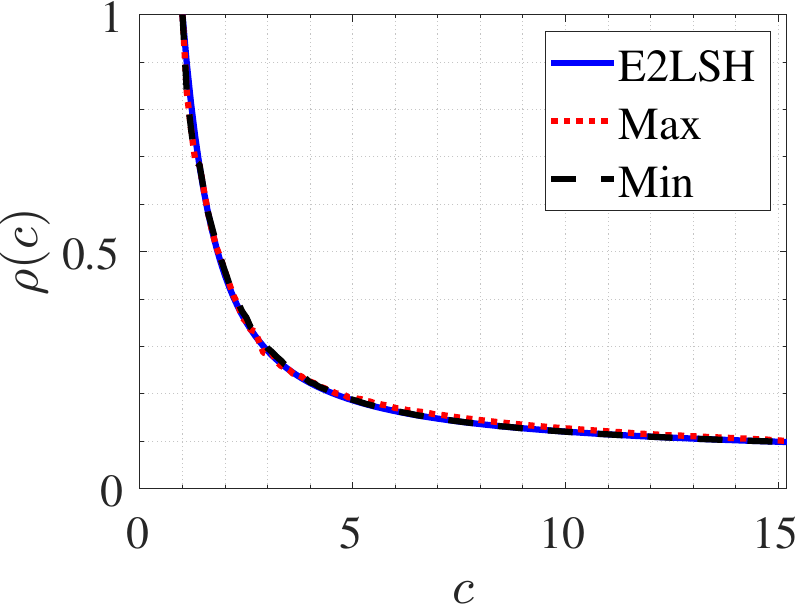}}
        \centerline{\scriptsize (c-2) Sift: $\tilde{w}$ = 1.9 and $w$ = 4}
	\end{minipage}
    \hfill
	\begin{minipage}{0.3\textwidth}
        \centering
		\centerline{\includegraphics[width=\textwidth]{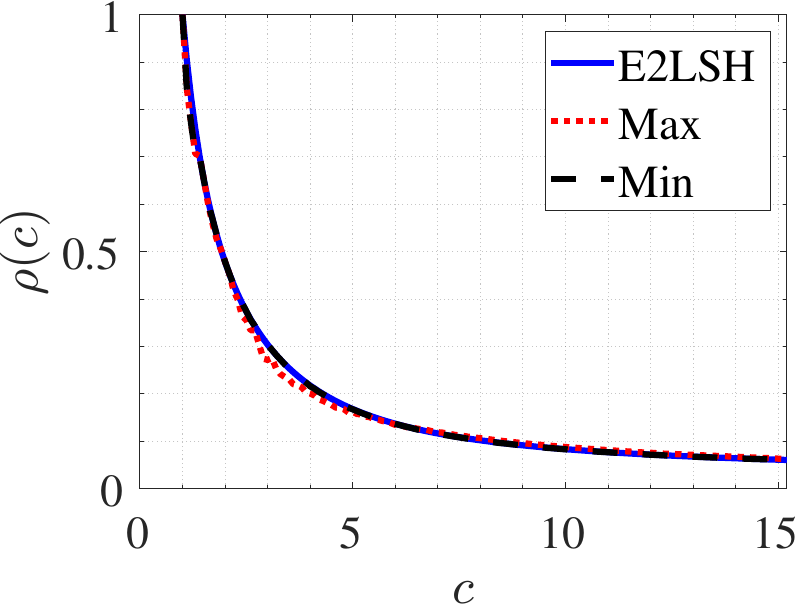}}
        \centerline{\scriptsize (c-3) Sift: $\tilde{w}$ = 4.75 and $w$ = 10}
	\end{minipage}
    \hfill
	\begin{minipage}{0.3\textwidth}
        \centering
		\centerline{\includegraphics[width=\textwidth]{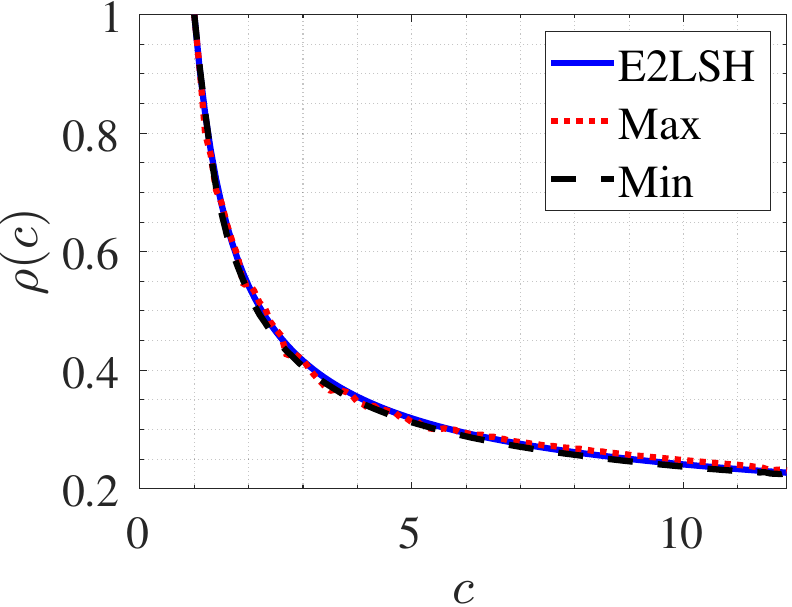}}
        \centerline{\scriptsize (d-1) Sun: $\tilde{w}$ = 0.343 and $w$ = 1.5}
	\end{minipage}
    \hfill
	\begin{minipage}{0.3\textwidth}
        \centering
		\centerline{\includegraphics[width=\textwidth]{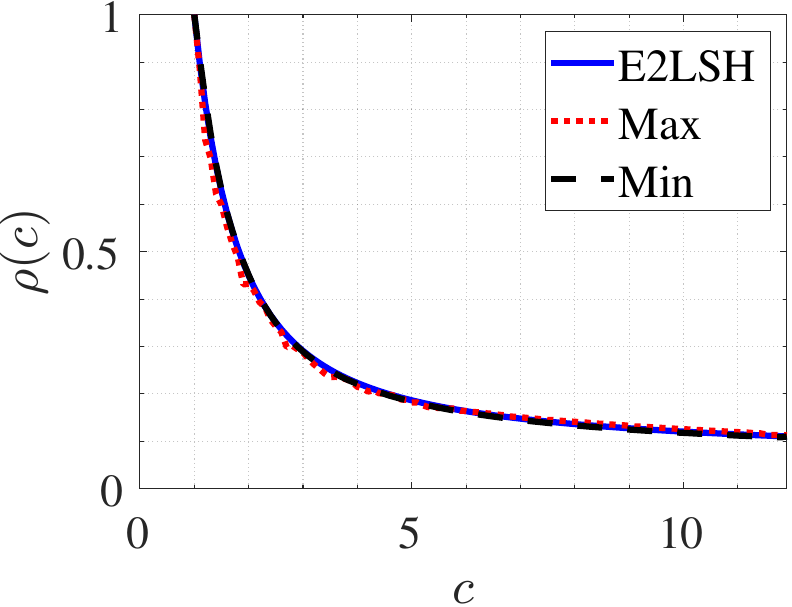}}
        \centerline{\scriptsize (d-2) Sun: $\tilde{w}$ = 0.98 and $w$ = 4}
	\end{minipage}
    \hfill
	\begin{minipage}{0.3\textwidth}
        \centering
		\centerline{\includegraphics[width=\textwidth]{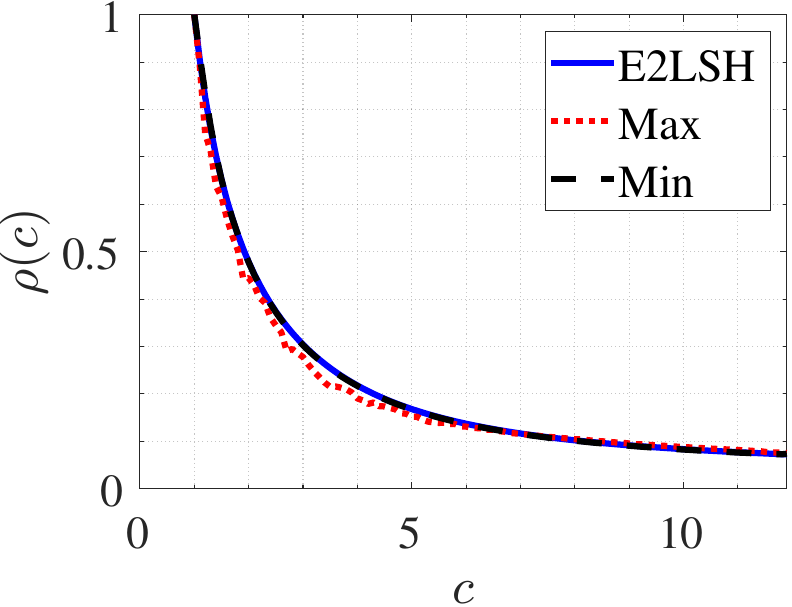}}
        \centerline{\scriptsize (d-3) Sun: $\tilde{w}$ = 2.4 and $w$ = 10}
	\end{minipage}
    \hfill
	\begin{minipage}{0.3\textwidth}
        \centering
		\centerline{\includegraphics[width=\textwidth]{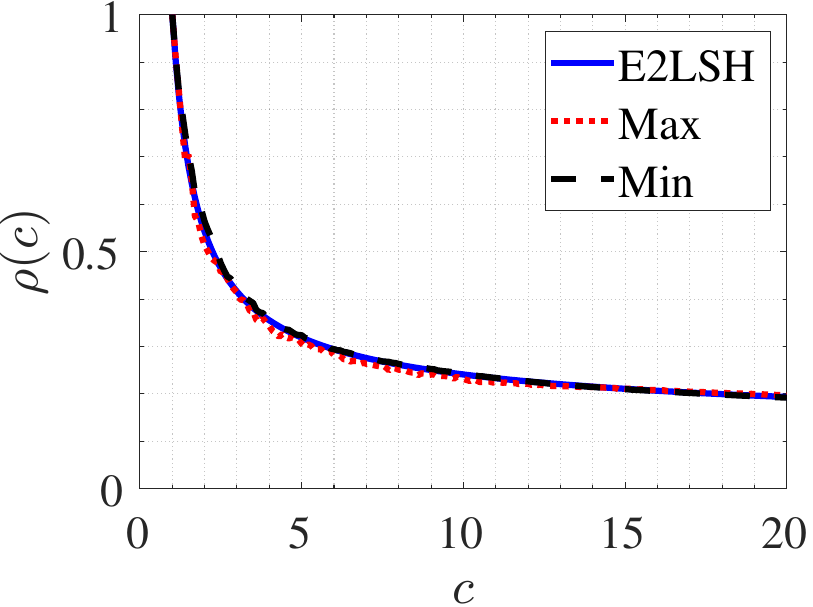}}
        \centerline{\scriptsize (e-1) Ukbench: $\tilde{w}$ = 0.7 and $w$ = 1.5}
	\end{minipage}
    \hfill
	\begin{minipage}{0.3\textwidth}
        \centering
		\centerline{\includegraphics[width=\textwidth]{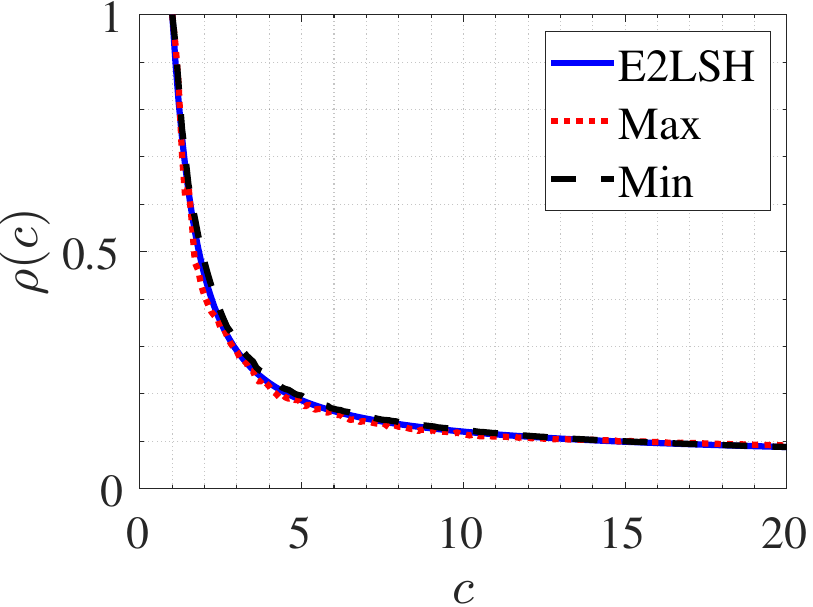}}
        \centerline{\scriptsize (e-2) Ukbench: $\tilde{w}$ = 2 and $w$ = 4}
	\end{minipage}
    \hfill
	\begin{minipage}{0.3\textwidth}
        \centering
		\centerline{\includegraphics[width=\textwidth]{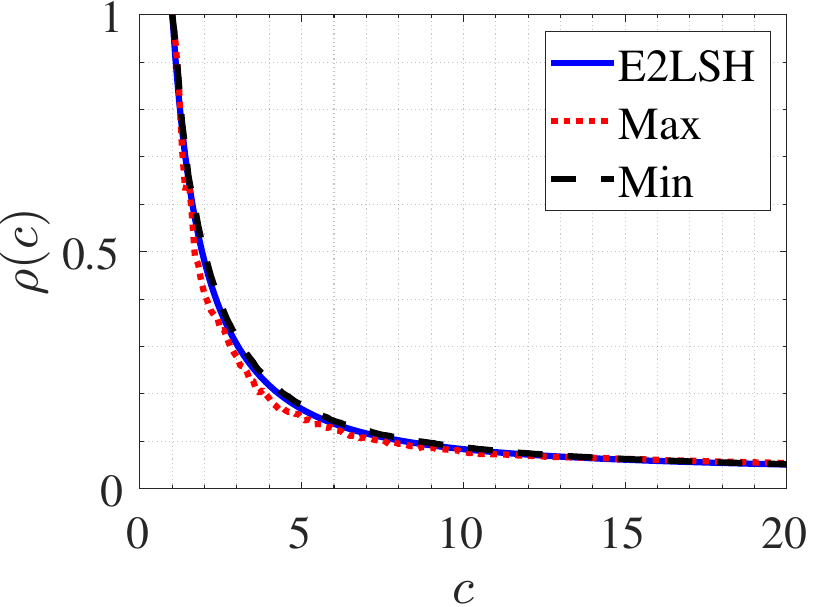}}
        \centerline{\scriptsize (e-3) Ukbench: $\tilde{w}$ = 4.95 and $w$ = 10}
	\end{minipage}
    \caption{$\rho$ curves under different bucket widths over datasets \emph{ImageNet}, \emph{Notre}, \emph{Sift}, \emph{Sun} and \emph{Ukbench}.}
    \label{fig:rho-curve-for-image-notre-sift-sun-ukbench}
\end{figure}

\begin{figure}[t]
	\centering
	\begin{minipage}{0.32\textwidth}
		\centering
		\centerline{\includegraphics[width=\textwidth]{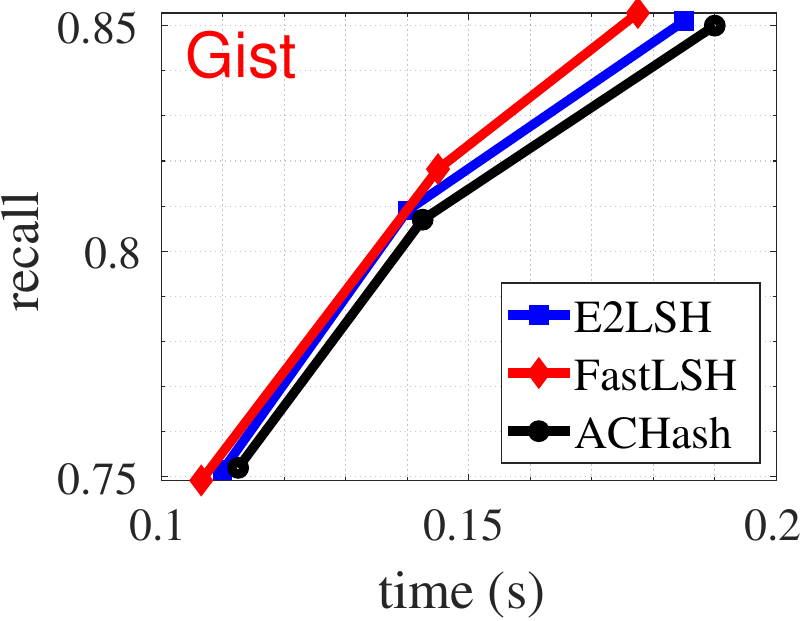}}
	\end{minipage}
    \hfill
	\begin{minipage}{0.32\textwidth}
        \centering
		\centerline{\includegraphics[width=\textwidth]{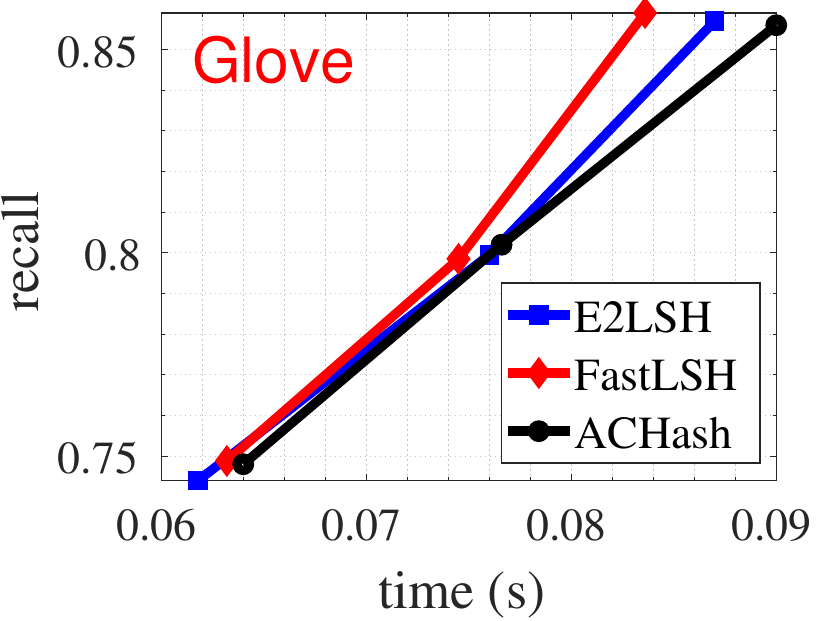}}
	\end{minipage}
    \hfill
    \begin{minipage}{0.32\textwidth}
        \centering
		\centerline{\includegraphics[width=\textwidth]{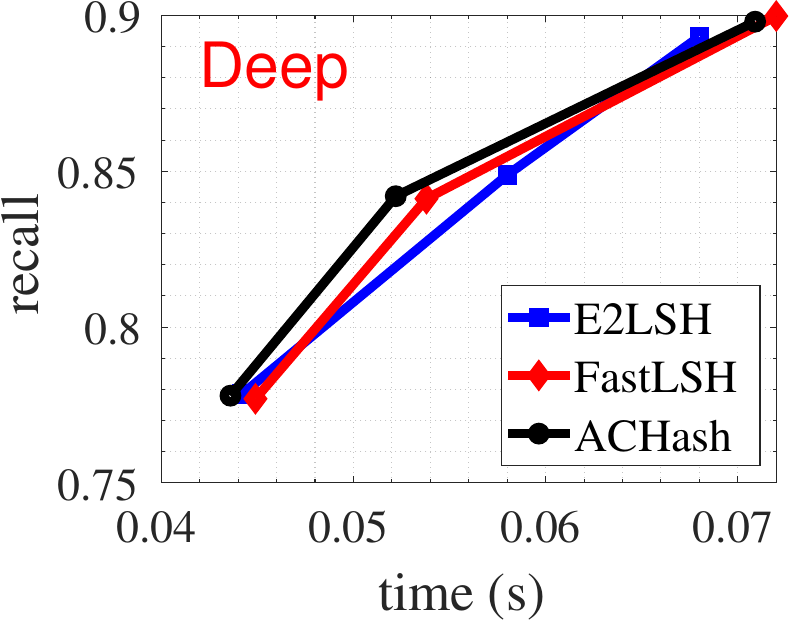}}
	\end{minipage}
    \hfill
	\begin{minipage}{0.32\textwidth}
        \centering
		\centerline{\includegraphics[width=\textwidth]{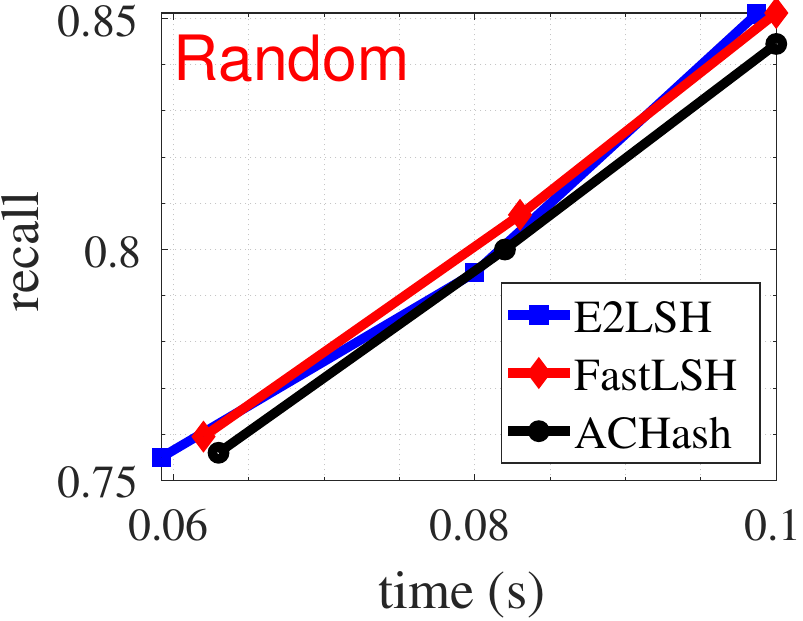}}
	\end{minipage}
   \hfill
	\begin{minipage}{0.32\textwidth}
        \centering
		\centerline{\includegraphics[width=\textwidth]{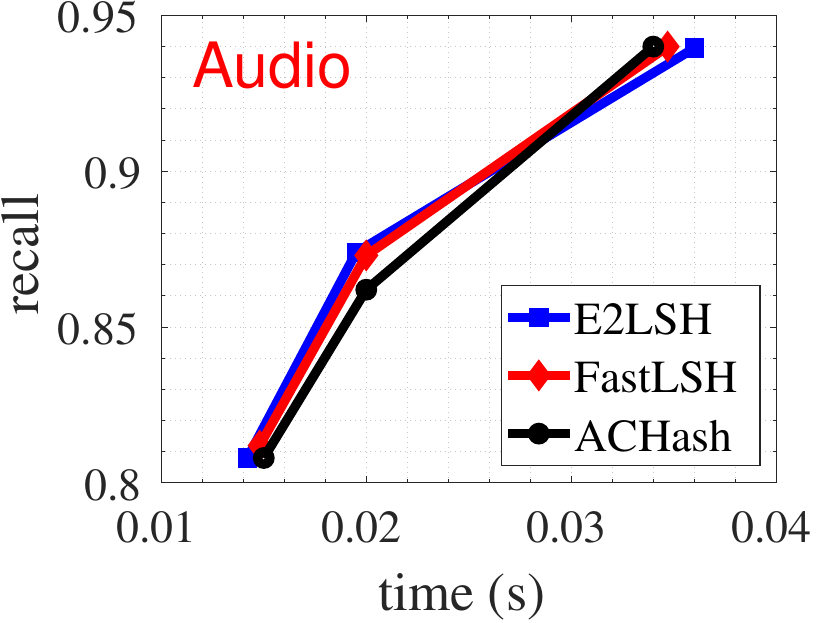}}
	\end{minipage}
    \hfill
	\begin{minipage}{0.32\textwidth}
        \centering
		\centerline{\includegraphics[width=\textwidth]{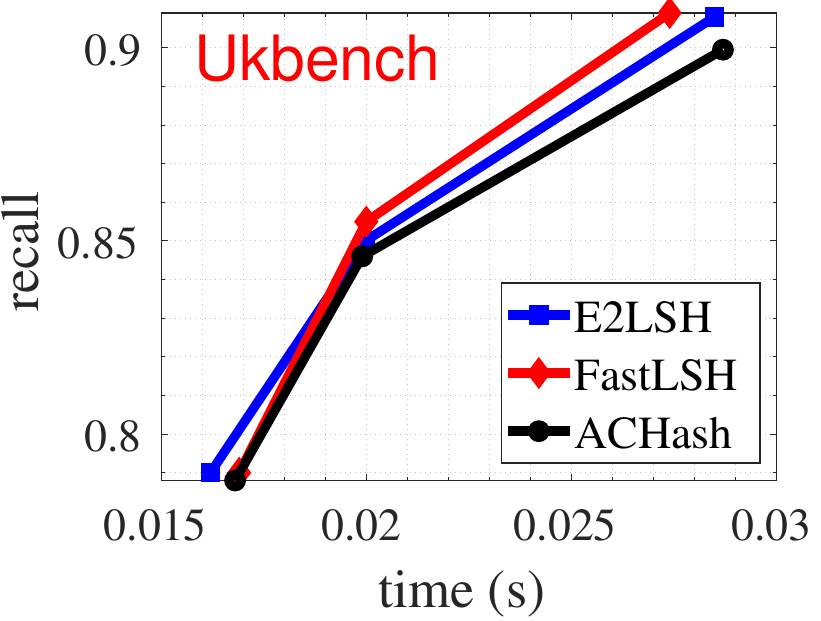}}
	\end{minipage}
    \hfill
	\begin{minipage}{0.32\textwidth}
        \centering
		\centerline{\includegraphics[width=\textwidth]{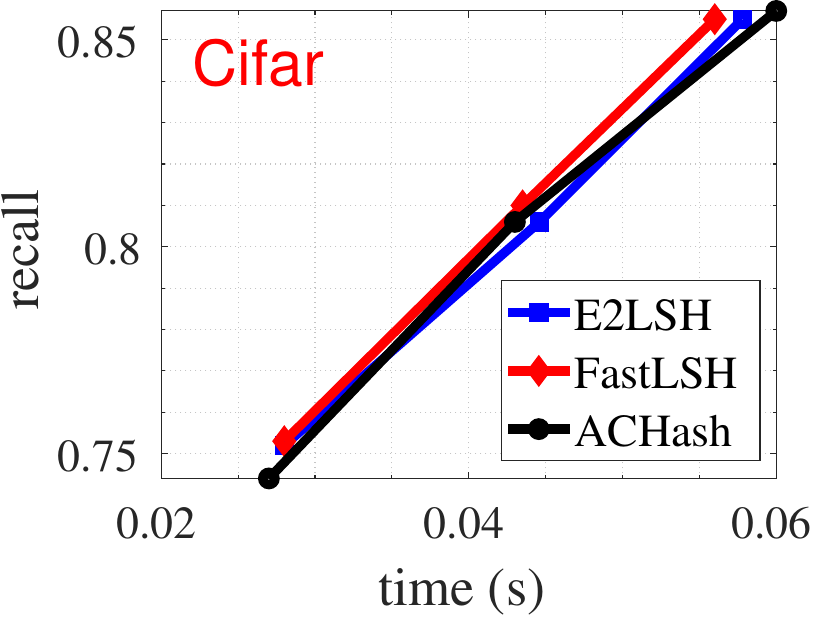}}
	\end{minipage}
    \hfill
    \begin{minipage}{0.32\textwidth}
        \centering
		\centerline{\includegraphics[width=\textwidth]{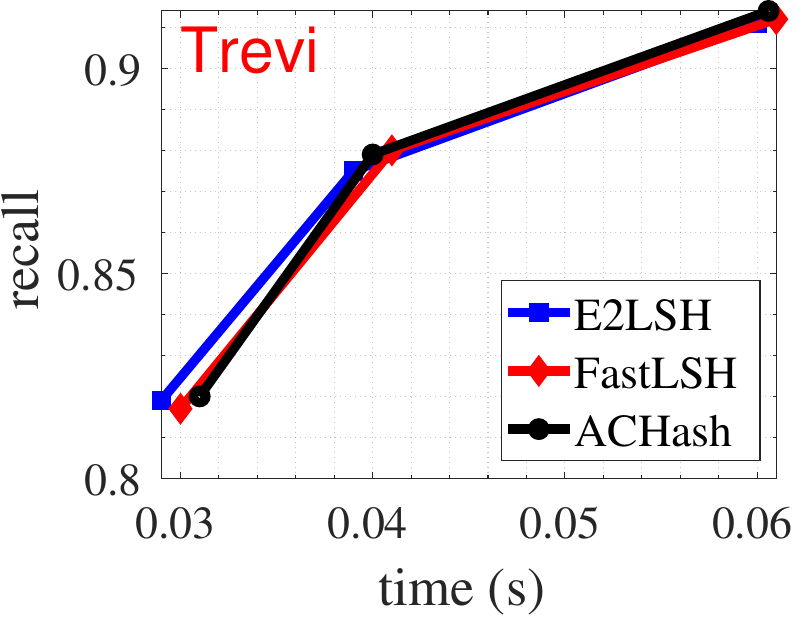}}
	\end{minipage}
    \hfill
    \begin{minipage}{0.32\textwidth}
        \centering
		\centerline{\includegraphics[width=\textwidth]{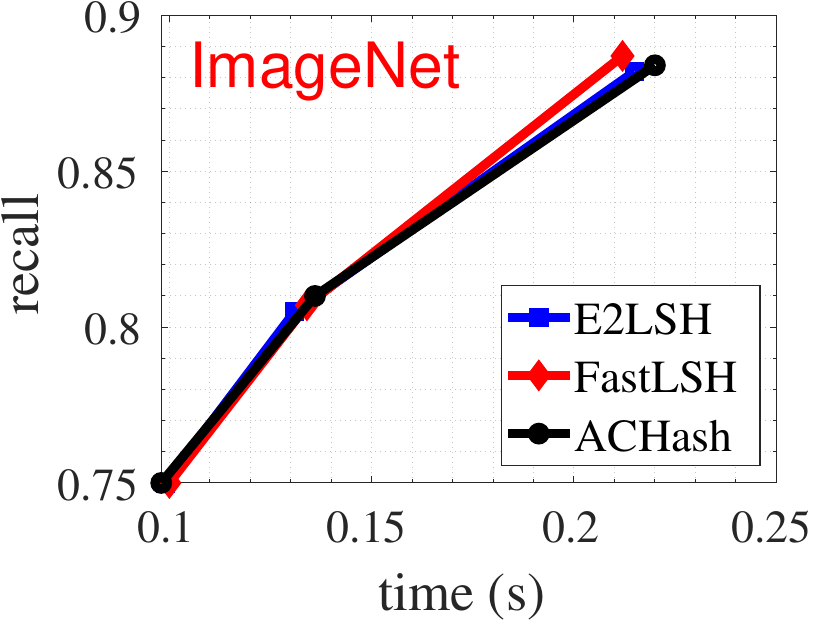}}
	\end{minipage}
    \hfill
    \begin{minipage}{0.32\textwidth}
        \centering
		\centerline{\includegraphics[width=\textwidth]{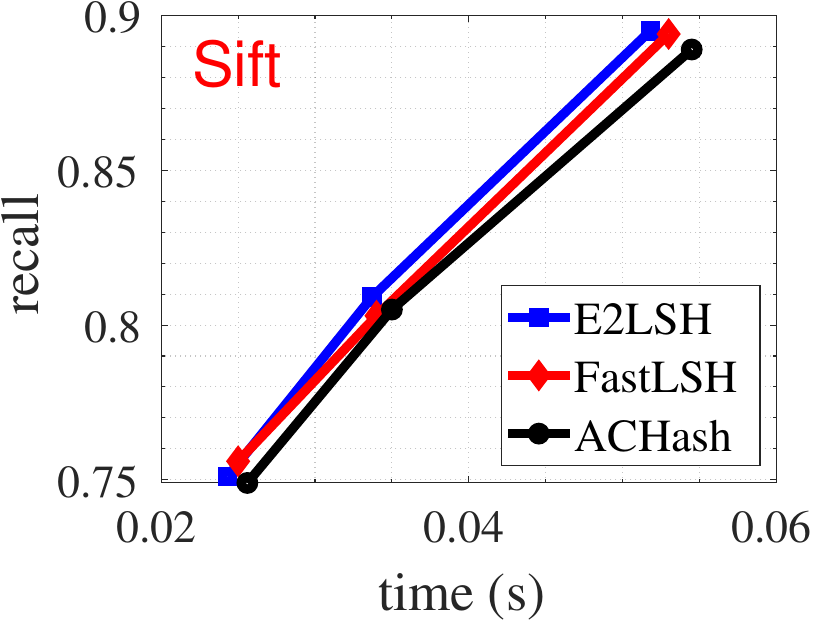}}
	\end{minipage}
    \hfill
    \begin{minipage}{0.32\textwidth}
        \centering
		\centerline{\includegraphics[width=\textwidth]{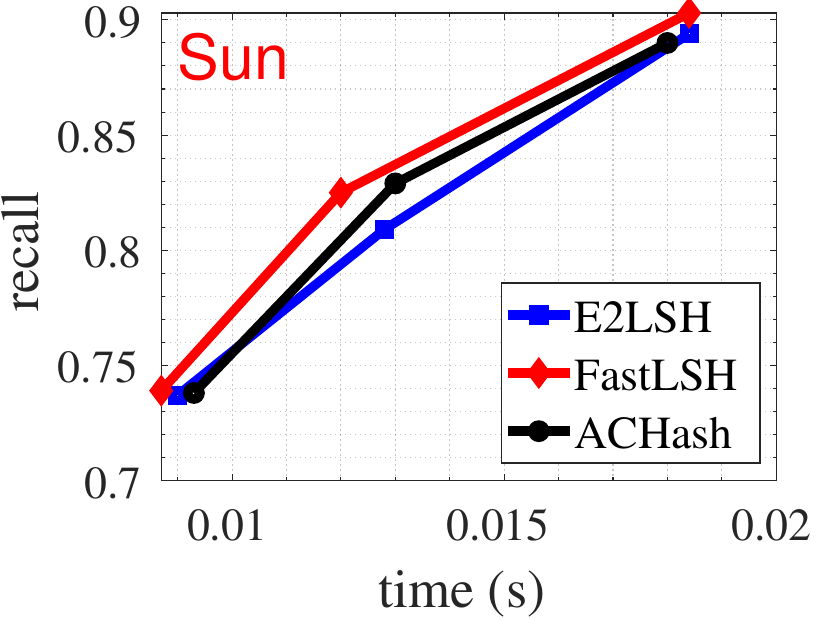}}
	\end{minipage}
    \hfill
    \begin{minipage}{0.32\textwidth}
        \centering
		\centerline{\includegraphics[width=\textwidth]{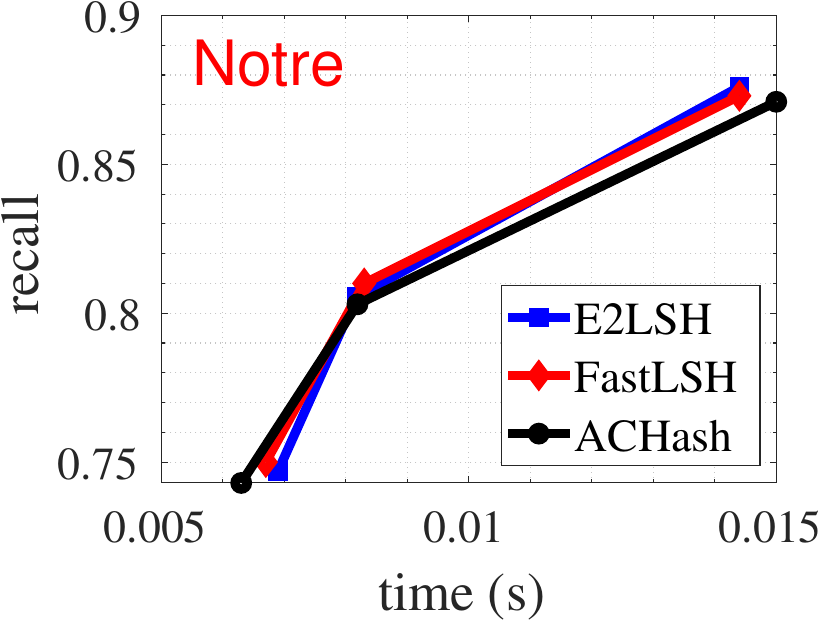}}
	\end{minipage}
    \caption{Recall v.s. average query time}
    \label{fig:time-recall-curve-in-appendix}
\end{figure}

\end{document}